\documentclass[12pt]{article}

\usepackage{amsmath}
\usepackage{times}
\usepackage[dvipsone]{graphicx}
\usepackage{color}
\usepackage{multirow}
\usepackage[authoryear]{natbib}
\usepackage{rotating}
\usepackage{bbm}
\usepackage{latexsym}
\usepackage{epstopdf}   

\usepackage{romannum}
\AtBeginDocument{\pagenumbering{arabic}}

\makeatletter
\@addtoreset{equation}{section}
\makeatother
\renewcommand{\theequation}{\arabic{section}.\arabic{equation}}

\newcommand{\captionfonts}{\normalsize}

\makeatletter
\long\def\@makecaption#1#2{%
  \vskip\abovecaptionskip
  \sbox\@tempboxa{{\captionfonts #1: #2}}%
  \ifdim \wd\@tempboxa >\hsize
    {\captionfonts #1: #2\par}
  \else
    \hbox to\hsize{\hfil\box\@tempboxa\hfil}%
  \fi
  \vskip\belowcaptionskip}
\makeatother

\usepackage{amsmath}
\usepackage{mathrsfs}

\usepackage{algorithm}
\usepackage{algorithmic}
\makeatletter
\renewcommand{\fnum@algorithm}{\fname@algorithm}
\makeatother

\usepackage{amsthm}

\usepackage{pmat}
\usepackage{subfig}
\usepackage{caption}

\usepackage{amssymb}

\usepackage[retainorgcmds]{IEEEtrantools}

\renewcommand{\theequationdis}{{\normalfont (\theequation)}}
\renewcommand{\theIEEEsubequationdis}{{\normalfont (\theIEEEsubequation)}}

\makeatletter
\newcommand*{\rom}[1]{\expandafter\@slowromancap\romannumeral #1@}
\makeatother

\newtheorem{thm}{Theorem}
\newtheorem{lem}{Lemma}
\newtheorem{dfn}{Definition}
\newtheorem{prp}{Proposition}
\newtheorem*{rmk}{Remark}
\newtheorem{rmk-2}{Remark}
\newtheorem{rmk-3}{Remark}
\newtheorem{rmk-4}{Remark}
\newtheorem{rmk-5}{Remark}
\newtheorem{rmk-6}{Remark}
\newtheorem{rmk-7}{Remark}
\newtheorem{rmk-8}{Remark}
\newtheorem{cl}{Corollary}

\newtheorem{ax}{Axiom}

\usepackage[hyperindex,breaklinks]{hyperref}
\hypersetup{
           breaklinks=true,   
           colorlinks=true,   
           pdfusetitle=true,  
        }
\usepackage{setspace}
\begin{document}

\ \vspace{20mm}\\

{\LARGE \flushleft  On the Principles of ReLU Networks with One Hidden Layer}

\ \\
{\bf \large Changcun Huang}\\
{cchuang@mail.ustc.edu.cn}\\
{Shuitu Institute of Applied Mathematics, Chongqing 400700, P.R.C}\\
%


\thispagestyle{empty}

\ \vspace{-0mm}\\
%
\begin{center} {\bf Abstract} \end{center}
A neural network with one hidden layer or a two-layer network (regardless of the input layer) is the simplest feedforward neural network, whose mechanism may be the basis of more general network architectures. However, even to this type of simple architecture, it is also a ``black box''; that is, it remains unclear how to interpret the mechanism of its solutions obtained by the back-propagation algorithm and how to control the training process through a deterministic way. This paper systematically studies the first problem by constructing universal function-approximation solutions. It is shown that, both theoretically and experimentally, the training solution for the one-dimensional input could be completely understood, and that for a higher-dimensional input can also be well interpreted to some extent. Those results pave the way for thoroughly revealing the black box of two-layer ReLU networks and advance the understanding of deep ReLU networks.

\ \\[-2mm]
{\bf Keywords:} One-hidden layer, black box, function approximation, training solution, ReLU network.

\section{Introduction}
Deep learning has achieved a great triumph in both scientific and engineering areas in recent years, such as \citet*{Jumper2021} in protein-structure prediction and \citet*{OpenAI2023} in large-language models. However, its underlying principle is still mysterious and usually referred to as a ``black box''\citep*{Castelvecchi2016,Roscher2020}, which involves not only application reliability but also potential uncontrollable AI risks.

Technically speaking, the black box of neural networks is roughly composed of two parts. The first is what the mechanism of the solution derived from the back-propagation algorithm \citep*{Rumelhart1986} (also called the ``\textsl{training solution}'' in this paper) is. The second is how to obtain a desired solution through a deterministic way. The two parts are correlated because if the solution space is unknown, it would be hard to control the solution-finding process. In fact, due to the lack of the knowledge of solutions, one can only adjust the parameters intuitively and stochastically, leading to uncertainty of the training results.

Although a ReLU network with one-hidden layer is the simplest network architecture, its solution space is largely unknown \citep*{DeVore2021} and hence its training process cannot be determinedly manipulated as well. In scientific areas, to study a system, the investigation of its smallest component is a natural way. So unveiling the black box of two-layer shallow ReLU networks may be the foundation of understanding deeper ones.

This paper aims to uncover the solution space of two-layer ReLU networks for function approximation and to understand the solution obtained by the back-propagation algorithm. Despite the main idea for the one-dimensional input being inspired by the one-sided bases of splines, the higher-dimensional case needs to be developed with the aid of some new principles, which include multiple strict partial orders and continuity restriction, rendering it essentially different from the one-dimensional input. It will be seen that the training solution can be well interpreted under our framework both theoretically and experimentally.

\subsection{Related Work}
The function-approximation capability of two-layer neural networks had attracted the attention of mathematical researchers for more than 30 years. There's a large body of research for the non-ReLU cases. \citet*{Pinkus1999} gave a comprehensive review on the results up to 1999; after that, for example, \citet*{Draghici2002}, \citet*{Ismailov2012}, \citet*{Costarelli2013}, \citet*{Guliyev2018} and \citet*{Almira2021} continued the research until recently. But none of those works exclusively aimed at extracting or summarizing the principle of training solutions, while developing the theory from pure mathematical viewpoint.

There were also works for two-layer ReLU networks \citep*{Breiman1993, DeVore2021, Hatano2021}. However, the similar problem exists as above. Thus, although much had been done for theoretical analysis, little is known about the mechanism of the solution obtained by the back-propagation algorithm.

To the relationship between splines and ReLU networks, \citet*{Daubechies2019} studied the approximation to univariate functions by deep ReLU networks and compared the effect with that of linear splines. The spline expressed by ReLU networks can also be regarded as the solution of some regularized optimization problems \citep*{Unser2019,Aziznejad2019,Bohra2020,Parhi2021}.

Despite using the term ``hinge function'', \citet*{Breiman1993}'s proof of theorem 3 involved the one-sided bases of linear splines and the conclusion holds for two-layer ReLU networks as well, a result mostly related to part of this paper. \citet*{Balestriero2021} extended \citet*{Breiman1993}'s work to the multi-output case and proved that a wide range of deep ReLU networks, such as convolutional neural networks and ResNets, can be written as the form of the composition of spline functions. The above correlations between splines and ReLU networks are either for deep ReLU networks or for pure theoretical analysis and are different from our way of introducing the idea of splines.

\subsection{Road Map of the Paper}
This paper develops the theory gradually from the simplest univariate case (\textbf{section 2}) to the multivariate case (\textbf{sections 3} to \textbf{7}), with the theoretical explanation of solutions embedded in the proof of the conclusions. After the theory having been established, experimental verification will be given in \textbf{section 8}, in which several examples of training solutions will be explained by the theory.

Throughout this paper, the term ``training solution'' is the abbreviated version of the solution obtained by the back-propagation algorithm, which is the output function $g(\boldsymbol{x})=\sum_i\lambda_i\phi_i(\boldsymbol{x})$ of a two-layer ReLU network interpolating a data set $D \subset \mathbb{R}^n$ for $n \ge 1$, where $\phi_i(\boldsymbol{x})=\sigma(w_i^T\boldsymbol{x}+b_i)$ with $\sigma(x)=\max\{0,x\}$ being the activation function of a rectified linear unit (ReLU).

To the details, \textbf{section 2} gives a simple example, the approximation to univariate function via two-layer ReLU networks, through the principle of one-sided bases of splines (theorem 1), which includes the idea to be generalized to the multivariate case in sections 3 and 4. \textbf{Section 3} generalizes the concept of knots on one-dimensional line as well as their ''less than'' relation (a strict partial order) to the higher-dimensional case. \textbf{Section 4} completes the generalization of section 2 for multivariate-function approximation over a single strict partial order (theorem 3), especially by establishing the relationship between the linear pieces of a two-layer ReLU network with higher-dimensional input (theorem 2).

\textbf{Section 5} enlarges the solution space by introducing the concept of ''two-sided bases'' of splines (theorem 5) on the basis of the preceding sections, after which the training solution for one-dimensional input is theoretically explained (corollary 3). This new concept is important for the diversity of solutions.

\textbf{Sections 6} and \textbf{7} further increase the complexity of the constructed solution by adding new principles, in order to get closer to the training solution for a higher-dimensional input. \textbf{Section 6} investigates the function approximation over multiple strict partial orders (theorem 7), with each providing a set of knots to realize the associated piecewise linear function, and embeds the principle of two-sided bases of section 5 into the new solution form (theorem 8). \textbf{Section 7} proposes a fundamental principle called ``continuity restriction'' (theorem 9) and the universal function approximation for higher-dimensional input is finally proved (theorem 10), which completes the theoretical framework of this paper.

\textbf{Section 8} uses experimental results to verify the theory and it is shown that the solution obtained by the back-propagation algorithm can be explained by our theoretical framework. \textbf{Section 9} highlights several conclusions related to the black-box problem. \textbf{Section 10} concludes this paper by a discussion and proposes two open problems for future studies.

The outline above does not contain all the results of this paper. The remaining ones may be an intermediate result or a relatively unimportant conclusion that need not to be included in the main framework.

\section{Approximation to Univariate Function}
The one-sided bases \citep*{Chui1992, Schumaker2007} of splines are nonlocal and hence are not as popular as its further developed version---$B$-splines. However, this ``disadvantage'' happens to be the intrinsic property of a ReLU whose activation function is an one-sided basis. The spline theory tells us that a two-layer ReLU network can realize a piecewise linear function via one-sided bases, from which we obtain a construction method. The main ideas summarized in section 2.3 will be generalized to the multivariate case in later sections 3 and 4.

\subsection{One-Sided Bases}
The notations of splines in this section are borrowed from \citet*{Schumaker2007} with some modifications. On closed interval $I = [0, 1]$, a spline is a piecewise polynomial defined on the subintervals derived from what is called ``knots'' and may satisfy some smoothness conditions.

Let
\begin{equation}
\Delta = \{x_{\nu}: \nu = 1, 2, \dots, \zeta-1\},
\end{equation}
where
\begin{equation}
0 < x_1 < x_2 < \cdots < x_{\zeta-1} < 1
\end{equation}
partition $[0, 1]$ into $\zeta$ subintervals $I_1 = [0, x_1]$ and $I_i = (x_{\mu-1}, x_{\mu}]$ for $\mu = 2, 3, \dots, \zeta$ with $x_{\zeta} = 1$. We call each of $x_1, x_2, \dots, x_{\zeta-1}$ a knot. Denote by  $\mathcal{P}_2$ the set of linear functions (polynomials with degree $1$). The space of continuous linear splines is defined to be
\begin{equation}
\begin{aligned}
\mathfrak{S}_1(\Delta)  =  \{s: s(x)=s_i(x) \in \mathcal{P}_2 \ &\text{for} \ x \in I_i, s_{i}(x_{i})= s_{i+1}(x_{i}) \ \text{for} \ i \ne \zeta, \\& i = 1, 2, \dots, \zeta\}.
\end{aligned}
\end{equation}
We sometimes use the redundant term ``spline function'' to emphasize its function property.

The one-sided bases \citep*{Chui1992,Schumaker2007} of $\mathfrak{S}_1(\Delta)$ can be defined as
\begin{equation}
\{\rho_{j}(x) = \sigma(x - x_{j}): j = -1, 0, \dots, \zeta-1\},
\end{equation}
where $\sigma(x) = \max\{0, x\}$ is the activation function of a ReLU and
\begin{equation}
x_{-1} < x_{0} \le 0 < x_1 < x_2 < \cdots < x_{\zeta-1} < 1.
\end{equation}
Equation 2.4 suggests that a ReLU network can realize a continuous linear spline in terms of the one-sided bases, by which we give a construction method next.

\subsection{Construction of Continuous Linear Splines}
Under equations from 2.1 to 2.3, a continuous linear spline $\mathcal{S}(x) \in \mathfrak{S}_1(\Delta)$ with $\zeta$ linear pieces can be expressed as
\begin{equation}
\mathcal{S}(x) = \{s_i = a_ix + b_i \  \text{for} \ x \in I_i: i = 1, 2, \dots, \zeta\},
\end{equation}
subject to
\begin{equation}
a_ix_{i+1} + b_i = a_{i+1}x_{i+1} + b_{i+1}
\end{equation}
for $i \ne \zeta$ that ensures the continuous property of $\mathcal{S}(x)$ at the knots.

\begin{lem}
To any continuous linear spline $\mathcal{S}(x)$ of equation 2.6, under the one-sided bases of equation 2.4, there exists a unique form
\begin{equation}
\mathcal{S}(x) = \sum_{j = -1}^{\zeta-1}\lambda_{j}\sigma(x - x_{j}),
\end{equation}
where $\lambda_{-1} = (a_1x_{0}+b_1)/(x_{0}-x_{-1})$, $\lambda_0 = (a_1x_{-1}+b_1)/(x_{-1}-x_0)$ and
\begin{equation}
\lambda_{\nu-1} = a_{\nu} - a_{\nu-1}
\end{equation}
for $\nu = 2, 3, \dots, \zeta$.
\end{lem}
\begin{proof}
Suppose that $s_1 = a_1x + b_1$ on $I_1$ has been given and its construction method would be discussed later. Then
\begin{equation}
s_2 = s_1 + (a_2 - a_1)\sigma(x - x_1),
\end{equation}
which is continuous with $s_1$ at knot $x = x_1$. Because $a_2$ can be arbitrarily set, any $s_2$ continuous with $s_1$ at $x_1$ can be expressed in the form of equation 2.10. Similarly, the recurrence form
\begin{equation}
s_{\nu} = s_{\nu-1} + (a_{\nu} - a_{\nu-1})\sigma(x - x_{\nu-1})
\end{equation}
holds for all $\nu = 2, 3, \dots, \zeta$, yielding $\lambda_{\nu-1} = a_{\nu} - a_{\nu-1}$ of equation 2.9, that is,
\begin{equation}
s_{\nu} = s_{\nu-1} + \lambda_{\nu-1}\sigma(x - x_{\nu-1}).
\end{equation}

To the production of $s_1$, by solving the linear equations
\begin{equation}
\begin{cases}
\  \ \ \ \ \lambda_{-1} + \lambda_{0} = a_1
\\
-x_{-1}\lambda_{-1} - x_0\lambda_{0} = b_1
\end{cases}
\end{equation}
derived from $\lambda_{-1}\sigma(x - x_{-1}) + \lambda_{0}\sigma(x - x_0) = a_1x + b_1$ for $x \in [0, x_1]$, we obtain $\lambda_{-1} = (a_1x_{0}+b_1)/(x_{0}-x_{-1})$ and $\lambda_0 = (a_1x_{-1}+b_1)/(x_{-1}-x_0)$.
\end{proof}

\begin{prp}
Any continuous linear spline $\mathcal{S}(x)$ of equation 2.6 with $\zeta$ linear pieces can be realized by a two-layer ReLU network $\mathfrak{N}$ whose hidden layer has $\Theta = \zeta + 1$ units, with infinitely many solutions.
\end{prp}
\begin{proof}
Let $\sigma(w_jx + b_j)$ for $j = -1, 0, \dots, \zeta-1$ be the output of the $j$th unit of the hidden layer of $\mathfrak{N}$. Then $w_jx + b_j = 0$ determines a knot $x_j = -b_j / w_j$ if $w_j \ne 0$. To the case of lemma 1, $w_j$ should be greater than 0. The parameters $w_j$ and $b_j$ can be set according to the associated knot and have infinitely many solutions. If $w_j \ne 1$, the corresponding weight $\lambda_j$ of equation 2.8 should be changed accordingly by multiplying a scale factor.
\end{proof}

Throughout this paper, the function-approximation error is measured by a 2-norm distance. A $C^1$-function $f: [0, 1] \to \mathbb{R}$ is the one whose derivative with respect to $x$ is continuous, which locally approximates a line at arbitrary $x_0 \in [0, 1]$ within a sufficiently small neighborhood of $x_0$.

\begin{thm}[Approximation to $C^1$-functions]
A two-layer ReLU network $\mathfrak{N}$ can approximate an arbitrary $C^1$-function $f: [0, 1] \to \mathbb{R}$ to any desired accuracy, in terms of a continuous linear spline $S(x) \in \mathfrak{S}_1(\Delta)$ realized by the one-sided bases of equation 2.4. If $S(x)$ has $\zeta$ linear pieces, then the number of the units of the hidden layer of $\mathfrak{N}$ required is $\Theta = \zeta + 1$.
\end{thm}
\begin{proof}
Because the location and number of the knots in $[0,1]$ can be arbitrarily set, the conclusion follows from proposition 1.
\end{proof}

\subsection{Ideas for Generalization}
Two main ideas arise from the proof of the preceding results, which will be generalized in sections 3 and 4:
\begin{itemize}
\item[(a)] The knots of equation 2.1 can be arranged in the strict partial order of equation 2.2, such that all the one-sided bases could exert their influences in one direction and each subinterval has its own distinguished activated unit to shape its linear function.

\item[(b)] The adjacent continuous linear pieces whose subdomains are separated by a knot has a simple relation of equation 2.12, contributing to a construction method of the weights of the one-sided bases.
\end{itemize}

\section{Strict Partial Order of Knots}
This section generalizes the knots and their ``less than'' relation on an one-dimensional line to the higher-dimensional case. In this paper, when referring to the $n$-dimensional space $\mathbb{R}^n$ or the universal set $U = [0, 1]^n$, if the range of interger $n$ is not explicitly defined, we tacitly assume that $n \ge 1$.

\subsection{Definition of Strict Partial Orders}
In space $\mathbb{R}^n$, a region is part of $\mathbb{R}^n$ separated by a set of $n-1$-dimensional hyperplanes and we here give its rigorous definition. The concepts of ``limit point'' and ``closure'' are from mathematical analysis \citep*{Rudin1976}. A limit point $p$ of a set $D$ is a point whose every neighborhood contains a point $q \ne p$ with $q \in D$. Denote by $D'$ the set of the limit points of $D$. The closure of $D$ is the set $\bar{D} = D \cup D'$. In \citet*{Stanley2012}'s definition, a region is an open set without containing its boundary and this will affect the definition of a continuous function on more than one region. That's why we introduce the above two concepts to modify \citet*{Stanley2012}'s definition as follows.

\begin{dfn}[Region]
A region of $\mathbb{R}^n$ is the closure of a connect component of $\mathbb{R}^n - \mathcal{H}$, where $\mathcal{H}$ is a set of $n-1$-dimensional hyperplanes.
\end{dfn}

The output $\sigma(\boldsymbol{w}^T\boldsymbol{x} + b)$ of a ReLU determines an $n-1$-dimensional hyperplane $\boldsymbol{w}^T\boldsymbol{x} + b = 0$ of $\mathbb{R}^n$, denoted by $l$, which divides $\mathbb{R}^n$ into two parts $l^+$ and $l^0$, representing the nonzero (positive)-output and zero-output regions, respectively. The notations $l^+$ and $l^0$ will be used throughout the rest of this paper. For convenience, we sometimes call $\sigma(\boldsymbol{w}^T\boldsymbol{x} + b)$ the output of hyperplane $l$.

\begin{figure}[!t]
\captionsetup{justification=centering}
\centering
\includegraphics[width=2.3in, trim = {4.3cm 2.6cm 4cm 3.0cm}, clip]{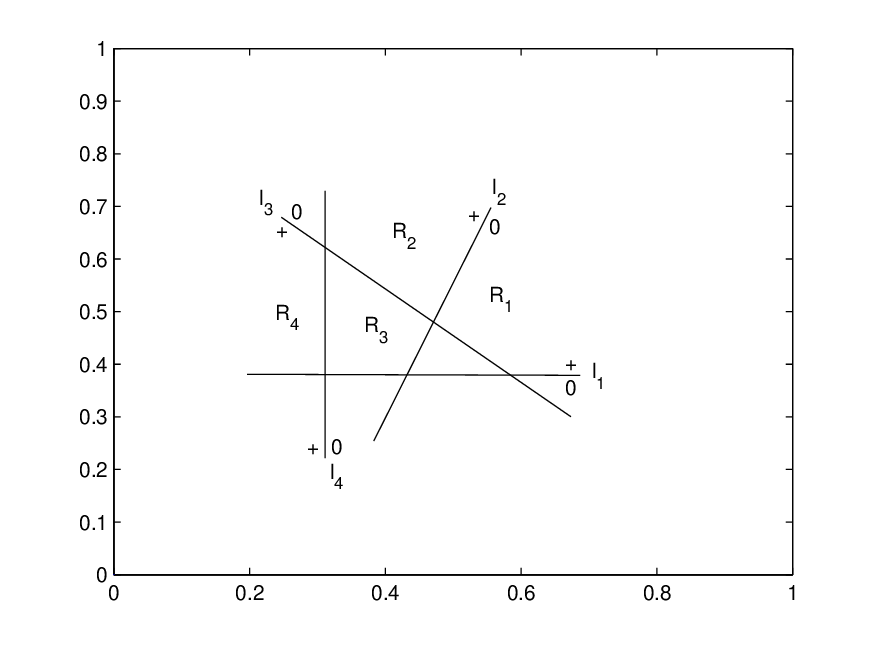}
\caption{Strict partial order of knots.}
\label{Fig.1}
\end{figure}

\begin{dfn}[Knot]
A knot of $\mathbb{R}^n$ is an $n-1$-dimensional hyperplane that partitions $\mathbb{R}^n$ into two parts. The term ``$n-1$-knot'' is to emphasize the dimensionality $n-1$.
\end{dfn}

\begin{dfn}[Strict partial order of knots]
Denote by $\mathcal{H} = \{l_i: i = 1, 2, \dots, \zeta\}$ a set of $n-1$-dimensional hyperplanes of $\mathbb{R}^n$ for $n \ge 2$. Let $\mathcal{R} = \{R_i, l_i^+, l_i^0: i = 1, 2, \dots, \zeta\}$, where $R_i$'s are from the regions formed by $\mathcal{H}$. Suppose that the following condition
\begin{equation}
\begin{aligned}
\mathscr{F}(\mathcal{R}) := \{R_{1} \subset l_{1}^+, & \ R_{\nu}: \bigcup_{j = 1}^{\nu-1}R_{j} \subseteq l_{\nu}^0, \ R_{\nu} \subset \bigcap_{\mu = 1}^{\nu}l_{\mu}^+, \  q_{\nu} = (R_{\nu} \cap R_{\nu-1}) \subseteq l_{\nu}, \\&
\ \dim(q_{\nu}) = n-1,\  \nu = 2, 3, \dots, \zeta\},
\end{aligned}
\end{equation}
is satisfied, where $\dim$ denotes dimensionality. Then if $i_1 < i_2$ for $1 \le i_1, i_2 \le \zeta$, we say that
\begin{equation}
l_{i_1} \prec l_{i_2}.
\end{equation}
For completeness of the theory, if there exists only one single hyperplane $l_1$, we formally write $l_1 \prec l_{\infty}$. Under equation 3.1, $R_i$'s for all $i = 1, 2, \dots, \zeta$ are called ordered regions. Write $\mathscr{P} := l_1 \prec l_2 \prec \dots l_{\zeta}$ or $\mathscr{P}: = (\mathcal{H}, \prec)$ and we sometimes simply call $\mathscr{P}$ an \textsl{order}.
\end{dfn}

\noindent
\textbf{Example}. In Figure \ref{Fig.1}, we have $l_1 \prec l_2 \prec l_3 \prec l_4$.

\begin{prp}
The relation $\prec$ of equation 3.2 is a strict partial order, and is equivalent to ``$<$''(less than) of real numbers in the sense that both the relationships between $R_i$'s for $i = 1, 2, \dots, \zeta$ and the influences of the units of $l_i$'s on $R_i$'s in terms of equation 3.1 are the same as the one-dimensional case manifested by equations 2.2 and 2.5---that is, $\mathscr{F}(\mathcal{I})$ of equation 3.1 also holds for the one-dimensional input, where $\mathcal{I} = \{I_{i}, x_i^+, x_i^0: i = 1, 2, \dots, \zeta\}$ in which $I_{i} = (x_i, x_{i + 1}]$ and $x_i$ is from equation 2.2.
\end{prp}
\begin{proof}
A strict partial order \citep*{Davey2002} is a relation that is not reflexive, antisymmetric and transitive. The relation $\prec$ is defined based on the ``less than'' relation $<$ of the subscripts of the $l_i$'s of equation 3.1. Because $<$ is a strict partial order, so is $\prec$.

$\mathscr{F}(\mathcal{I})$ describes the relationship that each interval $I_{i} = (x_i, x_{i + 1}]$ is on the left side of point $x_j$ for $j \ge i+1$ (i.e., $I_i \subset \bigcap_{j = i+1}^{\zeta}x_j^0$) and is on the right side of $x_{k}$ for $ k \le i$ (i.e., $I_i \subset \bigcap_{k = 1}^{i}x_k^+$). $\mathscr{F}(\mathcal{R})$ represents the similar meaning for regions, for which the ``right side'' and ``left side''of $l_i$ are substituted by $l_i^+$ and $l_i^0$, respectively, and the term $q_{\nu} = (R_{i} \cap R_{i-1}) \subseteq l_i$ with $\dim(q_{\nu}) = n-1$ is the generalized version of $\bar{I}_i \cap \bar{I}_{i-1} = x_i$.
\end{proof}

\subsection{Construction of Strict Partial Orders}
\begin{dfn}[Region of the universal set]
A region $R$ of $U = [0, 1]^n$ with respect to a set $\mathcal{H}$ of hyperplanes is the intersection of $U$ and a region $R'$ derived from $\mathcal{H}$, provided that $R \ne \emptyset$ and $\dim{R} = n$; that is, $R = U \cap R'$. Since $U$ itself is a region produced by some hyperplanes, $R$ is also a region. We also say that $R$ is a region of $U$ formed by $\mathcal{H}$.
\end{dfn}

\begin{prp}
To $U = [0, 1]^n$, there exists a set $\mathcal{H}$ of $n-1$-dimensional hyperplanes $l_i$'s for $i = 1, 2, \dots, \zeta$ to form a strict partial order of equation 3.1, with the ordered regions $R_i$'s of $U$ with respect to $\mathcal{H}$ satisfying $\bigcup_{i = 1}^{\zeta}R_i = U$.
\end{prp}
\begin{proof}
In the two-dimensional case of $n = 2$, it is easy to imagine two kinds of solutions. Let $U \subset l_1^+$ and then rotate $l_1$ around any point of $l_1$ counterclockwise or clockwise to form $l_{\nu}$'s for $\nu = 2, 3, \dots, \zeta$, each of which intersects $U$ at more than one point. Then a region $R_{\mu}$ between $l_{\mu}$ and $l_{\mu + 1}$ for $\mu = 1, 2, \dots, \zeta-1$ together with $R_\zeta \subset l_{\zeta}^+ \cap U$ could fulfil equation 3.1 and simultaneously $\bigcup_{i = 1}^{\zeta}R_i = U$. The other way is to translate $l_1^+$ to pass though $U$ by $\zeta-1$ steps, with each intersecting $U$ at more than one point, which can also form $\zeta$ regions satisfying equation 3.1 whose union is $U$. The second method can be generalized to the $n$-dimensional case for $n \ge 3$.
\end{proof}

\begin{prp}
Under some notations of proposition 3, more than one strict partial order $\mathscr{P}_i = (H_i, \prec)$ for $i = 1, 2, \dots, \psi$ with $\psi \ge 2$ could simultaneously exist, with the constraint that $\bigcup_{i = 1}^{\psi}H_i = \mathcal{H}$, $H_{\nu} \cap H_{\mu} = \emptyset$ for $1 \le \nu, \mu \le \psi$ and $\nu \ne \mu$, and that $\bigcup_{i = 1}^{\psi}\bigcup_{j = 1}^{\phi_i}R_j^{(i)}\subseteq U$, where $\{R_j^{(i)}: j = 1, 2, \dots, \phi_i = |H_i|\}$ is the set of the ordered regions of $U$ formed by $H_i$.
\end{prp}
\begin{proof}
For example, based on the proof of proposition 3, arbitrarily select $l_{\mu}$ for $2 \le \mu \le \zeta-1$. Exchange the positive- and zero-output regions of each $l_{\kappa}$ for $1 \le \kappa \le \mu$ by converting its equation $\boldsymbol{w}_{\kappa}^T\boldsymbol{x} + b_{\kappa} = 0$ into $-\boldsymbol{w}_{\kappa}^T\boldsymbol{x} - b_{\kappa} = 0$. Then two independent strict partial orders are formed, including $l_{\mu} \prec l_{\mu-1} \prec \dots \prec l_2$ and $l_{\mu + 1} \prec l_{\mu+2} \prec \dots \prec l_{\zeta}$. The union of the ordered regions of the two strict partial orders is $U - R_{\mu}$, a subset of $U$.
\end{proof}

\begin{rmk}
The existence proof of this proposition is far from complete in constructing more than one strict partial order. This conclusion is better to be regarded as the description of a more general phenomenon and we'll give more examples in section 6.
\end{rmk}

\section{Approximation over One Strict Partial Order}
This section first establishes the relationship between the linear pieces of a piecewise linear function of a two-layer ReLU network (theorem 2) analogous to equation 2.12, after which arbitrary piecewise linear function on a set of ordered regions can be constructed (theorem 3), completing the generalization of section 1.
\subsection{Axiomatic Foundation}
\begin{dfn}[A two-layer ReLU network]
The output of a two-layer ReLU network $\mathfrak{N}$ with $n$-dimensional input is defined as
\begin{equation}
y = \sum_{i = 1}^m\lambda_i\sigma(\boldsymbol{w}_i^T\boldsymbol{x} + b_i),
\end{equation}
where $\sigma(x) = \max\{0, x\}$ is the activation function of a ReLU. In equation 4.1, all the ReLUs $\mathscr{U}_i$'s form the hidden layer of $\mathfrak{N}$ and $\lambda_i$ is the \textsl{output weight} of the $i$th one. The $n$-dimensional input space is denoted by $\mathbb{R}^n$.
\end{dfn}

\begin{ax}
A ReLU $\mathscr{U}_i$ or its associated hyperplane $\mathcal{L}_i$ with equation $\boldsymbol{w}_i^T\boldsymbol{x} + b_i = 0$ of the hidden layer of $\mathfrak{N}$ could influence half of the input space $\mathbb{R}^n$, denoted by $\mathcal{L}_i^+$, through its nonzero positive output, and has no impact on the other half $\mathcal{L}_i^0$ in terms of its zero output.
\end{ax}

\begin{ax}
The output of a ReLU $\mathscr{U}_i$ on $\mathcal{L}_i^+$ is a linear function $y_i = \boldsymbol{w}_i^T\boldsymbol{x} + b_i$.
\end{ax}

\begin{dfn}[Continuity at a knot]
A function $f: \mathbb{R}^n \to \mathbb{R}$ is said to be continues at knot $\mathcal{L}$, if
\begin{equation}
\lim_{\boldsymbol{x} \to \mathcal{L}_+}f(\boldsymbol{x}) = \lim_{\boldsymbol{x} \to \mathcal{L}_0}f(\boldsymbol{x}) = f(\boldsymbol{x}_{\mathcal{L}}),
\end{equation}
where $\boldsymbol{x} \to \mathcal{L}_+$ and $\boldsymbol{x} \to \mathcal{L}_0$ mean that $\boldsymbol{x}$ approaches knot $\mathcal{L}$ from the parts $\mathcal{L}^+$ and $\mathcal{L}^0$, respectively, through an arbitrary one-dimensional line $\phi$ satisfying $\phi \cap \mathcal{L} \ne \emptyset$ and $\phi \not\subset \mathcal{L}$, and $\boldsymbol{x}_{\mathcal{L}}$ is the point $\phi \cap \mathcal{L}$. If equation 4.2 doesn't hold, we say that $f(\boldsymbol{x})$ is discontinuous at knot $\mathcal{L}$.
\end{dfn}

\subsection{Correlation between Linear Pieces}
\begin{lem}
Let $\mathscr{L}$ and $l_{\nu}$ be an $n-1$- and $\nu$-dimensional hyperplane of $\mathbb{R}^n$, respectively, where $1 \le \nu \le n-1$ and $n \ge 2$. If $l_{\nu} \nsubseteq \mathscr{L}$ and $l_{\nu} \cap \mathscr{L} \ne \emptyset$, then their intersection $l = \mathscr{L} \cap l_{\nu}$ is a $\nu-1$-dimensional hyperplane.
\end{lem}
\begin{proof}
We write the parametric equation of $l_{\nu}$ as
\begin{equation}
\boldsymbol{x} = \boldsymbol{x}_O + \sum_{i = 1}^{\nu}\alpha_{i}\boldsymbol{r}_{i}.
\end{equation}
The equation of $\mathscr{L}$ is $\boldsymbol{w}^T\boldsymbol{x} + b = 0$. Substituting the equation of $l_{\nu}$ into that of $\mathscr{L}$ yields $\sum_{i = 1}^{\nu}t_{i}\alpha_{i} = -(\boldsymbol{w}^T\boldsymbol{x}_O + b)$, where $t_i = \boldsymbol{w}^T\boldsymbol{r}_i$. Not all of $t_i$'s are zero; for otherwise, $\boldsymbol{w} \perp l_{\nu}$, implying $l_{\nu} \parallel \mathscr{L}$ or $l_{\nu} \subset \mathscr{L}$, which is a contradiction. Then select one $t_{\mu} \ne 0$ for $1 \le \mu \le \nu$ such that
\begin{equation}
\alpha_{\mu} = \big(-(\boldsymbol{w}^T\boldsymbol{x}_O + b) - \sum_{i \ne \mu}t_{i}\alpha_{i}\big)/t_{\mu}.
\end{equation}
By equations 4.3 and 4.4, we obtain the equation $\boldsymbol{x} = \boldsymbol{x}'_O + \sum_{i \ne \mu}\alpha_{i}\boldsymbol{r}'_{i}$ of $l = \mathscr{L} \cap l_{\nu}$, where $\boldsymbol{x}'_O = \boldsymbol{x}_O -(\boldsymbol{w}^T\boldsymbol{x}_O + b)/t_{\mu}\cdot\boldsymbol{r}_{\mu}$ and $\boldsymbol{r}'_{i} = \boldsymbol{r}_{i} - t_i/t_{\mu}\cdot\boldsymbol{r}_{\mu}$, which is a $\nu - 1$-dimensional hyperplane.
\end{proof}

\begin{figure}[!t]
\captionsetup{justification=centering}
\centering
\subfloat[One-dimensional case.]{\includegraphics[width=2.3in, trim = {4.0cm 2.5cm 3.5cm 1.5cm}, clip]{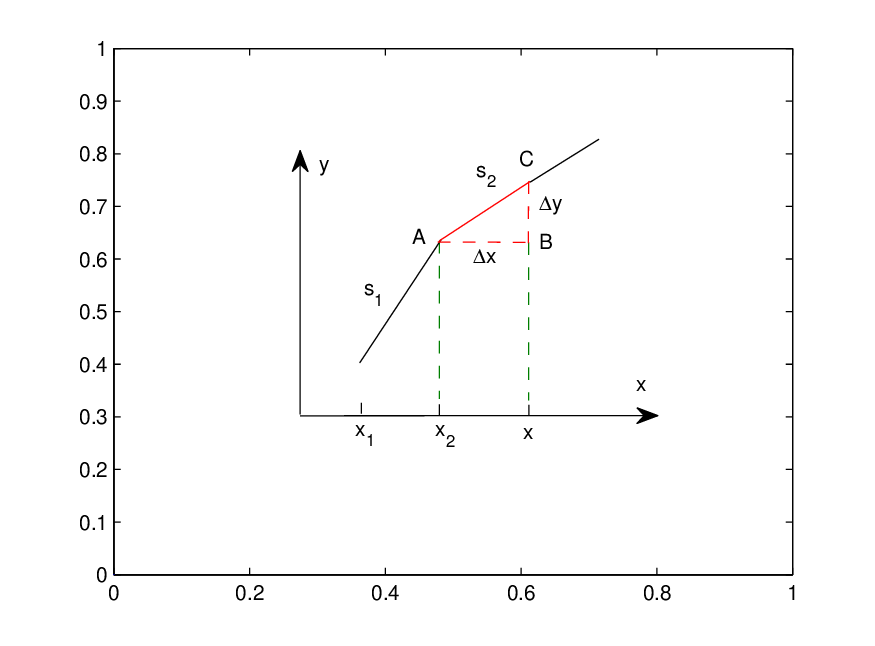}} \quad \quad \quad \quad
\subfloat[Two-dimensional case.]{\includegraphics[width=2.8in, trim = {3.0cm 1.5cm 1.6cm 1.0cm}, clip]{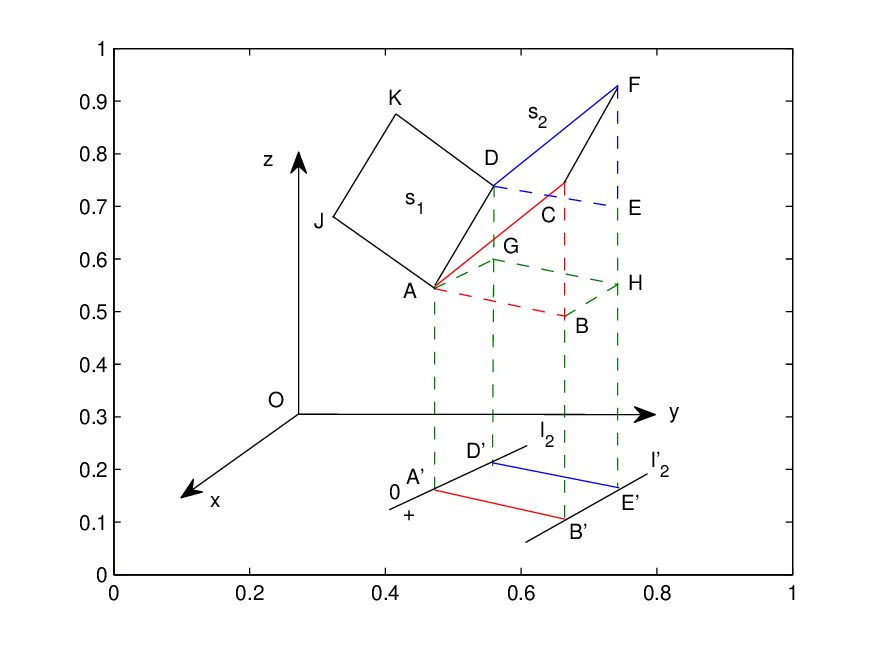}}
\caption{Correlation between linear pieces.}
\label{Fig.2}
\end{figure}

In definition 1, we introduced the closure operation of a set. When constructing a continuous piecewise linear function on some regions $R_i$'s for $i = 1, 2, \dots, \zeta$, the function values on the boundary of $R_i$ (see later definition 19) could be easily set. For example, in the following theorem 2, $f(\boldsymbol{x})$ on $R_1 \cap R_2$ can be simply classified into the case of $R_1$. Thus, for simplicity of descriptions, unless otherwise stated, we will not mention the boundary case hereafter.

\begin{thm}
Denote by $R_1$ and $R_2$ the two regions of $\mathbb{R}^n$ separated by knot $\mathscr{L}$ whose equation is $\boldsymbol{w}^T\boldsymbol{x} + b = 0$, with $R_1 = \mathscr{L}^0$ and $R_2 = \mathscr{L}^+$. Suppose that $f: (R_1 \cup R_2) \to \mathbb{R}$ is a piecewise linear function with two pieces and that the linear function
\begin{equation}
s_1 = \boldsymbol{w}_1^T\boldsymbol{x} + b_1
\end{equation}
for $\boldsymbol{x} \in R_1$ has been fixed. Then any continuous $f(\boldsymbol{x})$ can be expressed in this form
\begin{equation}
f(\boldsymbol{x}) = s_1(\boldsymbol{x}) + \lambda\sigma(\boldsymbol{w}^T\boldsymbol{x} + b),
\end{equation}
where the parameter $\lambda$ can determine arbitrary linear function $s_2 = \boldsymbol{w}_2^T\boldsymbol{x} + b_2$ on $R_2$ that is continuous with $s_1$ at knot $\mathscr{L}$. Conversely, if a piecewise linear function $f(\boldsymbol{x})$ is in the form of equation 4.6, it is continuous at $\mathscr{L}$. Equation 4.6 is equivalent to
\begin{equation}
s_2(\boldsymbol{x}) = s_1(\boldsymbol{x}) + \lambda\sigma(\boldsymbol{w}^T\boldsymbol{x} + b)
\end{equation}
for $\boldsymbol{x} \in R_2$. To a certain $s_2$, the parameter $\lambda$ of equation 4.7 is unique, if the equation $\boldsymbol{w}^T\boldsymbol{x} + b = 0$ of the knot is fixed; that is, the map between $\lambda$ and $s_2$ is bijective.
\end{thm}
\begin{proof}
We first prove the converse conclusion. Let $\boldsymbol{x}_0 \in \mathscr{L}$ be an arbitrary point of $\mathscr{L}$. Then
\begin{equation}
\begin{aligned}
\lim_{\alpha \to 0^+}f(\boldsymbol{x}_0 &+ \alpha \boldsymbol{d}) = \lim_{\alpha \to 0^-}f(\boldsymbol{x}_0 + \alpha \boldsymbol{d}) = f(\boldsymbol{x}_0) \\
&= s_1(\boldsymbol{x}_0) + \lambda\sigma(\boldsymbol{w}^T\boldsymbol{x}_0 + b)\\
&= \boldsymbol{w}_1^T\boldsymbol{x}_0 + b_1,
\end{aligned}
\end{equation}
where $\boldsymbol{d}$ is any direction of a line $l$ that intersects $\mathscr{L}$ with $l \cap \mathscr{L} \ne l$, and $\boldsymbol{w}_1$ and $b_1$ are from equation 4.5, which proves the continuity of $f(\boldsymbol{x})$ at $\mathscr{L}$.

The equivalence of equations 4.6 and 4.7 is obvious, since they can be deduced from each other. Thus, it's sufficient to prove equation 4.7. The remaining proof is composed of three parts, corresponding to the one-, two- and $n$-dimensional cases for $n \ge 3$, respectively. The first two parts provide the intuitive idea to be generalized to a higher-dimensional input space.

Throughout the proof, we tacitly assume that the representation of an angle between two lines or vectors is under the definition of the two-dimensional Cartesian coordinate system. For convenience, the notations of subdomains $R_1$ and $R_2$ of $f(\boldsymbol{x})$ are sometimes substituted with $\mathscr{L}^0$ and $\mathscr{L}^+$, respectively.

\textbf{One-dimensional input}. As shown in Figure \ref{Fig.2}a, to a fixed linear function $s_1(x) = a_1x + b_1$ on $[x_1, x_2]$, we use a geometric approach different from lemma 1 to prove that any $s_2$ continuous with $s_1$ at $x_2$ can be expressed as $s_2 = a_2x_2 + b_2 = s_1 + \lambda\sigma(x - x_2)$. Let $y_2 = s_1(x_2)$. In Figure \ref{Fig.2}a, $\Delta x = x - x_2$. By the right triangle $\Delta ABC$, to any $x \in (x_2, +\infty)$, we have
\begin{equation}
s_2(x) = y_2 + \Delta y = y_2 + a_2\Delta x = y_2 + a_2\sigma(x - x_2),
\end{equation}
where $a_2 = \tan\angle CAB$. Because
\begin{equation}
y_2 = s_1(x-\Delta x) = s_1(x) - a_1\Delta x.
\end{equation}
Equations 4.9 and 4.10 give
\begin{equation}
s_2 = s_1 + \lambda\Delta x = s_1 + \lambda\sigma(x - x_2),
\end{equation}
where $\lambda = a_2 - a_1$, the same as equation 2.10. Let $\theta_1$ be the angle between $s_1$ and $x$-axis and $\theta_2 = \angle CAB$. Then $\lambda$ of equation 4.11 can be expressed as
\begin{equation}
\lambda(\theta_1, \theta_2) = \tan\theta_2 - \tan\theta_1.
\end{equation}

\textbf{Two-dimensional input}. The proof is illustrated by an example of Figure \ref{Fig.2}b, in which two linear pieces $s_1$ and $s_2$ are connected at knot $l_2$. The notation of three non-collinear points (e.g., $xOy$ or $B'A'A$) without prefix ``$\Delta$'' or ``$\angle$'' denotes the plane formed by them. Throughout the proof, the term ``projection'' means \textsl{orthogonal projection}, which involves the foot of a perpendicular to a line or hyperplane passing through a point. For instance, $A'$ is the projection of $A$ on $xOy$, due to $AA' \perp xOy$.

By an example, the notations $AD$, $\bar{A}D$ and $\vec{A}D$ correspond to the segment, line and ray associated with points $A$ and $D$, respectively. Line $\bar{A}D = s_1 \cap s_2$ and its projection on plane $xOy$ is the knot $l_2$ denoted by equation
\begin{equation}
\boldsymbol{w}_2^T\boldsymbol{x} + b_2 = 0.
\end{equation}
We first examine the function value $s_2(B')$ on point $B' \in R_2 = l_2^+$, with $B'$ subject to $B'A' \perp l_2$. Translate $l_2$ into a new position $l'_2$ passing through $B'$---that is, $l_2' \parallel l_2$ and $B' \in l_2'$. By the construction of Figure \ref{Fig.2}b, we have $BB' \perp xOy$ and $AB \parallel A'B'$. Then $AB \perp BC$ or $\Delta ABC$ is a right triangle.

Now we make an analogy between the two red triangles of Figures \ref{Fig.2}a and \ref{Fig.2}b, respectively. Let $s^{(A)}_2 = \vec{A}C = (s_2 \cap B'A'A)$ be a linear function defined on ray $\vec{A'}B' \subset xOy$, corresponding to $s_2$ of Figure \ref{Fig.2}a. Denote by $s^{(A)}_1 = \vec{A}J = s_1 \cap B'A'A$, which is the counterpart of $s_1$ of Figure \ref{Fig.2}a.

We construct a two-dimensional coordinate system denoted by $\mathcal{C}_{A'}$ in plane $B'A'A$ analogous to the one of Figure \ref{Fig.2}a, with $A'$, $\bar{A'}B'$ and $\bar{A'}A$ bing the origin, $x$- and $y$-axis, respectively. Under the two-dimensional coordinate system $\mathcal{C}_{A'}$, let $x_A \in \vec{A'}B'$ be an arbitrary point of the positive $x$-axis. Point $x_A$ also has a vector representation $\boldsymbol{x}_A$ in the original three-dimensional coordinate system, for which there exists an one-to-one correspondence
\begin{equation}
\mathscr{F}(x_A) = \boldsymbol{x}_A.
\end{equation}
Due to geometric relationships $A'A \parallel Oz$ and $A'B' \subset xOy$, we also have
\begin{equation}
s_i(\boldsymbol{x}_A) = s^{(A)}_i(x_A)
\end{equation}
for $i = 1, 2$.

Then in plane $B'A'A$, to any point $B'$ on ray $\vec{A'}B'$, appling the one-dimensional results of equations 4.11 and 4.12 yields
\begin{equation}
s^{(A)}_2(x_A) = s^{(A)}_1(x_A) + \alpha(\theta_{A_1}, \theta_{A_2})\Delta x_A,
\end{equation}
where
\begin{equation}
\Delta x_A = A'B' = AB
\end{equation}
and
\begin{equation}
\alpha(\theta_{A_1}, \theta_{A_2}) = \tan\theta_{A_2} - \tan\theta_{A_1},
\end{equation}
in which $\theta_{A_2} = \angle BAC$ and $\theta_{A_1}$ is the angle between $s^{(A)}_1$ and $AB$.

Note that in equation 4.17, by Figure \ref{Fig.2}b, $A'B'$ equals the distance from point $B'$ to line $l_2$, that is,
\begin{equation}
A'B' = \frac{|\boldsymbol{w}_2^T\boldsymbol{x}_A + b_2|} {\Vert \boldsymbol{w}_2 \Vert_2} = \mathfrak{C}\sigma(\boldsymbol{w}_2^T\boldsymbol{x}_A + b_2),
\end{equation}
where $\boldsymbol{x}_A$ is the coordinate vector of $B'$ and
\begin{equation}
\mathfrak{C} = 1/\Vert \boldsymbol{w}_2 \Vert_2
\end{equation}
is a constant determined by $l_2$ of equation 4.13, in which 2-norm $\Vert \boldsymbol{x} \Vert_2 = \big(\sum_{i=1}^nx_i^2\big)^{1/2}$ is used. Equations from 4.15 to 4.20 give
\begin{equation}
s_2(\boldsymbol{x}_A) = s_1(\boldsymbol{x}_A) + \lambda_A\sigma(\boldsymbol{w}_2^T\boldsymbol{x}_A + b_2),
\end{equation}
where
\begin{equation}
\boldsymbol{x}_A \in \vec{A'}B'
\end{equation}
and
\begin{equation}
\lambda_A = \mathfrak{C}\alpha(\theta_{A_1}, \theta_{A_2})
\end{equation}
with $\alpha(\theta_{A_1}, \theta_{A_2})$ from equation 4.18, which is similar to equation 4.12 of the one-dimensional case.

Next, we investigate another point $E'$ on $l_2'$ of Figure \ref{Fig.2}b. Similar to point $B'$, the following conditions are satisfied in Figure \ref{Fig.2}b: $DD'\perp xOy$, $D'E' \perp l_2$, $FE' \perp xOy$, $DE \parallel GH \parallel D'E'$, and $AG \parallel BH$. Then $\Delta DEF$ is congruent to $\Delta ABC$, with $\angle DEF$ as a right angle, and thus
\begin{equation}
DE = AB.
\end{equation}
Because $A'B' \perp l_2$ and $AB \parallel A'B'$, together with the conditions above, we get
\begin{equation}
DE \parallel AB.
\end{equation}

Let $s^{(D)}_1 =  s_1 \cap E'D'D = \vec{D}K$ be a linear function defined on $\bar{D'}E' \cap l_2^0$ and $s^{(D)}_2 = s_2 \cap E'D'D = \vec{D}F$. Denote by $\theta_{D_1}$ the angle between $s^{(D)}_1$ and $DE$. Write $\theta_{D_2} = \angle EDF$. In plane $E'D'D$, applying the method of equation 4.21, we have
\begin{equation}
s_2(\boldsymbol{x}_D) = s_1(\boldsymbol{x}_D) + \lambda_D\sigma(\boldsymbol{w}_2^T\boldsymbol{x}_D + b_2),
\end{equation}
where
\begin{equation}
\boldsymbol{x}_D \in \vec{D'}E'
\end{equation}
and
\begin{equation}
\lambda_D = \mathfrak{C}\alpha(\theta_{D_1}, \theta_{D_2})
\end{equation}
in which
\begin{equation}
\alpha(\theta_{D_1}, \theta_{D_2}) = \tan\theta_{D_2} - \tan\theta_{D_1}.
\end{equation}

We now prove
\begin{equation}
\theta_{D_i} = \theta_{A_i}
\end{equation}
for $i = 1, 2$. In fact, $E'D'D \parallel B'A'A$ implies
\begin{equation}
(s^{(D)}_1 = E'D'D \cap s_1) \parallel (s^{(A)}_1 = B'A'A \cap s_1),
\end{equation}
together with $DE \parallel AB$ resulting in $\theta_{D_1} = \theta_{A_1}$. Equation $\theta_{D_2} = \theta_{A_2}$ holds due to congruence of $\Delta DEF$ and $\Delta ABC$.

Equations 4.23, 4.28 and 4.30 imply that
\begin{equation}
\lambda_D = \lambda_A := \lambda
\end{equation}
is a constant in equations 4.21 and 4.26.

As shown in Figure \ref{Fig.2}b, because $E'$ is arbitrarily selected from $l'_2$, the parameter $\lambda$ remains a constant at each point of $l'_2$ and this conclusion holds for any translated $l_2'$ satisfying $l_2' \parallel l_2$ and $l_2' \subset l_2^+$. To the other dimension, when fixing $A'$ and changing $B'$ along the ray $\vec{A'}B'$, $\lambda$ is also a constant as the one-dimensional case. When $A'$ runs over all the points of $l_2$, the set of $\boldsymbol{x}_A$ of equation 4.21 can include each point of $l_2^+$, that is,
\begin{equation}
l_2^+ \subseteq \{\boldsymbol{x}_A: A' \in l_2, \boldsymbol{x}_A \in \vec{A'}B'\}.
\end{equation}
Equations 4.21, 4.32 and 4.33 prove the two-dimensional result
\begin{equation}
s_2(\boldsymbol{x}) = s_1(\boldsymbol{x}) + \lambda\sigma(\boldsymbol{w}^T\boldsymbol{x} + b)
\end{equation}
for $\boldsymbol{x} \in l_2^+$ of equation 4.7.

Finally, we prove the uniqueness of parameter $\lambda$ for a certain $s_2$ of equation 4.34. The proof is by the further explanation of the geometric meaning of $\lambda$. To equation 4.30, let $\theta_{D_1} = \theta_{A_1} := \Theta$, which is a constant for the fixed $s_1$. And let $\theta_{D_2} = \theta_{A_2} := \theta$, which is a variable changing with different $s_2$. Thus, by equations 4.28, 4.29 and 4.32, we can write
\begin{equation}
\lambda = \mathscr{G}(\theta) = \mathfrak{C}(c - \tan\theta),
\end{equation}
where $\mathfrak{C}$ and $c = \tan\Theta$ are both constants, for which $\lambda$ is determined only by $\theta$.

By the example of Figure \ref{Fig.2}b, $\theta = \angle CAB$. Because the direction of $AB$ is fixed, $\theta$ varies according to the direction of
\begin{equation}
\vec{A}C = s_2 \cap B'A'A,
\end{equation}
where plane $B'A'A$ is also fixed. So different $s_2$ leads to different $\theta$, contributing to different $\lambda$. This proves the one-to-one correspondence between $s_2$ and $\lambda$.

\textbf{Case of input dimensionality $n \ge 3$}. Although a geometric object of $\mathbb{R}^n$ cannot be visualized as above, we can use algebraic ways to generalize the preceding results. The main idea is to reduce the problem to the one-dimensional case by constructing two-dimensional planes embedded in $\mathbb{R}^n$, analogous to the method of $\mathbb{R}^2$. To achieve this goal, we should ensure that some related concepts, relationships or conclusions of $\mathbb{R}^3$ are applicable to $\mathbb{R}^{n+1}$. They are listed below along with the short arguments, emphasizing their irrelevance to the input dimensionality.

\begin{itemize}
\item[\Romannum{1}.] There exists a unique perpendicular line $\Lambda_P$ to an $n$-dimensional hyperplane $l \subset \mathbb{R}^{n+1}$ passing through a certain point $P$. Proof: Letting $\boldsymbol{w}^T\boldsymbol{x} + b = 0$ be the equation of $l$, the vector $\boldsymbol{w}$ perpendicular to $l$ is unique up to a scale factor. We can regard $\boldsymbol{w}$ as the direction of an one-dimensional line embedded in $\mathbb{R}^{n+1}$. The line passing through point $P$ with the direction $\boldsymbol{w}$ is unique, which is the $\Lambda_P$.

\item[\Romannum{2}.] The distance from a point $\boldsymbol{x}_0$ to an $n$-dimensional hyperplane $l$ with equation $\boldsymbol{w}^T\boldsymbol{x} + b = 0$ of $\mathbb{R}^{n+1}$ is
\begin{equation}
\mathscr{D} = \frac{|\boldsymbol{w}^T\boldsymbol{x}_0 + b|} {\Vert \boldsymbol{w} \Vert_2}.
\end{equation}
Proof: Let $\boldsymbol{x}_0'$ be the orthogonal-projection point of $\boldsymbol{x}_0$ on $l$. Then $\boldsymbol{x}_0 - \boldsymbol{x}_0' = \mathscr{D} \cdot \boldsymbol{w} / \Vert \boldsymbol{w} \Vert_2$, which follows
\begin{equation}
\boldsymbol{w}^T(\boldsymbol{x}_0 - \boldsymbol{x}_0') = \mathscr{D} \cdot \boldsymbol{w}^T\boldsymbol{w} / \Vert \boldsymbol{w} \Vert_2.
\end{equation}
Since $\boldsymbol{x}_0' \in l$, $\boldsymbol{w}^T\boldsymbol{x}_0' + b = 0$, and this together with the equation 4.38 imply equation 4.37.

\item[\Romannum{3}.] The distance between two parallel $n$-dimensional hyperplanes $l \parallel l'$ of $\mathbb{R}^{n+1}$ is a constant, equal to the distance between $l$ and any point of $l'$. Proof: Denote by $\boldsymbol{w}^T\boldsymbol{x} + b = 0$ the hyperplane of $l$. Because $l' \parallel l$, the equation of $l'$ can be expressed as $\boldsymbol{w}^T\boldsymbol{x} + b + t = 0$, where $t$ is a constant. The distance from $\boldsymbol{x}' \in l'$ to $l$ is
\begin{equation}
d(\boldsymbol{x}') = \frac{|\boldsymbol{w}^T\boldsymbol{x}' + b|} {\Vert \boldsymbol{w} \Vert_2} = |t| / \Vert \boldsymbol{w} \Vert_2,
\end{equation}
which is also a constant.

\item[\Romannum{4}.] Let $p$ be a plane embedded in $\mathbb{R}^{n+1}$ and $l \subset \mathbb{R}^{n+1}$ an $n$-dimensional hyperplane satisfying $p \nsubseteq l$ and $p \cap l \ne \emptyset$. Then their intersection $q = p \cap l$ is an one-dimensional line. Proof: This conclusion is by lemma 2.
\end{itemize}

We continue to use the notations of the two-dimensional example of Figure \ref{Fig.2}b but regard the vectors and some geometric objects as being $n$- or $n-1$-dimensional instead. The main idea is to construct the right triangles as the red one of Figure \ref{Fig.2}a in a higher-dimensional space analogous to Figure \ref{Fig.2}b. Please bear in mind the example of Figure \ref{Fig.2}b and pay attention to the similarity as well as the difference between the two- and $n$-dimensional cases.

Let $\mathfrak{R} := \mathbb{R}^n$ on which $f(\boldsymbol{x})$ is defined and $\mathfrak{R}_z := \mathbb{R}^{n+1}$ be the space containing $f(\boldsymbol{x})$. Both of $s_1$ and $s_2$ are $n$-dimensional. Then $k_2 = s_1 \cap s_2$ and the corresponding knot $l_2$ are $n-1$-dimensional. Arbitrarily select a point $A \in k_2$ and suppose that its projection on $\mathfrak{R}$ is $A'$. Translate $l_2$ into $l'_2 \subset l_2^+$ with $l_2' \parallel l_2$. A line passing $A'$ perpendicular to $l_2$ intersects $l_2'$ at $B'$. Construct a line $BB'$ perpendicular to $\mathfrak{R}$, where the point $B$ is subject to $AB \parallel \mathfrak{R}$. Write $C = \bar{B}B' \cap s_2$. Up to now, we complete the construction of plane $B'A'A$ and a right triangular $\Delta ABC$ in $\mathfrak{R}_z$ as those of Figure \ref{Fig.2}b, both of which are embedded in $\mathfrak{R}_z$.

The dimensionality of the intersection of plane $B'A'A$ with $s_1$ or $s_2$ is one. Let $s^{(A)}_1 = B'A'A \cap s_1$ and $s^{(A)}_2 = B'A'A \cap s_2 = \vec{A}C$. Then in terms of $s^{(A)}_1$ and $s^{(A)}_2$, similar to the case of the two-dimensional input, the one-dimensional method could be applied in $B'A'A$, through which we obtain
\begin{equation}
s_2(\boldsymbol{x}_A) = s_1(\boldsymbol{x}_A) + \lambda_A\sigma(\boldsymbol{w}_2^T\boldsymbol{x}_A + b_2),
\end{equation}
where
\begin{equation}
\boldsymbol{x}_A \in \vec{A'}B' \subset \mathbb{R}^n,
\end{equation}
in which $A'$ is fixed and $B'$ changes along the ray $\vec{A'}B'$.

As the position of $A$ changes, we can similarly construct other right triangles as $\Delta ABC$. Although in a higher-dimensional space, the mechanism is investigated in the planes embedded in $\mathfrak{R}_z$ as $B'A'A$. Due to the construction process and conclusion \Romannum{3} above, the relationship between the different right triangles is the same as that of the three-dimensional space of Figure \ref{Fig.2}b. So
\begin{equation}
\lambda_A = \lambda
\end{equation}
is a constant for different $A$. Then for a fixed $l_2'$, we have
\begin{equation}
s_2(\boldsymbol{x}_{B'}) = s_1(\boldsymbol{x}_{B'}) + \lambda\sigma(\boldsymbol{w}_2^T\boldsymbol{x}_{B'} + b_2),
\end{equation}
where
\begin{equation}
\boldsymbol{x}_{B'} \in l_2' \subset \mathbb{R}^n,
\end{equation}
which corresponds to the arbitrary change of $B'$ in $l_2'$.

The fact that
\begin{equation}
l_2^+ \subseteq \{\boldsymbol{x}_A: A' \in l_2, \boldsymbol{x}_A \in \vec{A'}B'\},
\end{equation}
together with equations 4.40, 4.42 and 4.43, lead to the final conclusion of equation 4.7.

The proof of the uniqueness of $\lambda$ of equation 4.7 with respect to a certain $s_2$ for $n \ge 3$ is similar to that of the two-dimensional input. The intersection of an $n$-dimensional hyperplane with a plane is a line as well. So if we regard $s_2$ of equation 4.36 as an $n$-dimensional hyperplane, equations 4.35 and 4.36 still hold. We now prove that different $n$-dimensional $s_2$ results in different angle $\theta$ of equation 4.35. The key point is that $k_2 = s_2 \cap s_1$ has been fixed whose dimensionality is $n-1$. The subspace $k_2 \subset s_2$ has spanned the $n-1$ dimensions of $s_2$ and $\bar{A}C = s_2 \cap B'A'A$ could contribute to the remaining one. Thus, $k_2$ and $\bar{A}C$ determine $s_2$. Assume that $s_2^{(1)}$ and $s_2^{(2)}$ are two distinct instantiations of $s_2$. Then both of them pass through $k_2$. Write $\bar{A}C^{(i)} =  s_2^{(i)} \cap B'A'A$ for $i = 1, 2$. If $\bar{A}C^{(1)} = \bar{A}C^{(2)}$, then $s_2^{(1)} =  s_2^{(2)}$, which is a contradiction. So in combination with equation 4.35, the uniqueness of $\lambda$ for a certain $s_2$ is proved.

The similarity of the proofs between the $n$- and two-dimensional input spaces lies in equation 4.40, through which both of them can be reduced to the one-dimensional case. The difference is manifested by the dimensionality of $l_2'$ of equation 4.44, which are 1 and $n-1$, respectively. This completes the proof.
\end{proof}

\begin{rmk}
We summarize an intuitive rationale for theorem 2. Given a fixed $s_1$ and the knot $\mathscr{L}$, a continuous $s_2$ with $s_1$ at $\mathscr{L}$ yields $(k_2 = s_1 \cap s_2) \subset s_2$. Because $k_2$ is $n-1$-dimensional and $s_2$ is $n$-dimensional, the information of $n-1$ dimensions of $s_2$ has been provided and only one dimension left for the span of $s_2$. So it is natural for the conclusion of theorem 2 that one parameter $\lambda$ is sufficient to determine the whole $s_2$.
\end{rmk}

\begin{cl}[Continuity property]
A piecewise linear function $f: \mathbb{R}^n \to \mathbb{R}$ realized by a two-layer ReLU network $\mathfrak{N}$ is continuous.
\end{cl}
\begin{proof}
This conclusion is by the repeated application of equation 4.6 when there is more than one unit in the hidden layer of $\mathfrak{N}$.
\end{proof}

\begin{dfn}[Adjacent linear pieces]
Let $g: \mathbb{R}^n \to \mathbb{R}$ be a continuous piecewise linear function with $\zeta$ linear pieces $g_i$'s for $i = 1, 2, \dots, \zeta$. If $g_{\nu}$ and $g_{\mu}$ for $1 \le \nu, \mu \le \zeta$ and $\nu \ne \mu$ share an $n-1$-dimensional hyperplane, we say that they are adjacent. Denote by $b_{i\tau}$'s for $\tau = 1, 2, \dots, \rho_i$ the $n-1$-dimensional hyperplanes shared by $g_i$ and its adjacent linear pieces $g_{n_{i{\tau}}}$'s, where $1 \le n_{i{\tau}} \le \zeta$. Each $b_{i\tau}$ corresponds to a knot on $\mathbb{R}^n$ and we call it a knot of $g_i$.
\end{dfn}

\begin{figure}[!t]
\captionsetup{justification=centering}
\centering
\includegraphics[width=1.8in, trim = {5.0cm 3.0cm 4.2cm 2.0cm}, clip]{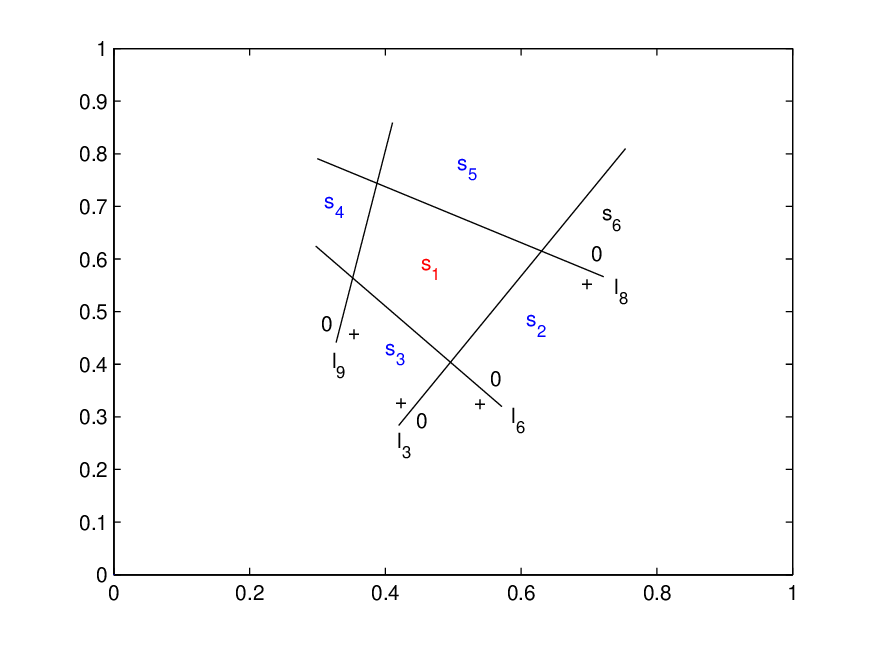}
\caption{Multiple representations of a linear piece.}
\label{Fig.3}
\end{figure}

\begin{cl}[Multiple representations of a linear piece]
Let $\mathfrak{N}$ be a two-layer network whose hidden layer has $m$ units $\mathscr{U}_{i}$'s for $i = 1, 2, \dots, m$. Suppose that each $\mathscr{U}_i$ corresponds to a unique knot $l_i$ in the $n$-dimensional input space $\mathbb{R}^n$. Denote by $\mathcal{S}: \mathbb{R}^n \to \mathbb{R}$ the piecewise linear function realized by $\mathfrak{N}$, which is continuous by corollary 1. Suppose that $\mathcal{S}(\boldsymbol{x})$ is composed of $\zeta$ linear pieces $s_j$'s for $j = 1, 2, \dots, \zeta$ on regions $r_j$'s, respectively. Write
\begin{equation}
\mathcal{S}(\boldsymbol{x}) = \sum_{i = 1}^m\lambda_i\sigma(\boldsymbol{w}_i^T\boldsymbol{x} + b_i)
\end{equation}
by definition 5. Suppose that each linear piece $s_j$ has $\rho_j$ adjacent ones, denoted by $s_{n_{j\nu}}$ for $\nu = 1, 2, \dots, \rho_j$ and $1 \le n_{j\nu} \le \zeta$, each of which shares an $n-1$-dimensional hyperplane with $s_j$ and corresponds to a distinct knot $l_{n_{j\nu}'}$ for $1 \le n_{j\nu}' \le m$.

Then to equation 4.46, $s_j(\boldsymbol{x})$ of $\mathcal{S}(\boldsymbol{x})$ has $\rho_j$ different independent representations, with each being either of the following two possible forms. One is
\begin{equation}
s_j(\boldsymbol{x}) = \lambda_{n_{j\nu}'}\sigma(\boldsymbol{w}_{n_{j\nu}'}^T\boldsymbol{x} + b_{n_{j\nu}'}) + s_{n_{j\nu}}(\boldsymbol{x}),
\end{equation}
where $s_{n_{j\nu}}(\boldsymbol{x})$ is the $\nu$th adjacent piece of $s_{j}$, if
\begin{equation}
r_j \subset l_{n_{j\nu}'}^+ \ \text{and} \ r_{n_{j\nu}} \subset l_{n_{j\nu}'}^0,
\end{equation}
where $r_j$ and $r_{n_{j\nu}}$ are the regions on which $s_j$ and $s_{n_{j\nu}}$ are defined, respectively, and are separated by the knot $l_{n_{j\nu}'}$ of unit $\mathscr{U}_{n_{j\nu}'}$. In equation 4.47, the parameter $\lambda_{n_{j\nu}'}$ uniquely determines $s_j$ based on $s_{n_{j\nu}}$. The other is
\begin{equation}
s_j(\boldsymbol{x}) = -\lambda_{n_{j\nu}'}\sigma(\boldsymbol{w}_{n_{j\nu}'}^T\boldsymbol{x} + b_{n_{j\nu}'}) + s_{n_{j\nu}}(\boldsymbol{x}),
\end{equation}
provided that
\begin{equation}
r_j \subset l_{n_{j\nu}'}^0 \ \text{and} \ r_{n_{j\nu}} \subset l_{n_{j\nu}'}^+.
\end{equation}

On the other hand, according to each unit $\mathscr{U}_k$ for $k = 1, 2, \dots, m$, equation 4.46 can be written as
\begin{equation}
s_{n_{k\mu}}(\boldsymbol{x}) = \lambda_{k}\sigma(\boldsymbol{w}_{k}^T\boldsymbol{x} + b_{k}) + s_{n_{k\mu}'}(\boldsymbol{x})
\end{equation}
for $\mu = 1, 2, \dots, \phi_k$, where $s_{n_{k\mu}}$ and $s_{n_{k\mu}'}$ for $1 \le n_{k\mu}, n_{k\mu}' \le \zeta$ are adjacent linear pieces sharing knot $l_{k}$ of $\mathscr{U}_k$ and $\phi_k$ is the number of the associated adjacent pairs, provided that $r_{n_{k\mu}} \subset l_{k}^+$ and $r_{n_{k\mu}'} \subset l_{k}^0$; or
\begin{equation}
s_{n_{k\mu}}(\boldsymbol{x}) = -\lambda_{k}\sigma(\boldsymbol{w}_{k}^T\boldsymbol{x} + b_{k}) + s_{n_{k\mu}'}(\boldsymbol{x}),
\end{equation}
if $r_{n_{k\mu}} \subset l_{k}^0$ and $r_{n_{k\mu}'} \subset l_{k}^+$.
\end{cl}
\begin{proof}
The proof is by theorem 2. Let us first see an example of Figure \ref{Fig.3} when $n = 2$. In Figure \ref{Fig.3}, each $s_j$ for $j = 1, 2, \dots, 6$ is a linear function on region $r_j$, and the four lines $l_{n_{1\nu}'}$ for $\nu = 1, 2, 3, 4$ are derived from the associated units of the hidden layer of a two-layer ReLU network $\mathcal{N}$, with $n_{11}' = 8$, $n_{12}' = 6$, $n_{13}' = 3$ and $n_{14}' = 9$.

We now examine the relationship between the adjacent linear functions $s_1$ and $s_2$. Among the four lines that shape the region $r_1$, $l_3$ separates $r_1$ from $r_2$ with $r_1 \subset l_3^+$ and $r_2 \subset l_3^0$; $l_8$ and $l_9$ could influence both $r_1$ and $r_2$; $l_6$ can be ignored since $(r_1 \cup r_2) \subset l_6^0$. To a solution for $s_1$ and $s_2$ via network $\mathcal{N}$, the influence of $l_8$ and $l_9$ on $r_1$ can be embedded in the expression of $s_2$, resulting in $s_1$ as a functional of $s_2$. There are two possible cases for the remaining lines not shown in Figure \ref{Fig.3} of the other units. The first is analogous to $l_8$ and $l_9$ whose effects can be combined into $s_2$. The second is similar to $l_6$, having no impact on both $s_1$ and $s_2$. Thus, we have
\begin{equation}
s_1(\boldsymbol{x}) = \lambda_3\sigma(\boldsymbol{w}_3^T\boldsymbol{x} + b_3) + s_2(\boldsymbol{x}),
\end{equation}
which is the case of equation 4.47 when $r_1 \subset l_3^+$ and $r_2 \subset l_3^0$. Equation 4.53 can also be written as
\begin{equation}
s_2(\boldsymbol{x}) = -\lambda_3\sigma(\boldsymbol{w}_3^T\boldsymbol{x} + b_3) + s_1(\boldsymbol{x})
\end{equation}
for the case of equation 4.49 with $r_2 \subset l_3^0$ and $r_1 \subset l_3^+$.

Similarly, by the method of equation 4.53, we have another two expressions of $s_1$, including $s_1 = \lambda_9\sigma(\boldsymbol{w}_9^T\boldsymbol{x} + b_9) + s_4$ and $s_1 = \lambda_8\sigma(\boldsymbol{w}_8^T\boldsymbol{x} + b_8) + s_5$. The principle of equation 4.54 yields $s_1 = -\lambda_6\sigma(\boldsymbol{w}_6^T\boldsymbol{x} + b_6) + s_3$. Thus, the total number of the expressions of $s_1$ is 4, equal to the number of the adjacent linear pieces of $s_1$.

Note that $s_5$ and $s_6$ in Figure \ref{Fig.3} have the same relationship $s_5(\boldsymbol{x}) = \lambda_3\sigma(\boldsymbol{w}_3^T\boldsymbol{x} + b_3) + s_6(\boldsymbol{x})$ as $s_1(\boldsymbol{x}) = \lambda_3\sigma(\boldsymbol{w}_3^T\boldsymbol{x} + b_3) + s_2(\boldsymbol{x})$ of $s_1$ and $s_2$, due to their sharing the same knot $l_3$, which is an instantiation of equation 4.51. And equation 4.52 can be similarly obtained as equation 4.54. This example of Figure \ref{Fig.3} contains all the general principles of corollary 2.
\end{proof}

\begin{rmk}
This corollary is a basic property of piecewise linear functions of two-layer ReLU networks and can help us to find the solution of output weights especially under a complex partition. It will be used in sections 6 and 7.
\end{rmk}

\subsection{Realization of Continuous Linear Splines}
\begin{lem}
Suppose that a region $R \subset \mathbb{R}^n$ is formed by a set $\mathcal{H} = \{l_i: i = 1, 2, \dots, m\}$ of $n-1$-dimensional hyperplanes derived from the ReLUs of the hidden layer of a two-layer network $\mathfrak{N}$, and that $R \subset \bigcap_{i = 1}^{m}l_i^+$ with $m \ge n+1 $. They any linear function on $R$ could be realized by $\mathfrak{N}$, provided that the rank of the size $(n+1) \times m$ matrix
\begin{equation}
\mathcal{W} = \begin{bmatrix}
\boldsymbol{w}_1 & \boldsymbol{w}_2 & \cdots & \boldsymbol{w}_m \\
b_1 & b_2 & \cdots & b_m
\end{bmatrix}
\end{equation}
is $n+1$, where $\boldsymbol{w}_i$ and $b_i$ come from the the equation $\boldsymbol{w}_i^T\boldsymbol{x} + b_i = 0$ of $l_i$.
\end{lem}
\begin{proof}
Let $y = \boldsymbol{w}^T\boldsymbol{x} + b$ be the linear function on $R$ to be realized. By equation 4.1, the output function of $\mathfrak{N}$ for region $R$ is $\sum_{i = 1}^{m}\lambda_i(\boldsymbol{w}_i^T\boldsymbol{x} + b_i)$. The goal is
\begin{equation}
\sum_{i = 1}^{m}\lambda_i(\boldsymbol{w}_i^T\boldsymbol{x} + b_i) = \boldsymbol{w}^T\boldsymbol{x} + b,
\end{equation}
where $\lambda_i$'s are the unknowns to be solved. Equation 4.56 yields a system of linear equations
\begin{equation}
\mathcal{W}\boldsymbol{\lambda} = \boldsymbol{b},
\end{equation}
where $\mathcal{W}$ is the matrix of equation 4.55, $\boldsymbol{\lambda}$ is a $m \times 1$ vector whose entries are $\lambda_i$'s, and $\boldsymbol{b} = [\boldsymbol{w}^ T, b]^T$. Note that in equation 4.56, the number of the unknowns satisfies $m \ge n + 1$. Thus if $\text{rank}(\mathcal{W}) = n + 1$, we can always find a solution of $\boldsymbol{\lambda}$ to fulfil equation 4.56.
\end{proof}

\begin{dfn}[Linear-output matrix \citep*{Huang2022}]
The matrix $\mathcal{W}$ of equation 4.55 for region $R$ is called the linear-output matrix of $R$ with respect to $\mathcal{H}$.
\end{dfn}

\begin{lem}
Let $l_i$ for $i = 1, 2, \dots, n+1$ be an $n-1$-dimensional hyperplane of $\mathbb{R}^n$ for $n \ge 2$, whose equation is $\boldsymbol{w}_i^T\boldsymbol{x} + b_i = 0$. Suppose that there exists a region $\mathscr{R} \subseteq \bigcap_{i = 1}^{n+1}l_i^+$. Then on the basis of $l_i$'s, a nonsingular matrix
\begin{equation}
\mathcal{W}' = \begin{bmatrix}
\boldsymbol{w}'_{1} & \boldsymbol{w}'_{2} & \cdots & \boldsymbol{w}'_{n+1} \\
b'_{1} & b'_2 & \cdots & b'_{n+1}
\end{bmatrix}
\end{equation}
can be constructed, where $\boldsymbol{w}'_{i}$ and $b'_i$ are the parameters of an $n-1$-dimensional hyperplane $l'_{i}$ with equation ${\boldsymbol{w}_i'}^T\boldsymbol{x} + b'_i = 0$, which can approach $l_i$ as precisely as possible (with $l_1' = l_1$) such that $\mathscr{R} \subset l_i'^{+}$ for all $i$ or $\mathscr{R} \subseteq \bigcap_{i = 1}^{n+1}l_i'^+$.
\end{lem}
\begin{proof}
The proof is constructive and uses an inductive method.

\textbf{Step 1}. Write $\boldsymbol{r}_1 = [\boldsymbol{w}_1^T, b_1]^T$, an $(n+1) \times 1$ vector of $n+1$-dimensional space. If $\boldsymbol{r}'_2 = [\boldsymbol{w}_2^T, b_2]^T$ is unparallel to $\boldsymbol{r}_1$, write $\boldsymbol{r}_2 = \boldsymbol{r}'_2$; otherwise, let
\begin{equation}
\boldsymbol{r}_2 = \boldsymbol{r}'_2 + \epsilon\mathscr{N}_1,
\end{equation}
where $\mathscr{N}_1 \perp \boldsymbol{r}_1$ and $\epsilon$ is a positive number that could be arbitrarily small. Then $\boldsymbol{r}_2 \nparallel \boldsymbol{r}_1$.

\textbf{Step 2}. If $\boldsymbol{r}'_3 = [\boldsymbol{w}_3^T, b_3]^T \notin \boldsymbol{r}_1 \oplus \boldsymbol{r}_2$, where ``$\oplus$'' represents linear-combination operation, let $\boldsymbol{r}_3 = \boldsymbol{r}'_3$. Otherwise, add a perturbation as
\begin{equation}
\boldsymbol{r}_3 = \boldsymbol{r}'_3 + \epsilon\mathscr{N}_2,
\end{equation}
where $\mathscr{N}_2 \perp \boldsymbol{r}_1 \oplus \boldsymbol{r}_2$ and $\epsilon$ is defined as in equation 4.59, resulting in  $\boldsymbol{r}_3 \notin \boldsymbol{r}_1 \oplus \boldsymbol{r}_2$.

\textbf{Inductive Step}. Suppose that we have constructed $\boldsymbol{r}_{k-1} \notin \boldsymbol{r}_1 \oplus \boldsymbol{r}_2 \oplus \dots \oplus \boldsymbol{r}_{k-2}$ for $3 \le k \le n+1$. If $\boldsymbol{r}'_{k} = [\boldsymbol{w}_{k}^T, b_{k}]^T \notin \boldsymbol{r}_1 \oplus \boldsymbol{r}_2 \oplus \dots \oplus \boldsymbol{r}_{k-1}$, let $\boldsymbol{r}_{k} = \boldsymbol{r}'_{k}$; otherwise, perturb it to
\begin{equation}
\boldsymbol{r}_{k} = \boldsymbol{r}'_{k} + \epsilon\mathscr{N}_{k-1},
\end{equation}
where $\mathscr{N}_{k-1} \perp \boldsymbol{r}_1 \oplus \boldsymbol{r}_0 \oplus \dots \oplus \boldsymbol{r}_{k-1}$ , such that $\boldsymbol{r}_{k} \notin \boldsymbol{r}_1 \oplus \boldsymbol{r}_0 \oplus \dots \oplus \boldsymbol{r}_{k-1}$.

Repeat the inductive step until $k = n + 1$. Let
\begin{equation}
\begin{bmatrix}
{\boldsymbol{w}_i'}^T, b_i'
\end{bmatrix}^T = \boldsymbol{r}_{i}
\end{equation}
for $i = 1, 2, \dots, n+1$. Then the hyperplane $l_i'$ with equation ${\boldsymbol{w}_i'}^T\boldsymbol{x} + b_i' = 0$ is constructed. By the construction process, among $\boldsymbol{r}_i$'s, each of them is not the linear combination of the remaining ones, and thus matrix $\mathcal{W}'$ of equation 4.58 is nonsingular.

When $\epsilon$ in equation 4.61 is sufficiently small, $l_i'$ could approximate $l_i$ to arbitrary desired accuracy, such that $\mathscr{R} \subset l_i^+$ could lead to $\mathscr{R} \subset l_i'^+$ for all $i$. This completes the proof.
\end{proof}

\begin{dfn}[Continuous linear spline]
Denote by $\Delta = \{l_{\nu}: \nu = 1, 2, \dots, \zeta\}$ a set of $n-1$-dimensional hyperplanes of $\mathbb{R}^n$ satisfying
\begin{equation}
l_1 \prec l_2 \prec \cdots \prec l_{\zeta}
\end{equation}
of equation 3.1, with $R_{\nu}$'s as the associated ordered regions of $U = [0, 1]^n$ formed by $\Delta$. Let $\mathfrak{U} = \bigcup_{\nu = 1}^{\zeta}R_{\nu}$ and then $\mathfrak{U} \subseteq U$. Write
\begin{equation}
\begin{aligned}
\mathfrak{S}_n(\Delta; \mathfrak{U})  =  \{s: s(\boldsymbol{x})=s_{\nu}(\boldsymbol{x}) \in \mathcal{P}_2 & \ \ \text{for} \ \boldsymbol{x} \in R_{\nu}, s_{\nu}(l_{\nu+1}) = s_{\nu+1}(l_{\nu+1}) \ \text{for} \ \nu \ne \zeta, \\
& \nu = 1, 2, \dots, \zeta\},
\end{aligned}
\end{equation}
where $\mathcal{P}_2$ is the set of the linear functions on $\mathbb{R}^n$. Then each $s(\boldsymbol{x}) \in \mathfrak{S}_n(\Delta; \mathfrak{U})$ is a continuous piecewise linear function and we call it a continuous linear spline on $\mathfrak{U}$.
\end{dfn}

\begin{dfn}[One-sided bases]
With some notations from definition 9, let
\begin{equation}
\mathfrak{B}_1 = \{\rho_{i}(\boldsymbol{x}) = \sigma(\boldsymbol{w}_{i}^T\boldsymbol{x} + b_{i}): -\xi \le i \le \zeta, \ n-1 \le \xi\},
\end{equation}
where $\rho_{i}(\boldsymbol{x})$ determines a knot $l_{i}$ of $\boldsymbol{w}_{i}^T\boldsymbol{x} + b_{i} = 0$ and $l_{\nu}$ for $\nu = 1, 2, \dots, \zeta$ is from equation 4.63. Suppose that
\begin{equation}
R_{\nu} \subset l^+_{\mu}
\end{equation}
for all $\nu$'s and all $\mu = -\xi, -\xi+1, \dots, 0$ and that the rank of the linear-output matrix for $R_1$ is $n+1$. Then we call $\mathfrak{B}_1$ a set of the one-sided bases of $\mathfrak{S}_n(\Delta; \mathfrak{U})$ of equation 4.64.

Each of $\rho_{\mu}(\boldsymbol{x})$'s whose associated hyperplane satisfies equation 4.66, together with $\rho_1(\boldsymbol{x})$ fulfilling $R_{\nu} \subset l_1^+$ for all $\nu$, are called a \textsl{global basis} (or \textsl{unit}) of $\mathfrak{S}_n(\Delta; \mathfrak{U})$ (or $\mathfrak{N}$), and the remaining ones are called \textsl{local basis} (or \textsl{unit}), where $\mathfrak{N}$ is the two-layer ReLU network whose units of the hidden layer are from $\mathfrak{B}_1$.
\end{dfn}

\begin{thm}[Approximation over a single strict partial order]
Any continuous linear spline $\mathcal{S}(\boldsymbol{x}) \in \mathfrak{S}_n(\Delta; \mathfrak{U})$ of equation 4.64 with $\zeta$ linear pieces could be realized by a two-layer ReLU network $\mathfrak{N}$ in terms of the one-sided bases of equation 4.65, whose hidden layer has
\begin{equation}
\Theta \ge \zeta + n
\end{equation}
units.
\end{thm}
\begin{proof}
We first prove the $\Theta = \zeta + n$ case. The proof is composed of two parts due to different construction methods. The first is the implementation of a linear function on $R_1$ and the second is for the remaining regions.

Because $\mathfrak{U} \subseteq U$ is a bounded set, we can always find $n$ hyperplanes $l'_{\kappa}$ for $\kappa = 1, 2, \dots, n$ satisfying $R_{\nu} \subset l_{\kappa}'^+$ for all $\nu$ and $\kappa$. Then through lemma 4, construct $\mathcal{L}_{\tau}$'s for $\tau = 1, 2, \dots, n+1$ on the basis of $\mathcal{L}'_{\tau}$, where $\mathcal{L}'_1 = l_1$ and $\mathcal{L}'_{\kappa + 1} = l'_{\kappa}$, such that the linear-output matrix formed by $\mathcal{L}_{\tau}$'s is nonsingular. Noting that $\mathcal{L}_1$ is not necessarily perturbed, so the integrity of $R_1$ is ensured during the operation of lemma 4. When $\epsilon$ of equation 4.61 is sufficiently small, $R_{\nu} \subset \mathcal{L}_{\tau}^+$ for all $\tau$. By lemma 3, any linear function including $s_1 = \boldsymbol{w}_1^T\boldsymbol{x} + b_1$ on $R_1$ could be realized by the $n+1$ units associated with $\mathcal{L}_{\tau}$'s.

The linear function $s_{\nu} = \boldsymbol{w}_{\nu}^T\boldsymbol{x} + b_{\nu}$ for $\nu = 2, 3, \dots, \zeta$ on $R_{\nu}$ is implemented by the recurrence formula
\begin{equation}
s_{\nu}(\boldsymbol{x}) = s_{\nu-1}(\boldsymbol{x}) + \lambda_{\nu}\sigma(\boldsymbol{w}_{\nu}^T\boldsymbol{x} + b_{\nu}),
\end{equation}
according to equation 4.7 and the strict partial order of $l_{\nu}$'s, where $\boldsymbol{w}_{\nu}^T\boldsymbol{x} + b_{\nu}$ is from the equation $\boldsymbol{w}_{\nu}^T\boldsymbol{x} + b_{\nu} = 0$ of $l_{\nu}$, and $s_1(\boldsymbol{x})$ has been previously constructed. By theorem 2, to any $s_{\nu}(\boldsymbol{x})$ continuous with $s_{\nu-1}(\boldsymbol{x})$ at knot $l_{\nu}$, if the equation $\boldsymbol{w}_{\nu}^T\boldsymbol{x} + b_{\nu} = 0$ of $l_{\nu}$ is fixed, there exists a unique parameter $\lambda_{\nu}$ producing it, which can be easily obtained by arbitrary one point of $s_{\nu}(\boldsymbol{x})$ that is not on $s_{\nu} \cap s_{\nu-1}$. As mentioned at the beginning of section 4.2, the function values on $R_{\nu - 1} \cap R_{\nu}$ are produced by $s_{\nu-1}$, analogous to the interval $(x_{j-1}, x_j]$ of section 2.1, since $R_{\nu - 1} \cap R_{\nu} \subset l_{\nu-1}^+$.

By the construction process of $s_1$ and $s_{\nu}$'s, the total number of the units required is $\Theta = (n+1) + (\zeta-1) = \zeta + n$. The case of $\Theta > \zeta + n$ is due to lemma 3. By lemma 3, when the number of the units for $s_1$ is greater than $n+1$, it can still be realized, provided that the rank of the linear-output matrix for $R_1$ is $n+1$. On the basis of the nonsingular $(n+1) \times (n+1)$ matrix $\mathcal{W}'$ of equation 4.58, we can add arbitrary number $m$ of hyperplanes in terms of $\mathcal{W} = [\mathcal{W}', \mathcal{W}'']$ whose rank is $n+1$, and each added hyperplane $\mathscr{L}_{\mu}$ for $\mu = 1, 2, \dots, m$ satisfies $R_{\nu} \subset \mathscr{L}_{\mu}^+$ for all $\mu$ and all $\nu = 1, 2, \dots, \zeta$. This proves the case of $\Theta > \zeta + n$.
\end{proof}

\section{Two-Sided Bases}
In the learning process, the parameters of a two-layer ReLU network are randomly initialized and iterated by the back-propagation algorithm. The training solution may not be of the one-sided bases of the preceding sections. For instance, to the univariate case, when positive and negative input weights simultaneously exist, the associated units would not be of the one-sided type. To interpret the training solution, we enlarge the solution space by the extension to two-sided bases, through which the solutions for one-dimensional input, including the training ones, could be completely understood (corollary 3). Experimental validation of section 8 will further demonstrate the necessity of this extension.

\subsection{Preliminaries}
\begin{dfn}[Two-sided bases]
This concept is based on the one-sided bases of definition 10. In all the cases that follow, we assume that the rank of the linear-output matrix for $R_1$ of equation 4.66 is $n+1$. On the basis of the one-sided bases of equation 4.65, add some new elements in $\mathfrak{B}_1$ as
\begin{equation}
\begin{aligned}
\mathfrak{B}^{(1)}_2 = \{\rho_{i}(\boldsymbol{x}), \rho'_{\tau}&(\boldsymbol{x}) = \sigma(-(\boldsymbol{w}_{i_{\tau}}^T\boldsymbol{x} + b_{i_{\tau}})): -\xi \le i \le \zeta, \\& 2 \le i_{\tau } \le \zeta, \tau = 1, 2, \dots, \beta\},
\end{aligned}
\end{equation}
where $1 \le \beta \le \zeta-1$ and $\xi \ge -1$ (similarly for $\beta$ and $\xi$ of equations 5.2 and 5.3 below). Support that $\xi + \beta + 1 \ge n$. We call the elements of $\mathfrak{B}_2^{(1)}$ the added two-sided bases of $\mathfrak{S}_n(\Delta; \mathfrak{U})$.

Write
\begin{equation}
\begin{aligned}
\mathfrak{B}^{(2)}_2 = \{\rho_{i}(\boldsymbol{x})&, \rho'_{\tau}(\boldsymbol{x}) = \sigma(-(\boldsymbol{w}_{i_{\tau}}^T\boldsymbol{x} + b_{i_{\tau}})): -\xi \le i \le \zeta, \\ &  i \ne i_{\tau}, 2 \le i_{\tau } \le \zeta, \tau = 1, 2, \dots, \beta\},
\end{aligned}
\end{equation}
with $\xi + 1 \ge n$ satisfied, whose elements are called the substituted two-sided bases of $\mathfrak{S}_n(\Delta; \mathfrak{U})$.

Let
\begin{equation}
\begin{aligned}
\mathfrak{B}^{(3)}_2 = \{\rho_{i}(\boldsymbol{x}), \rho'_{\tau}(\boldsymbol{x})& = \sigma(-(\boldsymbol{w}_{i_{\tau}}^T\boldsymbol{x} + b_{i_{\tau}})): -\xi < i \le \zeta, I \ne \emptyset \subset \{i_{\tau}\}, \\ &i \notin I, 2 \le i_{\tau} \le \zeta, \tau = 1, 2, \dots, \beta \}
\end{aligned}
\end{equation}
and the condition $\xi + \beta - |I| + 1 \ge n$ is satisfied. We call the elements of $\mathfrak{B}^{(3)}_2$ the compound two-sided bases of $\mathfrak{S}_n(\Delta; \mathfrak{U})$.

Any one of the above $\mathfrak{B}^{(1)}_2$, $\mathfrak{B}^{(2)}_2$ and $\mathfrak{B}^{(3)}_2$ is called a set of the two-sided bases of $\mathfrak{S}_n(\Delta; \mathfrak{U})$ and can be denoted by the simplified notion $\mathfrak{B}_2$. Each $\rho'_{\tau}(\boldsymbol{x})$ and $\rho_{i}(\boldsymbol{x})$ is called a \textsl{negative} and \textsl{positive} base, respectively, with the associated units respectively called \textsl{negative} and \textsl{positive} unit. A knot that corresponds to both negative and positive units is called a bidirectional knot.

The concepts of global and local bases (units) can be similarly defined as those of the one-sided case of definition 10.
\end{dfn}

\begin{rmk}
The reason of introducing the three types of two-sided bases will be explained in the proof of theorem 4 of sections 5.2.
\end{rmk}

\noindent
\textbf{Example}. Figure \ref{Fig.4} provides an intuitive example of the tree types of two-sided bases whose detailed description will be given in the proof of theorem 4.

\begin{lem}
Notations being as in lemma 3, suppose that $m > n+1$ and that among the $m$ hyperplanes whose positive-output regions include $R$, there are $\tau$ ones for $m - \tau \ge n+1$ whose output weights are fixed to be a constant, forming a subset $C \subseteq \mathcal{H}$, where $\mathcal{H}$ is the set of the $m$ hyperplanes. Let
\begin{equation}
\mathcal{H}' = \mathcal{H}-C
\end{equation}
and $m' = |\mathcal{H}'|$. Then $m' \ge n+1$. If the rank of the linear-output matrix $\mathcal{W}'$ of the region $R$ with respect to $\mathcal{H}'$ is $n+1$, any linear function on $R$ could be realized by the corresponding two-layer network $\mathfrak{N}$.
\end{lem}
\begin{proof}
In this case, equation 4.56 becomes
\begin{equation}
\sum_{\nu = 1}^{m'}\lambda_{i_{\nu}}(\boldsymbol{w}_{i_{\nu}}^T\boldsymbol{x} + b_{i_{\nu}}) = {\boldsymbol{w}'}^T\boldsymbol{x} + b',
\end{equation}
where hyperplane $l_{i_{\nu}}$ for $1 \le i_{\nu} \le m$ with equation $\boldsymbol{w}_{i_{\nu}}^T\boldsymbol{x} + b_{i_{\nu}} = 0$ comes from $\mathcal{H}'$, and the right-hand side ${\boldsymbol{w}'}^T\boldsymbol{x} + b'$ is obtained by the original ${\boldsymbol{w}}^T\boldsymbol{x} + b$ subtracting the weighted outputs of the units associated with the hyperplanes of $C$. The conclusion is obvious analogous to lemma 3.
\end{proof}

\begin{dfn}[Redundant hyperplane (unit)]
In lemma 3 or 5, at least $n+1$ units are sufficient to produce any linear function on $R$, for which we call each of the hyperplanes (units) except for arbitrary $n+1$ necessary ones a redundant hyperplane (unit).
\end{dfn}

\subsection{One-Dimensional Input}

\begin{figure}[!t]
\captionsetup{justification=centering}
\centering
\subfloat[Substituted case.]{\includegraphics[width=2.55in, trim = {4.2cm 3.8cm 1.6cm 1.4cm}, clip]{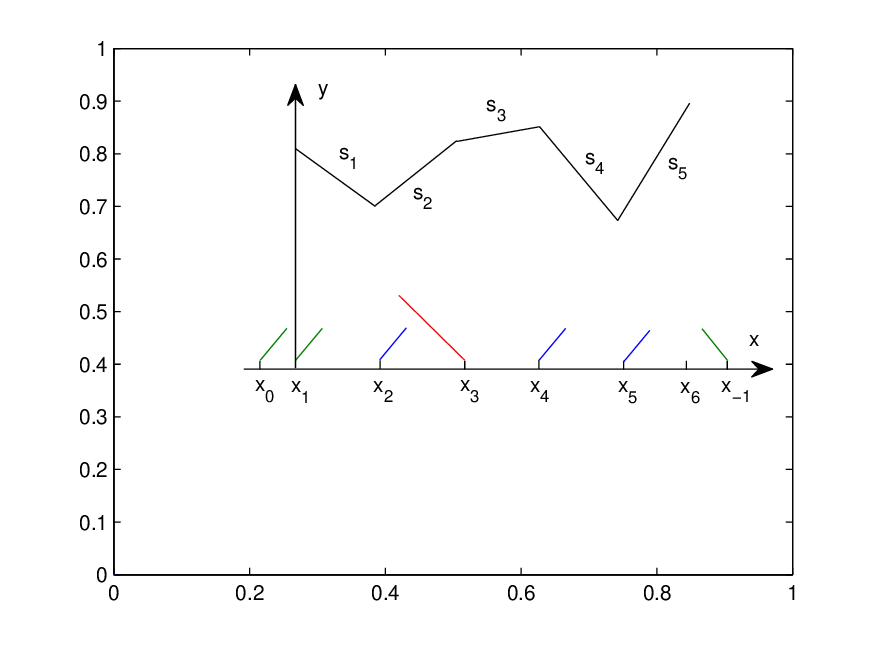}} \\
\subfloat[Added case.]{\includegraphics[width=2.7in, trim = {4.2cm 3.8cm 1.6cm 1.4cm}, clip]{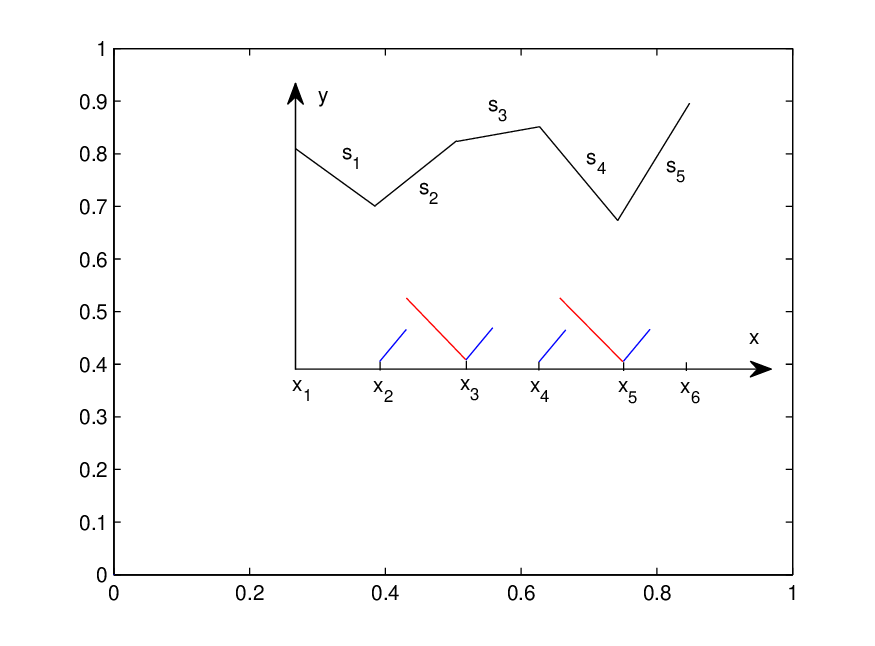}} \quad \quad
\subfloat[Compound case.]{\includegraphics[width=2.7in, trim = {4.2cm 3.8cm 1.6cm 1.4cm}, clip]{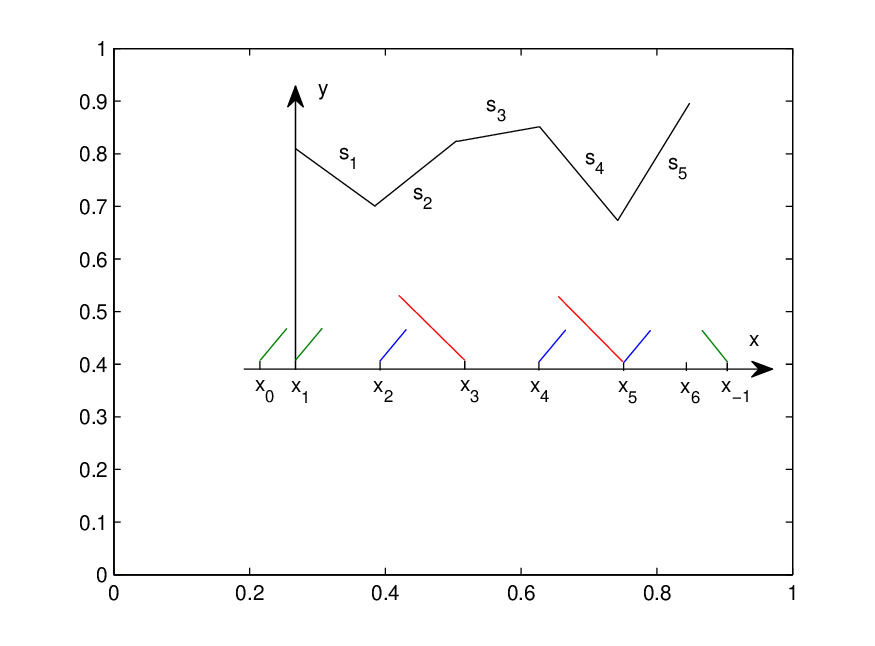}}
\caption{Two-sided bases.}
\label{Fig.4}
\end{figure}

\begin{thm}
Let $\mathcal{S}(x) \in \mathfrak{S}_1(\Delta)$ of equation 2.6 be an arbitrary continuous linear spline with $\zeta$ linear pieces. Then a two-layer ReLU network $\mathfrak{N}$ can realize $\mathcal{S}(x)$ in terms of any one of the three types of the two-sided bases of definition 11, with the hidden layer having
\begin{equation}
\Theta \ge \zeta + 1
\end{equation}
units.
\end{thm}
\begin{proof}
The proof is constructive with the aid of Figure \ref{Fig.4}, in which the examples of three kinds of two-sided bases are given. In Figure \ref{Fig.4}, the red long and blue short lines represent the negative and positive units, respectively; each green short one is of a global unit. The interval $[x_1, x_6]$ is the domain on which the linear spline is defined.

\textbf{\Romannum{1}. Substituted two-sided bases}. We first see the simplest substituted two-sided bases of Figure \ref{Fig.4}a. Note that the correlation between adjacent linear pieces of equation 4.7 can also be expressed as
\begin{equation}
\mathscr{S}_1(\boldsymbol{x}) = \mathscr{S}_2(\boldsymbol{x}) - \lambda\sigma(\boldsymbol{w}^T\boldsymbol{x} + b),
\end{equation}
where $\mathscr{S}_1$ and $\mathscr{S}_2$ correspond to $s_1$ and $s_2$ of equation 4.7, respectively; we change the two notions to avoid the confusion in this proof. To the example of Figure \ref{Fig.4}a, we thus have
\begin{equation}
s_3 = s_2 - \lambda_3\sigma(-(x - x_3)),
\end{equation}
where $\lambda_3$ can be determined by $s_2$ and $s_3$ similarly to lemma 1. Denote by $u_3$ and $l_3$ the unit and knot of $\sigma(-(x - x_3))$, respectively. Then $u_3$ could influence $l_3^+ = (-\infty, x_3)$.

Negative unit $u_3$, together with the global units of $u_{-1}$ at $x_{-1}$, $u_0$ at $x_0$ and $u_1$ at $x_1$, account for the production of $s_1$ on $[x_1, x_2]$. The output weight $\lambda_3$ of $u_3$ of equation 5.8 has been fixed for $s_3$ and hence we can only use $u_{-1}$, $u_0$ and $u_1$ to realize $s_1$.

In the one-dimensional case, it's easy to construct a linear-output matrix $\mathcal{W}$ for $[x_1, x_2]$ of $s_1$ whose rank is 2, without the need of the perturbation of lemma 4. For example, the parameters of $u_0$ and $u_1$ contribute to a submatrix
\begin{equation}
\mathcal{W}' = \begin{bmatrix}
1 &  1\\
-x_0 & -x_1
\end{bmatrix}
\end{equation}
of $\mathcal{W}$, which is nonsingular due to $x_0 \ne x_1$; the principle of this example is general for the case of the one-dimensional input. If the number of the units for $s_1$ is greater than 2, lemma 3 or 5 is required.

In the example of Figure \ref{Fig.4}a, $u_3$ can be regarded as a redundant unit owing to two reasons. First, if we assume that the linear function $s_2$ has been realized beforehand, the parameter $\lambda_3$ of $u_3$ is determined and thus the influence of $u_3$ on $s_1$ is of a type of redundant unit. Second, it is certainly true that $s_2$ can be produced before $s_3$, since for a fixed $\lambda_3$ the two global units can generate $s_1$ by lemma 5, leading to the implementation of $s_2$ due to their continuities.

The construction of the remaining linear functions except for $s_3$ is the same as lemma 1. Note that the influence of $u_3$ on $[x_1, x_3)$ can be embedded in the expression of $s_1$. For $x \ge x_3$, the accomplished $s_3$ becomes a new initial linear function for the succeeding ones and $u_3$ will have no impact on them. Compared with the strict partial order of one-sided bases, this is a disturbance but can be resolved. The case of more than one negative unit can be similarly analyzed.

\textbf{\Romannum{2}. Added two-sided bases}. In case \Romannum{1}, the output weight $\lambda_3$ of unit $u_3$ of Figure \ref{Fig.4}a cannot be freely adjusted for $s_1$ due to the constraint of equation 5.8, and the production of $s_1$ should rely on global units . This situation can be changed by simultaneously introducing positive and negative units at the same knot, which is called a bidirectional knot by definition 11. As shown in Figure \ref{Fig.4}b, there exist two units at $x_3$. One is positive (blue line) and the other is negative (red line). Then the relationship between $s_2$ and $s_3$ for $x \in (x_3, x_4]$ is
\begin{equation}
\begin{aligned}
s_3 = s_2 - &(-\lambda_3^{(1)}(x - x_3)) + \lambda_3^{(2)}\sigma(x - x_3)\\
&= s_2' + \lambda_3^{(2)}\sigma(x - x_3),
\end{aligned}
\end{equation}
where
\begin{equation}
s_2' = s_2 + \lambda_3^{(1)}(x - x_3)
\end{equation}
and $\lambda_3^{(1)}$ is from $s_2 = s_3 + \lambda_3^{(1)}\sigma(-(x - x_3))$ for $x \in (x_2, x_3]$ when only a negative unit exists. Notice that $s_2'$ is also continuous with $s_3$ at knot $x_3$, since
\begin{equation}
\lim_{x \to x_3^0}s_2'(x) = s_2(x_3) = \lim_{x \to x_3^+}s_3(x)
\end{equation}
by equations 5.10 and 5.11, where $x_3^0$ and $x_3^+$ are the left and right sides of $x_3$, respectively. Thus, the solution of $\lambda_3^{(2)}$ for $s_3$ exists.

By equation 5.10, the parameter $\lambda_3^{(2)}$ can be used to shape $s_3$, while $\lambda_3^{(1)}$ could be freely changed to produce $s_1$. Through this method, we can use local units to implement the initial linear function $s_1$, without introducing extra knots that the linear spline doesn't have or resorting to global units.

The parameters $\lambda_3^{(1)}$ and $\lambda_3^{(2)}$ are set as follows. If there's no negative unit at $x_3$, we have
\begin{equation}
s_3 = s_2 + \lambda_3\sigma(x - x_3) = s_2 + \lambda_3(x - x_3)
\end{equation}
for $x \in (x_3, x_4]$, where $\lambda_3$ can be obtained by lemma 1. Suppose that $\lambda_3^{(1)}$ has been set for $s_1$. By equation 5.10, we have
\begin{equation}
s_3 = s_2 + (\lambda_3^{(1)}+\lambda_3^{(2)})(x-x_3)
\end{equation}
for $x \in (x_3, x_4]$. Equations 5.13 and 5.14 imply
\begin{equation}
\lambda_3^{(2)} = \lambda_3 - \lambda_3^{(1)}.
\end{equation}

In Figure \ref{Fig.4}b, $x_5$ is similar to $x_3$ and also has negative and positive units. So we can use $\lambda_3^{(1)}$ and $\lambda_5^{(1)}$ to produce arbitrary $s_1$, without requiring any global unit. Two bidirectional knots are sufficient to yield a nonsingular linear-output matrix for $s_1$ by their negative units, analogous to equation 5.9.

Other linear functions can be produced by either equation 4.7 or 5.7. The distinction between cases \Romannum{2} and \Romannum{1} is that the former doesn't necessarily require global units to realize the initial $s_1$.

\textbf{\Romannum{3}. Compound two-sided bases}. The combination of cases \Romannum{1} and \Romannum{2} is the compound two-sided bases, with an example shown in Figure \ref{Fig.4}c. In this case, a knot that only has a negative unit and a bidirectional knot can simultaneously exist. Compared with cases \Romannum{1} and \Romannum{2}, this type doesn't introduce new mechanisms but only combines different kinds of units or knots.

The examples of Figure \ref{Fig.4} include all the general principles of this theorem. To the number of the units of the hidden layer required, in any of the three types of two-sided bases, the minimum is $\zeta + 1$, because $s_1$ needs at least two units and each of the other linear functions needs one, the same as the case of one-sided bases. So inequality 5.6 holds. This completes the proof.
\end{proof}

\subsection{General $n$-Dimensional Input}
As in section 4, in order for integrity, each conclusion of general input $\mathbb{R}^n$ in this section will include that of the one-dimensional input as a special case.
\begin{thm}[Principle of two-sided bases]
Notations being from theorem 3, any continuous linear spline $\mathcal{S}(\boldsymbol{x}) \in \mathfrak{S}_n(\Delta; \mathfrak{U})$ with $\zeta$ linear pieces can be realized by a two-layer ReLU network $\mathfrak{N}$, through any one of the three types of the two-sided bases of definition 11. The number of the units of the hidden layer of $\mathfrak{N}$ required satisfies
\begin{equation}
\Theta \ge \zeta + n.
\end{equation}
\end{thm}
\begin{proof}
Because of the equivalence of the strict partial orders in $\mathbb{R}^n$ and $\mathbb{R}$ (proposition 2), on the basis of theorem 3 for the one-sided bases of $\mathbb{R}^n$, the proof is similar to that of theorem 4 for $\mathbb{R}$.
\end{proof}

\begin{figure}[!t]
\captionsetup{justification=centering}
\centering
\includegraphics[width=2.7in, trim = {4.8cm 3.8cm 1.6cm 1.4cm}, clip]{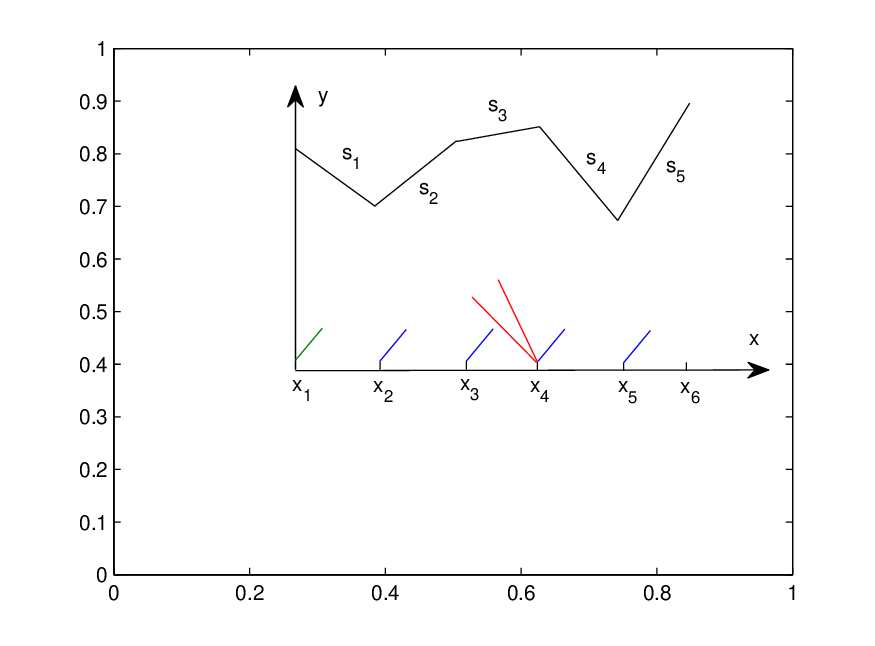}
\caption{Equivalent units.}
\label{Fig.5}
\end{figure}

\begin{prp}
Under definition 11, if there's more than one negative (positive) unit associated with the same knot, then they are equivalent to one in the contribution of forming a piecewise linear function.
\end{prp}
\begin{proof}
In fact, for two negative (positive) units $\rho_1(\boldsymbol{x}) = \sigma(\boldsymbol{w}_1^T\boldsymbol{x} + b_1)$ and $\rho_2(\boldsymbol{x}) = \sigma(\boldsymbol{w}_2^T\boldsymbol{x} + b_2)$ derived from the same knot, we have $\rho_1(\boldsymbol{x}) = \alpha\rho_2(\boldsymbol{x})$, where $\alpha$ is a positive real number. So $\lambda_1\rho_1(\boldsymbol{x}) + \lambda_2\rho_2(\boldsymbol{x}) = (\alpha\lambda_1 + \lambda_2)\rho_2(\boldsymbol{x})$, which is equivalent to one unit. Figure \ref{Fig.5} gives an example of equivalent units (red ones at knot $x_4$).
\end{proof}

\begin{thm}
Notations being as in theorem 3, except for the case that some linear function $s_1$ on $R_1$ can be produced by less than $n+1$ global units, any solution of a two-layer ReLU network $\mathfrak{N}$ for a continuous linear spline $\mathcal{S}(\boldsymbol{x}) \in \mathfrak{S}_n(\Delta; \mathfrak{U})$ must be in the form of either one- or two-sided bases of theorems 3 and 5, respectively.
\end{thm}
\begin{proof}
To a $\mathcal{S}(\boldsymbol{x})$ produced by $\mathfrak{N}$, each unit $u_i$ of the hidden layer of $\mathfrak{N}$ must be either a global unit realizing the first linear function $s_1$ or a local unit yielding one of the knots $l_i$'s. A local unit generates a knot in terms of either negative or positive form. Once we establish a strict partial order of $l_i$'s, due to axioms 1 and 2, a solution on the ordered regions must obey the principle of theorem 3 or 5, corresponding to the one- or two-sided bases, respectively. Note that even if the number of the global units is less than $n+1$, they can still form some linear functions to be $s_1$ and this is the special case mentioned in this theorem. We didn't consider the case of both equivalent units (proposition 5) and inactivated units because they are trivial and can be easily incorporated into our theoretical framework.
\end{proof}

\begin{cl}[Completeness of constructed solutions]
If the input is one-dimensional, to a continuous linear spline $\mathcal{S}(x) \in \mathfrak{S}_1(\Delta)$ of equation 2.6, except for the case that some $s_1$ on $I_1$ can be realized by less than two units, any solution of a two-layer ReLU network $\mathfrak{N}$ for $\mathcal{S}(x)$, including the one obtained by the back-propagation algorithm, is of the type of either one- or two-sided bases.
\end{cl}
\begin{proof}
In this case, to each solution of $\mathfrak{N}$ for $\mathcal{S}(x)$, the knots are real numbers and the ``less than'' relation of the knots is the unique strict partial order. By theorem 6, the conclusion follows.
\end{proof}

\section{Approximation over Multiple Strict Partial Orders}
We constructed an arbitrary piecewise linear function over a single strict partial order in section 4. However, universal set $U = [0, 1]^n$ cannot be easily covered by a single strict partial order with each of the ordered regions being arbitrarily small, and this problem directly affects the approximation error. Correspondingly, the training solution obtained by the back-propagation algorithm is usually much more complicated than the constructed one of section 4, as can be seen by some examples of section 8. That's why we make this extension to multiple strict partial orders. Notice that theorem 7 doesn't prove the universal approximation capability because there's a hypothesis to be solved in section 7.

\subsection{Preliminaries}
\begin{dfn}[Adjacent Regions]
Let $H$ be a set of $n-1$-dimensional hyperplanes of $\mathbb{R}^n$ and $R_i$'s for $i = 1, 2, \dots, \zeta$ be the regions formed by $H$. Two regions $R_{\nu}$ and $R_{\mu}$ for $1 \le \nu, \mu \le \zeta$ and $\nu \ne \mu$ are said to be adjacent, if the intersection $R_{\nu} \cap R_{\mu}$ is $n-1$-dimensional (i.e., $\dim(R_{\nu} \cap R_{\mu}) = n-1$) and is part of a hyperplane of $H$.
\end{dfn}

\noindent
\textbf{Example}. In Figure \ref{Fig.7}a, $R^{(1)}_1$ and $R^{(1)}_2$ are adjacent, while $R^{(1)}_1$ and $R^{(1)}_3$ as well as $R^{(1)}_1$ and $R^{(2)}_1$ are not.

\begin{dfn}[Space of piecewise linear functions over a partition]
Some notations being from definition 13, suppose that $\mathcal{R} = \bigcup_{i = 1}^{\zeta}R_i$ is the set of the regions of $U = [0, 1]^n$ formed by $H$ and that $U = \mathcal{R}$. If two linear functions $f(\boldsymbol{x})$ and $g(\boldsymbol{x})$ are continuous at some knot, we write $f \frown g$. Let
\begin{equation}
\begin{aligned}
\mathfrak{K}_n(H)  &:=  \{s: s(\boldsymbol{x})=s_i(\boldsymbol{x}) \in \mathcal{P}_2 \ \text{for} \ x \in R_i, \\& s_i \frown \mathscr{N}_i, \ i = 1, 2, \dots, \zeta\},
\end{aligned}
\end{equation}
where $\mathcal{P}_2$ is the set of linear functions on $\mathbb{R}^n$, and $\mathscr{N}_i = \{s_{n_{i\kappa}}(\boldsymbol{x}): \kappa = 1, 2, \dots, \phi_i, 1 \le n_{i\kappa} \le \zeta\}$ whose each element $s_{n_{i\kappa}}$ is a linear function on region $R_{n_{i\kappa}}$ adjacent to $R_i$ and is continuous with $s_i$ at the knot $k_i = R_i \cap R_{n_{i\kappa}}$. So $\mathfrak{K}_n(H)$ is the set of the continuous piecewise linear functions on $\mathcal{R}$ whose knots are from the hyperplanes of $H$; and we call $\mathfrak{K}_n(H)$ the space of continuous piecewise linear functions with respect to $H$.
\end{dfn}

\begin{prp}
Any function $g(\boldsymbol{x}): [0, 1]^n \to \mathbb{R}$ realized by a two-layer ReLU network $\mathfrak{N}$ satisfies $g(\boldsymbol{x}) \in \mathfrak{K}_n(H)$, where $H$ is the set of the hyperplanes of the units of the hidden layer of $\mathfrak{N}$.
\end{prp}
\begin{proof}
This conclusion is by corollary 1 and definition 14.
\end{proof}

If we say that a function $f(\boldsymbol{x})$ for $\boldsymbol{x} \in \mathbb{R}^n$ is a  ``$C^1$ function'', it means that its partial derivative with respect to each entry of $\boldsymbol{x}$ is continuous; and in this case, $f(\delta(\boldsymbol{x}_0))$, where $\delta(\boldsymbol{x}_0)$ is a small enough neighborhood of any fixed $\boldsymbol{x}_0$, looks like a hyperplane or could approximate a hyperplane as precisely as possible, a property useful for local linear approximations.

\begin{figure}[!t]
\captionsetup{justification=centering}
\centering
\includegraphics[width=3.1in, trim = {2.4cm 1.7cm 2.1cm 1.4cm}, clip]{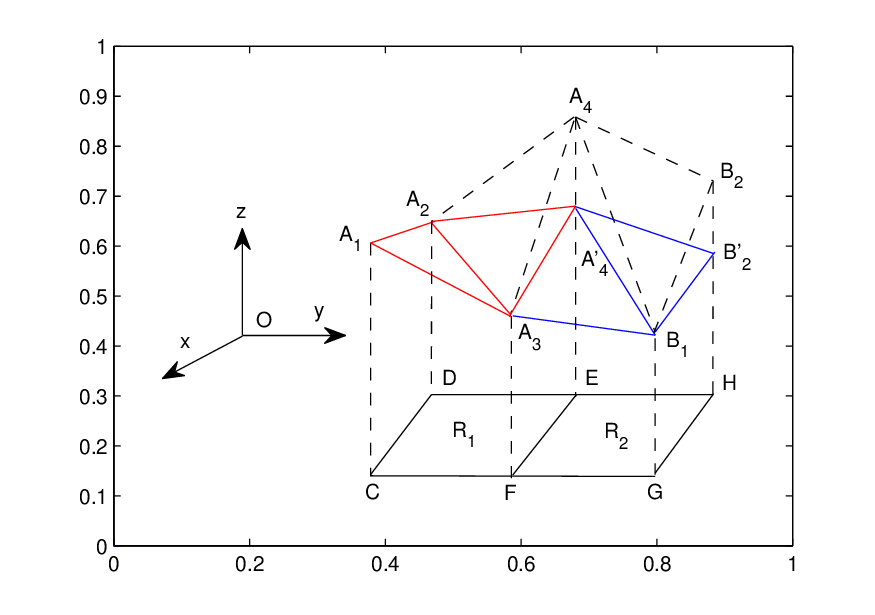}
\caption{Continuous linear pieces on regions.}
\label{Fig.6}
\end{figure}

\begin{lem}
Any $C^1$-function $f: U = [0, 1]^n \to \mathbb{R}$ for $n \ge 2$ can be approximated by a continuous piecewise linear function $\hat{f}(\boldsymbol{x}) \in \mathfrak{K}_n(H)$ of equation 6.1 with arbitrary precision, provided that $\max\{v_i: i = 1, 2, \dots, \zeta\}$ is sufficiently small, where $v_i$ is the volume of $R_i \in \mathcal{R}$.
\end{lem}
\begin{proof}
The formal definitions of the associated terminology of polytopes, such as \textsl{facet} and \textsl{simplex}, can be found in \citet*{Grunbaum2003}.

\textbf{Case of input dimensionality $n = 2$}. We first prove the case of the two-dimensional input. If each region $R_i$ is a triangle, the construction of $\hat{f}(\boldsymbol{x})$ is trivial, since the three vertices of $R_i$ can determine a unique linear function on $R_i$ approximating $f(\boldsymbol{x})$. So the key point is of a polygon region with more than three vertices. We introduce a method with the aid of Figure \ref{Fig.6}.

In Figure \ref{Fig.6}, $R_1$ and $R_2$ are two adjacent regions as well as two quadrilaterals. Let $f(\boldsymbol{x})$ be a $C^1$-function defined on $R = R_1 \cup R_2$. $A_j$'s for $j = 1, 2, 3, 4$ are the four points of $f(\boldsymbol{x})$ corresponding to the four vertices of $R_1$, respectively, and are not coplanar or on a common plane; $B_1$ and $B_2$ along with $A_3$ and $A_4$ shared with $R_1$ are the function points of the vertices of $R_2$ and are also not coplanar. We want to construct a piecewise linear function $\hat{f}(\boldsymbol{x})$ on $R$ continuous at knot $EF = R_1 \cap R_2$ to approximate $f(\boldsymbol{x})$.

Select points $C$, $D$ and $F$ to form a triangle (i.e., a 2-simplex). Then the corresponding  three points $A_1 =(C,f(C))$, $A_2 = (D, f(D))$ and $A_3 = (F, f(F))$ yield a plane $A_1A_2A_3$ that intersects line $\bar{A}_4E$ (a notation from the proof of theorem 2) at point $A_4'$. Quadrilateral $A_1A_2A_4'A_3$ is of the linear function $s_1$ on $R_1$. Triangle $\Delta A_4'A_3B_1$ determines a plane intersecting $\bar{B}_2H$ at $B_2'$ and the quadrilateral $A_4'A_3B_1B_2'$ is of the linear function $s_2$ on $R_2$, which is continuous with $s_1$ at knot $EF$. This completes the construction of $\hat{f}(\boldsymbol{x})$.

To the approximation capability, for example, as the area of $R_1$ becomes smaller and smaller, the $C^1$-function $f(\boldsymbol{x})$ on $R_1$ will be more and more similar to a planar, such that it could be approximated by $A_1A_2A_4'A_3$ with arbitrary precision. This proves this lemma for this example.

If there are more than two regions $R_i$ for $i = 1, 2, \dots, \zeta$, we can choose any one of them, say, $R_{\nu}$ for $1 \le \nu \le \zeta$, to be like $R_{1}$ of Figure \ref{Fig.6} and construct a linear function on it by arbitrary three of its vertices. Then the linear functions on the adjacent regions of $R_{\nu}$ can be implemented similarly to that on $R_{2}$ of Figure \ref{Fig.6}. Each of the adjacent regions could be again regarded as a seed analogous to $R_{\nu}$ to further process the new regions adjacent to it. This procedure can be done recursively.

There's a case that should be specially dealt with. When the linear functions on two unparallel sides of a region have been implemented through its adjacent regions, the linear function on it is already determined and needs not be constructed by the above method; this principle is called ``boundary determination'' in this proof.

Thus, after each step of constructing or determining a function for a region $R_{\tau}$ with $1 \le \tau \le \zeta$ via the previous two ways, we should check whether a region adjacent to $R_{\tau}$ satisfies the boundary-determination condition to avoid the case that more than one pair of unparallel sides has been provided with linear functions. The recursive application of the above two procedures could lead to a piecewise linear function approximating $f(\boldsymbol{x})$ continuous at the lines dividing $U = [0, 1]^2$.

Note that the four vertical or horizontal lines that shape $U = [0, 1]^2$ should be considered as playing the same role in forming the regions of $\mathcal{R}$ as the ones of $H$ that divide $U$. Under this condition, each of $\mathcal{R}$ is either a triangle or polygon with at least three sides, including the marginal ones, such that the above proof is applicable to all the regions of $U$.

During the construction process, it is possible that the error introduced in one region may be propagated and accumulated in the subsequent steps. We solve this problem by the following method. To the first region $R_{\nu}$ processed, if
\begin{equation*}
\max_{x\in R_{\nu}}{|f(\boldsymbol{x})-p_{\nu}}(\boldsymbol{x})|>\varepsilon,
\end{equation*}
where $p_{\nu}(\boldsymbol{x})$ is the linear function constructed on $R_{\nu}$ and $\varepsilon$ is a positive real number that could be arbitrarily small, use additional lines to further divide $R_{\nu}$ until one of the newly obtained subregions is small enough to make the maximum difference between $f(\boldsymbol{x})$ and the constructed linear function as above (we call it the maximum error of the region) is less than $\varepsilon$; we update $R_{\nu}$ by this new region. Accordingly, the maximum error on the boundary of $R_{\nu}$ is also bounded by $\varepsilon$.

Then we should preserve this maximum error $\varepsilon$ for each region as follows. To the second region $R_{\tau}$ adjacent to $R_{\nu}$, the maximum error on $R_{\tau} \cap R_{\nu}$ is ensured by the previous step. If the maximum difference for the linear function $p_{\tau}(\boldsymbol{x})$ on $R_{\tau}$ is larger than $\varepsilon$, we partition $R_{\tau}$ into small regions until one of them adjacent to $R_{\nu}$ changes this situation, contributing to the updated version of $R_{\tau}$. Note that this operation may disturb the previously processed $R_{\nu}$. If the newly dividing of $R_{\tau}$ also partitions $R_{\nu}$, select one of the subregions as a seed and first construct the linear functions on $R_{\nu}$. Since $R_{\nu}$ is small enough for $\varepsilon$, so are its subregions. Then both $R_{\nu}$ and $R_{\tau}$ are bounded by $\varepsilon$ and the maximum error $\varepsilon$ is propagated from $R_{\nu}$ to $R_{\tau}$.

The above process can be repeatedly done. When the further partition of a region also divides the preceding regions, the construction must go back to the earliest step whose associated region is influenced by the new partition, while the the maximum error $\varepsilon$ is ensured for this revision since the original regions are already small enough. By these recursive operations, $\varepsilon$ can be propagated to new regions and be preserved in old ones simultaneously, until all of them fill $U=[0,1]^2$.

By this construction method, it's easy to prove the 2-norm distance $\|f(x)-\hat{f(x)}\|_2 < \varepsilon$, since
\begin{equation*}
\begin{aligned}
\|f(\boldsymbol{x})-\hat{f}(\boldsymbol{x}) \|_2^2 = \int_U|f(\boldsymbol{x})&-\hat{f}(\boldsymbol{x})|^2d\boldsymbol{x}=\sum_i\int_{R_i}|f(\boldsymbol{x})-p_i(\boldsymbol{x}))|^2d\boldsymbol{x} \\
&\le \varepsilon^2\sum_i\int_{R_i}d\boldsymbol{x} = \varepsilon^2.
\end{aligned}
\end{equation*}

Despite the error control given above, when each region is sufficiently small, the maximum error $\varepsilon$ can be automatically propagated through the regions without needing additional partition, in terms of the preceding way. So the general condition that `` each region is small enough'' is sufficient to guarantee a solution of this lemma, but which may not be the optimal one in the sense that some regions are possibly over partitioned into too small pieces.

\textbf{Case of input dimensionality $n \ge 3$}. When $n \ge 3$, taking the hyperplanes that form $U$ into consideration, each region $R_{i} \in \mathcal{R}$ is a polytope and has at least two facets with dimensionality $n-1$ (i.e., $n-1$-dimensional sides) that are part of the boundary of $R_i$. Select arbitrary one $R_{\nu} \in \mathcal{R}$ as the ``seed''. On one of the facets of $R_{\nu}$, denoted by $f_1$, choose $n$ points $p_{\kappa}$ for $\kappa = 1, 2, \dots, n$ satisfying that $\boldsymbol{e}_{\mu}$'s are linearly independent of each other, where $\boldsymbol{e}_{\mu} = p_{\mu+1} - p_1$ for $\mu = 1, 2, \dots, n-1$. Let $p_{n+1}$ be a point of any other facet with $p_{n+1} \notin l_1$, where $l_1$ is the $n-1$-dimensional hyperplane that $f_1$ lies on. Then $\boldsymbol{e}_{n} = p_{n+1} - p_1$ is linearly independent of $\boldsymbol{e}_{\mu}$'s. Thus, points $p_j$'s for $j = 1, 2, \dots, n+1$ could form an $n$-simplex $S_{\nu}$ (i.e., an $n$-dimensional hypertriangle). Correspondingly, the function values $f(p_{j})$'s can generate an $n$-dimensional linear approximation to $f(\boldsymbol{x})$ on $R_{\nu}$.

Notice that in order to grasp the main feature of $f(\boldsymbol{x})$ on $R_{\nu}$, the volume of $S_{\nu}$ should be as large as possible when choosing the $n+1$ points. However, when $R_{\nu}$ is small enough, $f(\boldsymbol{x})$ could be approximately linear to any desired accuracy, such that the different selection of $p_j$'s would have a limited influence.

Up to now, we have construct a linear function $s_{\nu}$ on $R_{\nu}$ approximating $f(\boldsymbol{x})$. To each of the adjacent regions of $R_{\nu}$, say, $R_{\mu}$, $\mathscr{L}_{\mu} = R_{\mu} \cap R_{\nu}$ is $n-1$-dimensional. The $n-1$-dimensional linear function on $\mathscr{L}_{\mu} \subset R_{\nu}$ has been implemented previously, which is part of $s_{\nu}$. So to determine an $n$-dimensional linear piece on $R_{\mu}$, only one additional point of $f(\boldsymbol{x})$ is required, which can be chosen by the method of $p_{n+1}$ for $R_{\nu}$ above. After all of the adjacent regions of $R_{\nu}$ having been processed as $R_{\mu}$, we regard $R_{\mu}$ as a new seed to further deal with the regions adjacent to it.

In a higher-dimensional input, the boundary-determination principle holds as well. A key point is that two $n-1$-dimensional hyperplanes $l_1$ and $l_2$ of $\mathbb{R}^{n+1}$ with $\mathscr{L} = l_1 \cap l_2 \ne \emptyset$ and $l _1 \ne l_2$ can from a unique $n$-dimensional hyperplane $\mathcal{L}$. In fact, we can construct a coordinate system $\mathcal{C} = (O, \{\boldsymbol{e}_{\nu}: \nu = 1, 2, \dots, n\})$ from $l_1$ and $l_2$ to span $\mathcal{L}$. Select an arbitrary point of $\mathscr{L}$ as the original point $O$, and $n-1$ linearly independent vectors $\boldsymbol{e}_{\mu}$'s for $\mu = 1, 2, \dots, n-1$ from $l_1$. Any vector of $l_2$ but not $l_1$ is chosen to be $\boldsymbol{e}_n$. Then $\mathcal{L}$ is uniquely determined by $\mathcal{C}$. This conclusion results in the boundary determination for a higher-dimensional input. So as the two-dimensional case, after a region $R_{\tau}$ being dealt with by either of the above two methods, check whether some one adjacent to $R_{\tau}$ satisfies the condition of boundary determination. Repeat the two procedures until all the $R_i$'s are processed.

When some region is not small enough, to avoid error propagation and accumulation, the maximum-error method as the two-dimensional case should be used. To this point, there's no essential difference between higher-dimensional input and two-dimensional input. If all the regions are sufficiently small, additional partition is not necessarily required. The approximation-error analysis is similar to the two-dimensional case. This completes the proof.
\end{proof}

\subsection{Multiple Strict Partial Orders}
\begin{dfn}[Initial region and initial linear function]
Under definition 3 of section 3.1, to the strict partial order $\mathscr{P} := l_1 \prec l_2 \prec \dots \prec l_{\zeta}$ and its ordered regions $\mathcal{R} = \{R_i: i =1, 2, \dots, \zeta\}$, there exists a region $R_0 \subset l_1^0$ satisfying $L = R_0 \cap R_1 \subset l_1$ and $\dim(L) = n-1$. A linear function on $R_1$ can be constructed from the one on $R_0$ through the unit of $l_1$. We call $R_0$ the initial region of $\mathscr{P}$ and the linear function on $R_0$ is called the initial one of those on $\mathcal{R}$.
\end{dfn}

\begin{thm}[Approximation over multiple strict partial orders]
Suppose that a set $\mathcal{H} = \{l_1, l_2, \dots, l_N\}$ of $n-1$-dimensional hyperplanes of $\mathbb{R}^n$ for $n \ge 2$ divides $U = [0, 1]^n$ into $M$ regions
\begin{equation}
\mathcal{R} = \{R_m: 1 \le  m \le M\},
\end{equation}
and that $\mathcal{H}$ can form $\psi$ strict partial orders on $U$ denoted by
\begin{equation}
\{\mathscr{P}_i = (H_i, \prec): i = 1, 2, \dots, \psi\},
\end{equation}
where
\begin{equation}
H_i = \{l^{(i)}_j: j = 1, 2, \dots, \psi_i\} \subset \mathcal{H}
\end{equation}
with $l^{(i)}_1 \prec l^{(i)}_2 \prec \dots \prec l^{(i)}_{\psi_i}$ and $\psi_i = |H_i|$, and where
\begin{equation}
H_{\nu} \cap H_{\mu} = \emptyset
\end{equation}
for $1 \le \nu, \mu \le \psi$ and $\nu \ne \mu$. Write
\begin{equation}
\mathcal{H} = \bigcup_{i = 1}^{\psi}H_i \cup \mathscr{H},
\end{equation}
where
\begin{equation}
|\mathscr{H}| \ge n+1
\end{equation}
and each $l \in \mathscr{H}$ satisfies $U \subset l^+$. Each $\mathscr{P}_i$ generates a set
\begin{equation}
\{R_j^{(i)}\} := \{R_j^{(i)}: j = 1, 2, \dots, \psi_i\} \subset \mathcal{R}
\end{equation}
of ordered regions, with $\{R_j^{(\nu)}\} \cap \{R_j^{(\mu)}\} = \emptyset$, where $\nu$ and $\mu$ are as in equation 6.5. Assume that the following four conditions are satisfied:

\begin{itemize}
\item[\rm{\Romannum{1}}.] The set of equation 6.2 can be expressed as
\begin{equation}
\mathcal{R} = \bigcup_{i = 1}^{\psi_i}\{R_j^{(i)}\} \cup \mathfrak{R}_{0} \cup \mathscr{R},
\end{equation}
where $\mathfrak{R}_{0}$ is one of the regions of $\mathcal{R}$ and $\mathscr{R}$ is a set of regions whose total volume could be arbitrarily small.

\item[\rm{\Romannum{2}}.] The strict partial orders $\mathscr{P}_i$'s can be arranged in an order
\begin{equation}
\mathscr{P}_{i_1}, \mathscr{P}_{i_2}, \dots, \mathscr{P}_{i_{\psi}},
\end{equation}
where $1 \le i_{\tau} \le \psi$ for $\tau = 1, 2, \dots, \psi$, such that
\begin{equation}
\mathfrak{R}_{0} \subseteq \bigcap_{k = 1}^{\psi_{i_1}}l^{(i_1)0}_{k}
\end{equation}
and
\begin{equation}
\{R_j^{(i_{\nu})}\} \cup \mathfrak{R}_{0} \subseteq \bigcap_{u = \nu+1}^{\psi}\bigcap_{k = 1}^{\psi_{i_\mu}}l^{(i_{\mu})0}_{k}
\end{equation}

for $1 \le \nu \le \psi-1$.

\item[\rm{\Romannum{3}}.] Let $R^{(i)}_0$ be the initial region of order $\mathscr{P}_i$ and suppose that
\begin{equation}
R^{(i)}_0 \subseteq \bigcap_{k = 1}^{\xi_i}l^{+}_{n_{ik}}
\end{equation}
for $1 \le n_{ik} \le N$, where $l_{n_{ik}}$'s for $k$ contained in $\mathcal{H}$ are all the hyperplanes whose positive-output regions contain $R^{(i)}_0$ and $N = |\mathcal{H}|$. So the initial linear function $s_0^{(i)}$ on $R^{(i)}_0$ is produced by the units of $l_{n_{ik}}$'s. Then all of the ordered regions of $\mathscr{P}_i$ satisfy
\begin{equation}
\{R_j^{(i)}\} \subseteq \bigcap_{k = 1}^{\xi_i}l^{+}_{n_{ik}},
\end{equation}
similar to $R^{(i)}_0$ of equation 6.13.

\item[\rm{\Romannum{4}}.] The implementation of the initial linear function $s_0^{(i)}$ on $R_0^{(i)}$ of $\mathscr{P}_i$ for all $i$ has three possible cases: (a) the first is for the region $\mathfrak{R}_{0}$ of equation 6.9 on which the linear function is determined by the units of $\mathscr{H}$ of equation 6.6, with the associated linear-output matrix nonsingular; (b) the second is that the initial region $R_0^{(i)}$ happens to be one of the ordered regions of some other $\mathscr{P}_{\nu}$ whose linear functions have already been constructed; (c) the third is that although $s_0^{(i)}$ is not exclusively designed by the above two methods, the first linear function $s_1^{(i)}$ on $R_1^{(i)}$ is continuous with $s_0^{(i)}$.
\end{itemize}
Let $\mathfrak{N}$ be a two-layer ReLU network whose units of the hidden layer is derived from the hyperplanes of $\mathcal{H}$. Then if the maximum volume of $R_m$'s of equation 6.2 is sufficiently small, any $C^1$ function $f : U \to \mathbb{R}$ can be approximated by $\mathfrak{N}$ with arbitrary precision, in terms of the implementation of a continuous piecewise linear function $\hat{f}(\boldsymbol{x}) \in \mathfrak{K}_n(\mathcal{H})$. Suppose that $\mathscr{R} = \emptyset$ of equation 6.9 and that to a certain approximation error $\epsilon$, the number of the linear pieces of $\hat{f}(\boldsymbol{x})$ required is $\zeta$; then the number of the units of the hidden layer of $\mathfrak{N}$ satisfies
\begin{equation}
\Theta \ge \zeta + n.
\end{equation}
\end{thm}
\begin{proof}
First, by lemma 6, there exists a continuous piecewise linear function $\hat{f}(\boldsymbol{x}) \in \mathfrak{K}_n(\mathcal{H})$ approximating $f(\boldsymbol{x})$ with arbitrary precision. Then we show how to implement $\hat{f}(\boldsymbol{x})$ via $\mathfrak{N}$. The construction is decomposed into $\psi$ steps and must be in the order $\mathscr{P}_{i_1}, \mathscr{P}_{i_2}, \dots, \mathscr{P}_{i_{\psi}}$ of equation 6.10 of condition \Romannum{2}.

Equation 6.10 of condition \Romannum{2} ensures that the operation on $\mathscr{P}_{i_{\mu}}$ for $2 \le \mu \le \psi$ cannot influence the accomplished results of $\mathscr{P}_{i_{\kappa}}$ for $\kappa < \mu$. To each strict partial order $\mathscr{P}_i$, condition \Romannum{3} is necessary for the recursive production of the linear functions on $\{R_j^{(i)}\}$ through equation 4.68, by preserving the influences of $l^{+}_{n_{ik}}$'s of equation 6.14 on the initial region $R_0^{(i)}$ for $\{R_j^{(i)}\}$. By condition \Romannum{4}, the initial linear function of each $\mathscr{P}_i$ has been provided, so the piecewise linear function desired on $\{R_j^{(i)}\}$ can be realized by theorem 3.

Note that in the order of equation 6.10, the influence (if any) of $\mathscr{P}_{i_{\nu}}$ on $\mathscr{P}_{i_{\mu}}$ for $1 \le \nu < \mu \le \psi$ is embedded in the initial linear function $s_0^{(\mu)}$ of $\mathscr{P}_{i_{\mu}}$ according to condition \Romannum{3} and thus can be resolved. Pay attention to the specialty of $\mathfrak{R}_{0}$ of equation 6.9. Because the units of $\mathscr{H}$ could influence the whole $U$, they can only be used for one time and only one region of $\mathcal{R}$ of equation 6.2 can be chosen as $\mathfrak{R}_{0}$.

The regions of $\mathscr{R}$ of equation 6.9 are not ordered by $\mathscr{P}_i$'s and the linear functions on them cannot be designed. However, by condition \Romannum{1}, their total volume could be arbitrarily small, such that their influence can be ignored to any degree of accuracy, since smaller volume of $\mathscr{R}$ results in less contribution to the distance between $\hat{f}(\boldsymbol{x})$ and the function generated by $\mathfrak{N}$.

To the number of the units of the hidden layer of $\mathscr{R}$ required, in the case $\mathscr{R} = \emptyset$ of equation 6.9, all the regions of $\mathcal{R}-\mathfrak{R}_{0}$ are ordered, including the initial ones of the strict partial orders $\mathscr{P}_i$'s. So each region of $U$ except for $\mathfrak{R}_{0}$ needs only one unit to shape its linear function. The linear function on $\mathfrak{R}_{0}$ should use at least $n+1$ units by lemma 3. Thus, the minimum number of the units needed is $n+1+\zeta-1 = \zeta+n$, which proves the inequality 6.15.
\end{proof}

\begin{rmk-2}
Compared to theorem 1 for the one-dimensional input, the case of the $n$-dimensional input for $n \ge 2$ is much more complex owning to multi-directions, such that a single strict partial order with sufficiently small ordered regions cannot be easily constructed to cover the whole $U = [0, 1]^n$, for which multiple strict partial orders were introduced.
\end{rmk-2}

\begin{rmk-2}
Corollary 2 of multiple representations of a linear piece plays a role in the function construction of this theorem, through which the output weight of a unit can be obtained based on the local geometric information of the hyperplanes, regardless of the disturbance of some other ones.
\end{rmk-2}

\begin{figure}[!t]
\captionsetup{justification=centering}
\centering
\includegraphics[width=2.8in, trim = {2.8cm 1.5cm 1.7cm 1.0cm}, clip]{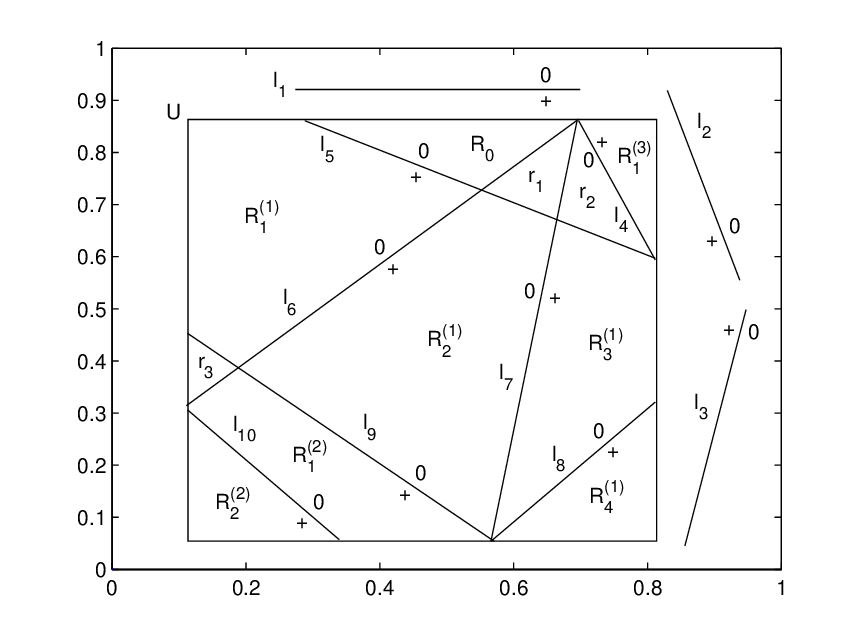}
\caption{Multiple strict partial orders.}
\label{Fig.7}
\end{figure}

\begin{dfn}[Universal global unit (hyperplane)]
Any one of the units (hyperplanes) of $\mathscr{H}$ of equation 6.6 that could influence the whole universal set $U$ in terms of nonzero output is called a \textsl{universal global unit} (\textsl{hyperplane}) of $U$.
\end{dfn}

\noindent
\textbf{Example}. Figure \ref{Fig.7} gives an example of theorem 7. There are three local strict partial orders, including $\mathscr{P}_1 = l_5 \prec l_6 \prec l_7 \prec l_8$, $\mathscr{P}_2 = l_9 \prec l_{10}$, and $\mathscr{P}_3 = l_4 \prec l_{\infty}$, which are manifested by the superscripts of the ordered regions (e.g., $R^{(1)}_2$ for $\mathscr{P}_1$). The set $\mathscr{H}$ of equation 6.6 of the universal global hyperplanes is $\mathscr{H} = \{l_1, l_2, l_3\}$. The order of condition \Romannum{2} could be either $\mathscr{P}_1, \mathscr{P}_2, \mathscr{P}_3$ or $\mathscr{P}_1, \mathscr{P}_3, \mathscr{P}_2$, both of which satisfy equations 6.11 and 6.12 ($R_0$ of this figure corresponding to $\mathfrak{R}_0$ of equation 6.9).

The initial region of $\mathscr{P}_1$ is $R^{(1)}_0 = R_0$ whose linear function is implemented by the units of $\mathscr{H}$ via lemma 3, while that of $\mathscr{P}_2$ is $R^{(2)}_0 = R^{(1)}_2$ whose linear function has been provided by $\mathscr{P}_1$. The $R^{(3)}_0 = r_2$ of $\mathscr{P}_3$ is the third case of condition \Romannum{4} and its linear function is assumed to be continuous with the one on $R^{(3)}_1$.

The function construction must be in the order $\mathscr{P}_1, \mathscr{P}_2, \mathscr{P}_3$ or $\mathscr{P}_1, \mathscr{P}_3, \mathscr{P}_2$. Using the first one, for example, because condition \Romannum{2} is satisfied, the operation of greater $i$ for $\mathscr{P}_i$ with $i = 1, 2, 3$ cannot affect the accomplished results of less $i$. The influence (if any) of less $i$ on greater $i$ can be embedded in the initial function of the latter via condition \Romannum{3}, as mentioned in the proof of theorem 7. To each isolated $\mathscr{P}_i$, the method of equation 4.68 is used to generate the desired piecewise linear function, on the basis of the initial function that has already been provided.

In this example of Figure \ref{Fig.7}, the set $\mathscr{R}$ of equation 6.9 is $\mathscr{R} = \{r_1, r_2, r_3\}$, on which the linear functions are not defined by theorem 7. We assumed that the total area (volume) of $\mathscr{R}$ could be arbitrarily small in theorem 7 for function approximations, and the region dividing of Figure \ref{Fig.7} can be further done based on the existing result to make the area of $\mathscr{R}$ smaller.

\subsection{Two-Sided Bases}

\begin{dfn}[Order tree]
In equation 6.9, if $\mathscr{R} = \emptyset$, all the initial linear functions of $\mathscr{P}_{i}$'s of equation 6.3 are directly or indirectly derived from the one $\mathfrak{S}(\boldsymbol{x})$ on $\mathfrak{R}_{0}$ of equation 6.9. More specifically, each initial linear region $R_0^{(i)}$ of $\mathscr{P}_i$ is either $\mathfrak{R}_{0}$ or an ordered region of some $\mathscr{P}_j$ for $1 \le j < i$. To the second case, the initial region $R_0^{(j)}$ of $\mathscr{P}_j$ can be traced back to $\mathfrak{R}_{0}$ or some ordered region of $\mathscr{P}_k$ for $k < j$; this tracing procedure can be done recursively until $\mathfrak{R}_{0}$ is reached. We then call the set $\{\mathscr{P}_i: i= 1, 2, \dots, \psi\}$ an \textsl{order tree} $\mathcal{T}$ and $\mathfrak{S}(\boldsymbol{x})$ (or $\mathfrak{R}_{0}$) is called the \textsl{initial linear function} (or \textsl{initial region}) of $\mathcal{T}$.
\end{dfn}

\noindent
\textbf{Example}. In Figure \ref{Fig.7}, there are two order trees $\{\mathscr{P}_1, \mathscr{P}_2\}$ and $\{\mathscr{P}_3\}$.

\begin{dfn}[Negative form of hyperplanes]
Given an $n-1$-dimensional hyperplane $\mathcal{L}$ of $\mathbb{R}^n$ whose equation is $\boldsymbol{w}^T\boldsymbol{x} + b =0$, let $-\mathcal{L}$ be the same hyperplane but with the equation changed to be $-\boldsymbol{w}^T\boldsymbol{x} - b =0$ whose associated ReLU is $\sigma(-\boldsymbol{w}^T\boldsymbol{x} - b)$ instead of the original $\sigma(\boldsymbol{w}^T\boldsymbol{x} + b)$. We call $-\mathcal{L}$ the negative form of $\mathcal{L}$.
\end{dfn}

\begin{lem}
Notations being from theorem 7, to each strict partial order $\mathscr{P}_{i_{\kappa}}= (H_{i_{\kappa}}, \prec)$ for $\kappa = 1, 2, \dots, \psi$ of equation 6.10, let
\begin{equation}
l^{(i_{\kappa})}_{\mu} \in H_{i_{\kappa}}
\end{equation}
be an arbitrary hyperplane of $\mathscr{P}_{i_{\kappa}}$
and suppose that
\begin{equation}
\mathscr{R} = \emptyset,
\end{equation}
where $\mathscr{R}$ is from equation 6.9. Then the units of $\mathscr{P}_{i_{\kappa}}$ can be modified to be of two-sided bases by changing $l^{(i_{\kappa})}_{\mu}$ into $-l^{(i_{\kappa})}_{\mu}$ or adding $-l^{(i_{\kappa})}_{\mu}$ at the knot $l^{(i_{\kappa})}_{\mu}$, without influencing the conclusion of theorem 7. The number of $l^{(i_{\kappa})}_{\mu}$'s as well as that of $\mathscr{P}_{i_{\kappa}}$'s for equation $6.16$ could be more than one.
\end{lem}
\begin{proof}
Under the condition of equation 6.17, there exists a unique order tree and the initial linear function of each strict partial order of equation 6.10 is either $\mathfrak{S}(\boldsymbol{x})$ on $\mathfrak{R}_0$ or an accomplished one on some ordered region.

By the definition of strict partial orders and condition \Romannum{2} of theorem 7, if $l^{(i_{\kappa})}_{\mu}$ in equation 6.16 is substituted with $-l^{(i_{\kappa})}_{\mu}$, it will influence the linear functions on $\mathcal{R}_1:= \{R_j^{(i_{\kappa})}\}$ for $1 \le j \le \mu-1$ of $\mathscr{P}_{i_{\kappa}}$ when $\mu \ge 2$, on $\mathcal{R}_2 := \{R_j^{(i_{\nu})}\}$ for $1 \le \nu \le \kappa -1$ of other strict partial orders when $\kappa \ge 2$, and on $\mathfrak{R}_0$. If $\kappa = 1$ or the initial linear function $s^{(i_{\kappa})}_0$ of $\mathscr{P}_{i_{\kappa}}$ is the one $\mathfrak{S}(\boldsymbol{x})$, the impact of $-l^{(i_{\kappa})}_{\mu}$ can be embedded in the expression of $\mathfrak{S}(\boldsymbol{x})$, without affecting the function construction of $\mathscr{P}_{i_{\kappa}}$, according to the principle of the substituted two-sided bases of theorems 4 and 5. The remaining strict partial orders are not influenced due to the invariance of the piecewise linear function of $\mathscr{P}_{i_{\kappa}}$.

If $\kappa \ne 1$ or the initial linear function $s^{(i_{\kappa})}_0 \ne \mathfrak{S}(\boldsymbol{x})$, there exist some strict partial orders between $\mathscr{P}_{i_1}$ and $\mathscr{P}_{i_{\kappa}}$ in equation 6.10, which are collectively denoted by the set $\mathcal{P}$. We also embed the influence of $-l^{(i_{\kappa})}_{\mu}$ in $\mathfrak{S}(\boldsymbol{x})$ such that the linear functions of $\mathcal{P}$ are not affected. Then the initial linear function $s^{(i_{\kappa})}_0$ of $\mathscr{P}_{i_{\kappa}}$ derived from some ordered region of $\mathcal{P}$ could remain invariant, though disturbed by $-l^{(i_{\kappa})}_{\mu}$; this again leads to the case of the two-sided bases of theorems 4 and 5. So the piecewise linear function of $\mathscr{P}_{i_{\kappa}}$ can be constructed, by which those of $\mathscr{P}_{i_{\tau}}$'s for $\tau > \kappa$ would not be influenced.

The general case for more than one $l^{(i_{\kappa})}_{\mu}$ or $\mathscr{P}_{i_{\kappa}}$ of equation 6.16 is trivial under the discussion above, and the case of added two-sided bases can be similarly dealt with based on theorems 4 and 5.
\end{proof}

\begin{thm}[Two-sided bases for multiple strict partial orders]
Given an order tree $\mathcal{T}$, especially with condition \Romannum{2} of theorem 7 being satisfied, any one of its hyperplanes could be changed into its negative form to replace the original positive unit or add a new negative unit, such that two-side bases could be formed, without influencing the construction of the piecewise linear function of $\mathcal{T}$.
\end{thm}
\begin{proof}
This conclusion is directly from the proof of lemma 7.
\end{proof}

\begin{rmk-3}
The regions under the two-sided bases of theorem 8 will still be regarded as ordered ones as the one-sided case.
\end{rmk-3}

\begin{rmk-3}
This theorem plays an important role in explaining training solutions, since the case of one-sided bases rarely occurs compared to that of two-sided bases.
\end{rmk-3}

\noindent
\textbf{Example}. In Figure \ref{Fig.7}, any line of order tree $\mathcal{T} = \{\mathscr{P}_1, \mathscr{P}_2\}$ could be changed into its negative form to yield two-sided bases, without affecting the realization of the piecewise linear function of $\mathcal{T}$.

\section{Continuity Restriction}
There may exist some regions that do not belong to any strict partial order, as discussed in theorem 7 of section 6. This is a typical problem imposing ubiquitous influences. We solve it by adding a new principle called ``continuity restriction'' (theorem 9), which proves to be fundamentally important in both developing the approximation theory (theorem 10) and explaining the training solution (subsequent section 8.2).

\subsection{Basic Principle}

\begin{figure}[!t]
\captionsetup{justification=centering}
\centering
\subfloat[Adjacent regions.]{\includegraphics[width=2.2in, trim = {3.6cm 3.2cm 3.6cm 1.4cm}, clip]{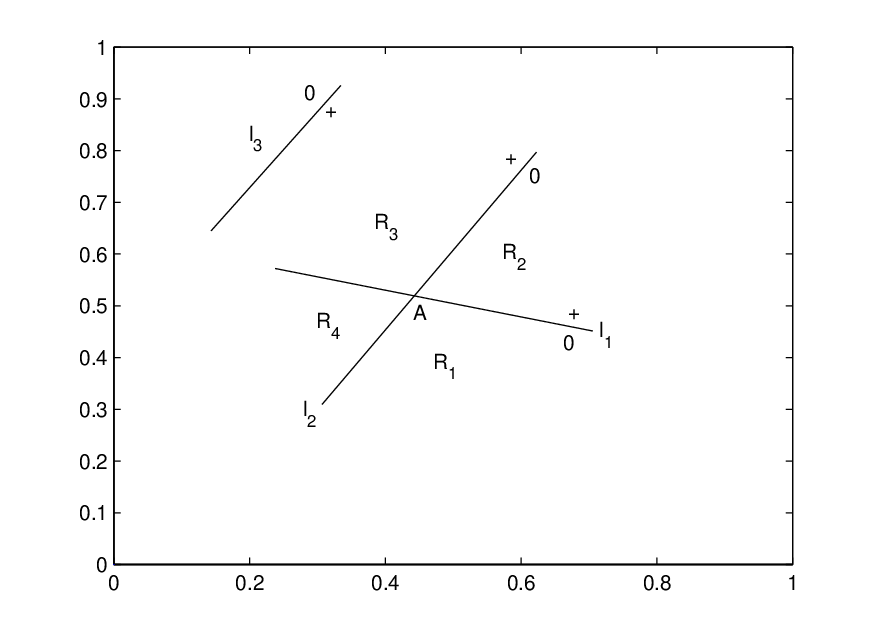}} \quad \quad \quad
\subfloat[Linear pieces on the regions of (a).]{\includegraphics[width=2.7in, trim = {2.7cm 1.7cm 2.7cm 1.3cm}, clip]{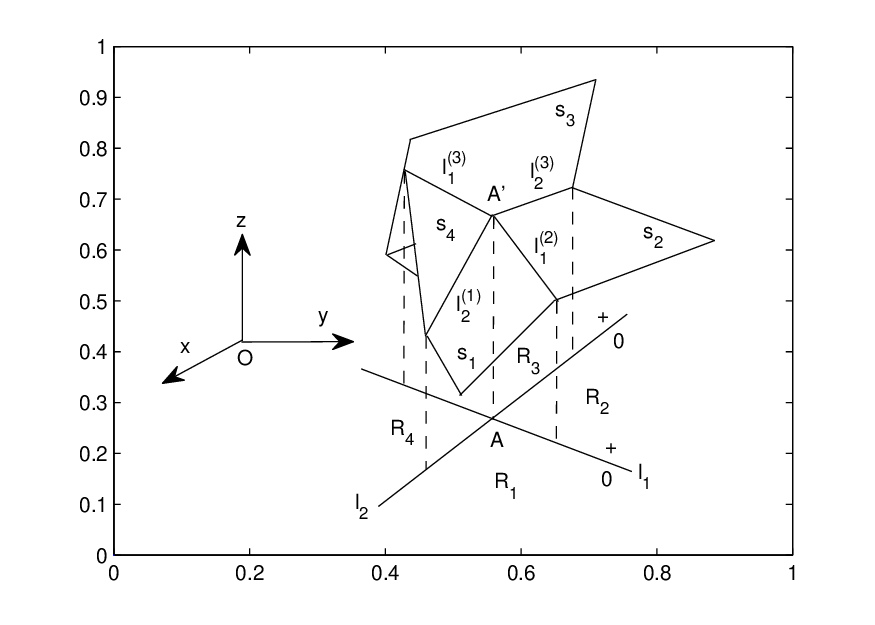}}
\caption{Continuity restriction.}
\label{Fig.8}
\end{figure}

\begin{lem}
Let $R_i$'s for $1 \le i \le 4$ be four regions of $\mathbb{R}^n$ for $n \ge 2$, which are formed by a set $H$ of $n-1$-dimensional hyperplanes and are separated by two hyperplanes $l_1$ and $l_2$ satisfying $l_1 \cap l_2 \ne \emptyset$. Without loss of generality, let $R_i$ be adjacent to $R_{i+1}$ with $R_5 := R_1$, such as in Figure \ref{Fig.8}a. Suppose that there exists a continuous piecewise linear function $\hat{f}(\boldsymbol{x})$ defined on $\mathcal{R} = \bigcup_i R_i$, with the linear piece on $R_i$ being $s_i$. If the linear functions of $\hat{f}(\boldsymbol{x})$ on $R_1$ and $R_3$ have been realized by a two-layer ReLU network $\mathfrak{N}$, whose units of the hidden layer are from the hyperplanes of $H$, then the one on $R_4$ is also implemented by $\mathfrak{N}$.
\end{lem}
\begin{proof}
\textbf{Case of input-dimensionality $n=2$}. In Figure \ref{Fig.8}a, besides $l_1$ and $l_2$ that partition $\mathcal{R}$, for simplicity, only one line $l_3$ is depicted. The lines of $H$ can be classified into two categories: one is like $l_3$ whose positive-output region includes all of $R_i$'s; and the other is composed of those that cannot be activated by any one of $R_i$'s. All the influences of the first category can be embedded in the linear function of any one of $R_i$'s, such as $s_1$ on $R_1$. Then by equation 4.68, $s_2$ and $s_3$ on $R_2$ and $R_3$, respectively, can be constructed on the basis of $s_1$; and after that the parameters of the units of $l_1$ and $l_2$ are fixed and a certain linear function $s_4'$ on $R_4$ is also formed.

We now prove that $s_4' = s_4$. In Figure \ref{Fig.8}b, $s_i$ for $1 \le i \le 4$ is the linear function of $\hat{f}(\boldsymbol{x})$ on $R_i$. Line $l_1^{(2)} \subset s_2$ (the subscript of $s_2$ associated with the superscript of $l_1^{(2)}$) and $l_1^{(2)} = s_1 \cap s_2$ whose projection on plane $xOy$ is on $l_1$ (corresponding to the subscript of $l_1^{(2)}$); the meanings of $l_2^{(3)}$, $l_1^{(3)}$ and $l_2^{(1)}$ are similar. Lines $l_1^{(3)}$ and $l_2^{(1)}$ determine a unique plane $s_4$ continuous with $s_3$ and $s_1$. Suppose that $s_j$'s for $j = 1, 2, 3$ have been previously implemented by network $\mathfrak{N}$ as above, with $l_1^{(3)} \subset s_3$ and $l_2^{(1)} \subset s_1$ also determined. Let $\hat{f}'(\boldsymbol{x})$ be the function on $\mathcal{R}$ realized by $\mathfrak{N}$, which may be different from $\hat{f}(\boldsymbol{x})$ merely in the linear function $s_4'$ on $R_4$. We know that $\hat{f}'(\boldsymbol{x})$ is continuous by corollary 1. To maintain the continuity property of $\hat{f}'(\boldsymbol{x})$ at the knots $l_1$ and $l_2$ for $s_3$, $s'_4$ and $s_1$, $s'_4 = s_4$ is the only choice. This proves $\hat{f}'(\boldsymbol{x}) = \hat{f}(\boldsymbol{x})$.

Another point is that if we change either or both of $l_1$ and $l_2$ into negative forms as in lemma 7, the conclusion still holds. This modification only influences the production way of $s_j$'s and has no impact on the geometric relationships between the linear pieces of $\hat{f}(\boldsymbol{x})$ as well as the continuous property of $\hat{f}'(\boldsymbol{x})$.

We can also summarize the above proof by two intuitive and simple facts: the first is the boundary-determination principle discussed in the proof of lemma 6, through which $s_4$ is uniquely determined by $s_1$ and $s_3$; the second is that the output $\hat{f}'(\boldsymbol{x})$ of network $\mathfrak{N}$ is a continuous piecewise linear function whose parameters are set to produce $s_1$ and $s_3$. Thus $\hat{f}'(\boldsymbol{x})=\hat{f}(\boldsymbol{x})$ on $R_1\cap R_3\cap R_4$.

\textbf{Case of input-dimensionality $n \ge 3$}. The key point of the generalization to $n \ge 3$ lies in the fact that two unparallel $n-1$-dimensional hyperplanes of $\mathbb{R}^{n+1}$ can from a unique $n$-dimensional hyperplane, which has been proved in the proof of lemma 6. Others are similar to the case of $n = 2$. Note that in the higher-dimensional case, the intersection $A = \bigcap_il_i$ is not a point as in Figure \ref{Fig.8}b. However, this is irrelevant to this lemma, since we only use the adjacent relationships of the regions formed by $l_1$ and $l_2$, regardless of their intersection $A$.
\end{proof}

\begin{dfn}[Boundary of a region]
Suppose that $R \subset \mathbb{R}^n$ for $n \ge 2$ is a region formed by a set $H = \{l_1, l_2, \dots, l_{\psi}\}$ of $n-1$-dimensional hyperplanes. The \textsl{boundary} $\mathscr{B}$ of $R$ is defined to be
\begin{equation}
\mathscr{B} = \{\mathcal{L}_{\nu}: \mathcal{L}_{\nu} = R \cap l_{\nu} \ne \emptyset, 1 \le \nu \le \psi \}.
\end{equation}
\end{dfn}

\begin{thm}[Principle of continuity restriction]
Let $R_i$'s for $i = 1, 2, \dots, \zeta$ be the regions of $U = [0, 1]^n$ formed by a set $H = \{l_1, l_2, \dots, l_{\psi}\}$ of $n-1$-dimensional hyperplanes with $n \ge 2$. Write $\mathcal{R} = \bigcup_iR_i$. Denote by $\hat{f}(\boldsymbol{x}) \in \mathfrak{K}_n(H)$ of equation 6.1 a continuous piecewise linear function to be realized by a two-layer ReLU network $\mathfrak{N}$. To a region $R_{\kappa} \in \mathcal{R}$ for $1 \le \kappa \le \zeta$, let $\mathcal{L}_{n_{\kappa\nu}}$ and $\mathcal{L}_{n_{\kappa\mu}}$ for $1 \le n_{\kappa\nu}, n_{\kappa\mu} \le \psi$ be two elements of its boundary satisfying $l_{n_{\kappa\nu}} \cap l_{n_{\kappa\mu}} \ne \emptyset$, where $l_{n_{\kappa\nu}}$ and $l_{n_{\kappa\mu}}$ are the hyperplanes containing $\mathcal{L}_{n_{\kappa\nu}}$ and $\mathcal{L}_{n_{\kappa\mu}}$, respectively, as in equation 7.1. Then, if the function values on $\mathcal{L}_{\kappa_{\nu}} \cup \mathcal{L}_{\kappa_{\mu}}$ have been implemented by $\mathfrak{N}$, the linear function on $R_{\kappa}$ is also simultaneously produced.

\end{thm}
\begin{proof}
The proof is contained in the argument of lemma 8 and this theorem is a more general description of lemma 8.
\end{proof}

\begin{rmk}
Despite its simplicity, this theorem provides the possibility of realizing the linear function on a region that is not included in a strict partial order. In combination with the ordered regions, its effect could be complex through recursive applications, rendering it a fundamental property of the solution of two-layer ReLU networks.
\end{rmk}

\noindent
\textbf{Example}. In Figure \ref{Fig.7}, although the regions $r_1$, $r_2$ and $r_3$ of $\mathscr{R}$ are not ordered, the linear functions on them can all be automatically realized through the continuity restriction of theorem 9, such that any piecewise linear function on the regions of Figure \ref{Fig.7} can be precisely implemented by a two-layer ReLU network. More examples will be shown in sections 7.2 and 8.2.

\subsection{Universal-Approximation Capability}

\begin{figure}[!t]
\captionsetup{justification=centering}
\centering
\subfloat[Principle of recursive procedures.]{\includegraphics[width=2.3in, trim = {3.8cm 2.0cm 2.8cm 1.0cm}, clip]{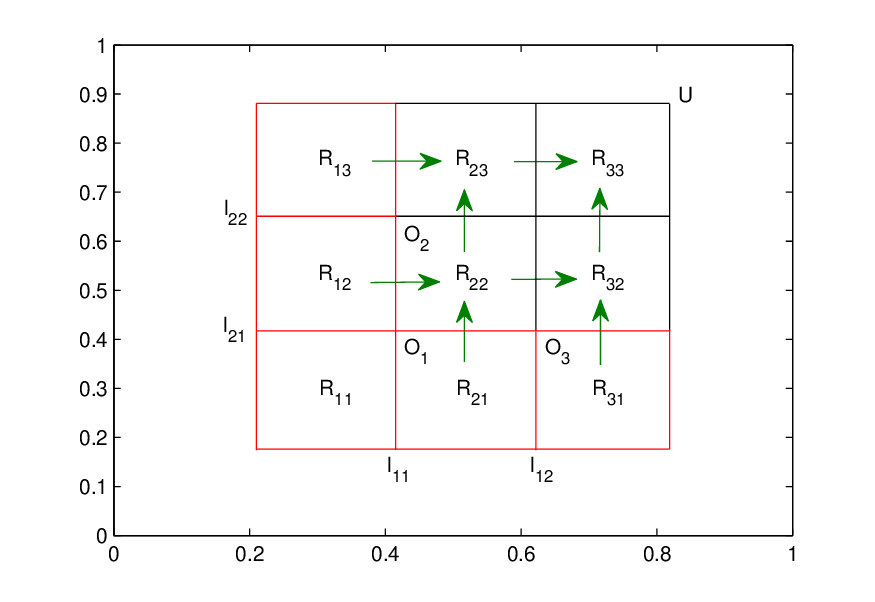}} \quad \quad
\subfloat[Case of the three-dimensional input.]{\includegraphics[width=2.7in, trim = {2.6cm 1.5cm 2.3cm 1.0cm}, clip]{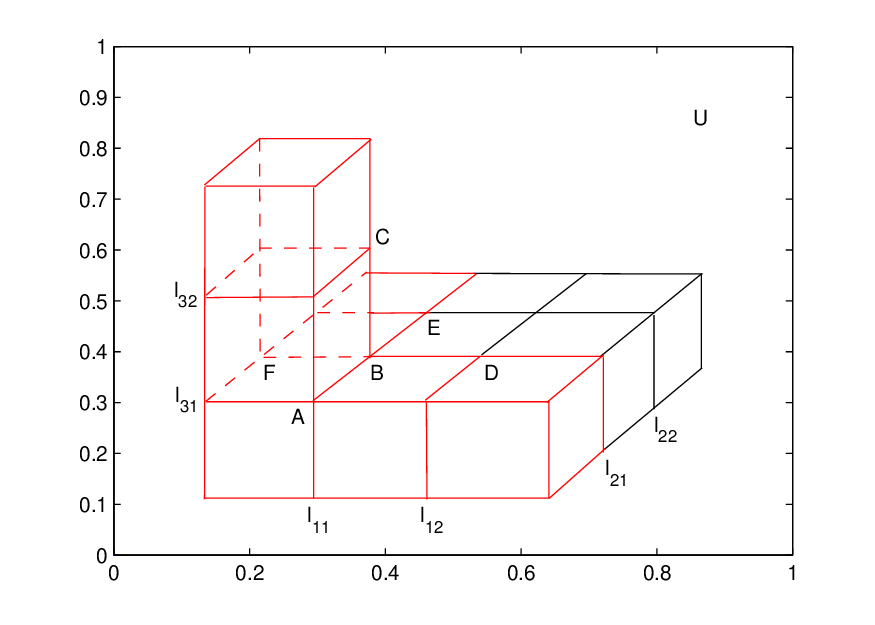}} \\
\subfloat[Function-construction process.]{\includegraphics[width=2.7in, trim = {2.6cm 1.5cm 2.3cm 1.0cm}, clip]{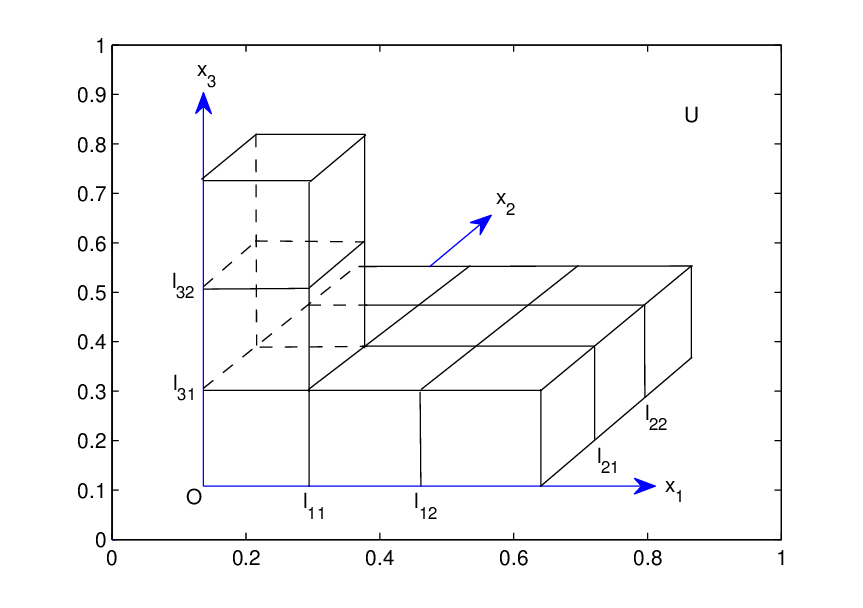}}
\caption{Function approximation via continuity restriction.}
\label{Fig.9}
\end{figure}

\begin{dfn}[Region adjacent to a hyperplane]
Under definition 19, to each hyperplane $l_{j_{\nu}}$ of equation 7.1 satisfying $q = R \cap l_{j_{\nu}} \ne \emptyset$ and $\dim{q} = n-1$, we say that $R$ is \textsl{adjacent} to $l_{j_{\nu}}$.
\end{dfn}

\begin{lem}
Denote by $\boldsymbol{x} = [x_1, x_2, \dots, x_n]^T$ for $n \ge 2$ a vector of $\mathbb{R}^n$ with $x_i$ as the $i$th dimension for $i = 1, 2, \dots, n$. Let $l_{ij}$ be an $n-1$-dimensional hyperplane whose equation is
\begin{equation}
x_i = j / M
\end{equation}
for $j = 1, 2, \dots, M-1$ with $M \ge 2$. The universal set $U = [0,1]^n$ is then divided into a set $\mathcal{R}$ of $M^n$ regions by $l_{ij}$'s for all $i$ and $j$, viz.,
\begin{equation}
|\mathcal{R}| = M^n.
\end{equation}
By equation 6.1, $\hat{f}(\boldsymbol{x}) \in \mathfrak{K}_n(H)$ denotes an arbitrary continuous piecewise linear function defined on $\mathcal{R}$ with
\begin{equation}
H = \{l_{ij}: 1 \le i \le n,  1 \le j \le M-1\}.
\end{equation}
Let
\begin{equation}
\mathcal{B} =  \mathcal{R} \cap \bigg(\bigcup_{\mu = 1}^{n}\bigcap_{\nu = 1, i_{\mu\nu} \in A_{\mu}}^{n-1}l_{i_{\mu\nu}1}^0\bigg),
\end{equation}
where $A_{\mu}$ for $\mu = 1, 2, \dots, n$ is an $n-1$-subset of $\{1, 2, \dots, n\}$, which is composed of the regions with each only adjacent to all of $l_{i1}$'s or any $n-1$ ones of $l_{i1}$'s (e.g., red ones of Figure \ref{Fig.9}a), with the cardinality
\begin{equation}
|\mathcal{B}| = (M-1)n + 1.
\end{equation}
Denote by $\mathfrak{N}$ a two-layer ReLU network whose units of the hidden layer are from the hyperplanes of $H$ together with at least $n+1$ universal global ones with respect to $U$. Then, if the linear functions of $\hat{f}(\boldsymbol{x})$ on $\mathcal{B}$ have been realized by $\mathfrak{N}$, the remaining ones on $\mathcal{R} - \mathcal{B}$ can be simultaneously implemented.
\end{lem}
\begin{proof}
\textbf{Case of two-dimensional input}. By the example of Figure \ref{Fig.9}a, this lemma tells us that to a piecewise linear function $\hat{f}(\boldsymbol{x})$ on $R_{ij}$'s for $i,j = 1, 2, 3$, if the linear functions of $\hat{f}(\boldsymbol{x})$ on the red regions (i.e., the ones of the leftmost column and bottom row) are realized by network $\mathfrak{N}$, whose units are from the lines partitioning the nine regions as well as some universal global ones with respect to $U = [0, 1]^2$, then the linear functions on the remaining regions can be simultaneously implemented.

The proof is by theorem 9. In Figure \ref{Fig.9}a, there are four lines, denoted by $l_{\nu\mu}$ for $\nu, \mu = 1, 2$; the vertical lines are $l_{1\mu}$'s, while the horizontal ones are $l_{2\mu}$'s; $l_{11} \cap l_{21} = O_1$, $l_{11} \cap l_{22} = O_2$, and $l_{12} \cap l_{21} = O_3$. Let $s_{ij}$ be the linear function of $\hat{f}(\boldsymbol{x})$ on $R_{ij}$; among them, $s_{i1}$'s and $s_{1j}$'s have been implemented by $\mathfrak{N}$. Due to $s_{12}$ and $s_{21}$, the function values on line segments $O_1O_2$ and $O_1O_3$ are determined, respectively, resulting in the production of $s_{22}$ according to theorem 9. Again $s_{22}$ and $s_{31}$ lead to $s_{32}$. To the third-row regions, because $s_{13}$ has been set beforehand and $s_{22}$ and $s_{32}$ have been previously constructed, it is the same case as the second row, for which $s_{23}$ and $s_{33}$ are also realized.

\textbf{Case of three-dimensional input}. The case of the three-dimensional input is shown in Figure \ref{Fig.9}b, in which each region is a cube. For simplicity, only part of the regions of $U$ are depicted. The linear functions on the red cubes have been realized by network $\mathfrak{N}$. We generalize the notations of the two-dimensional input as follows. Let $\boldsymbol{x} = [x_1, x_2, x_3]^T$  be a vector of $\mathbb{R}^3$, with each dimension shown in Figure \ref{Fig.9}c. Each region of Figure \ref{Fig.9}b is denoted by $R_{ijk}$ for $i, j, k = 1, 2, 3$, where the subscripts $i$, $j$ and $k$ correspond to dimensions $x_1$, $x_2$ and $x_3$, respectively; $s_{ijk}$ denotes the linear function of $\hat{f}(\boldsymbol{x})$ on $R_{ijk}$. If $k$  of dimension $x_3$ is fixed, the arrangement of the two-dimensional notations $R_{ijk}$'s for $i$ and $j$ is the same as that of Figure \ref{Fig.9}a. Then the red regions of Figure \ref{Fig.9}b are $R_{i11}$'s of direction $x_1$, $R_{1j1}$'s of direction $x_2$ and $R_{11k}$'s of direction $x_3$, with the corresponding linear functions $s_{i11}$'s, $s_{1j1}$'s and $s_{11k}$'s having been implemented by $\mathfrak{N}$, respectively.

To the regions $R_{ij1}$'s of the first floor of Figure \ref{Fig.9}b, the production of the linear functions on the black cubes is the same as that of Figure \ref{Fig.9}a, since $s_{i11}$'s and $s_{1j1}$'s have been provided. The difference from the two-dimensional input begins at the second floor of $R_{ij2}$'s, among which only the linear function $s_{112}$ on $R_{112}$ has been given. We can first use $s_{112}$ and the linear functions of the first floor to yield $s_{i12}$'s and $s_{1j2}$'s, such that the problem is reduced to the first-floor case. For example, to direction $x_1$, $s_{212}$ is determined by the function values on faces $ABC$ and $ABD$ from $s_{112}$ and $s_{211}$, respectively; $s_{122}$ of direction $x_2$ is from $BCF$ and $BEF$; the remaining ones of $s_{i12}$'s and $s_{1j2}$'s can be similarly obtained one by one. Thus, the second floor can be dealt with as the first one, and so is the the third one.

\textbf{Case of four-dimensional input}. Since the four-dimensional case cannot be visualized as above, we prove the three-dimensional result by another way for further generalizations. Write $\hat{f}(x_1, x_2, x_3) := \hat{f}(\boldsymbol{x})$. We first use planes $l_{ij}$ for $i,j = 1, 2$ of Figure \ref{Fig.9}c to partition $U = [0,1]^3$ into $3^2 = 9$ regions, regardless of dimension $x_3$. To any fixed number $x_3 = \alpha$, function $\hat{f}(x_1, x_2, \alpha)$ is two-dimensional and is the case of Figure \ref{Fig.9}a. Thus, although the first floor of $U$ composed of $R_{ij1}$'s is three-dimensional, it can be processed by the two-dimensional method. Note that the red cubes of Figure \ref{Fig.9}b can be obtained by shifting $R_{111}$ one by one. For instance, translating $R_{111}$ in direction $x_1$ by step $1/M$ when $M = 3$ gives $R_{211}$, and $R_{121}$ in direction $x_2$.

After the dividing of dimension $x_3$ through planes $l_{31}$ and $l_{32}$, the second floor made up of $R_{ij2}$'s appears. The linear function $s_{112}$ on $R_{112}$ is the only one of the second floor that has been set. We shift $R_{112}$ to form the red cubes of the second floor as the first one, during which the linear functions on them can be simultaneously produced. For example, translate $R_{112}$ in direction $x_1$ by step $1/M$ yields $R_{212}$, adjacent to both $R_{112}$ and $R_{211}$ because $R_{212}$ and $R_{112}$ as well as $R_{212}$ and $R_{211}$ are separated by $l_{11}$ and $l_{31}$, respectively. Let $p_1 = R_{212} \cap R_{112}$ and $p_2 = R_{212} \cap R_{211}$, both of which are two-dimensional and the linear functions on them have been known beforehand. Because $(p_1 \cap p_2) \subset l_{11} \cap l_{31} \ne \emptyset$, $s_{212}$ on $R_{212}$ is determined by the function values on $p_1$ and $p_2$.

We further move $R_{212}$ in direction $x_1$ to become $R_{312}$ and to generate $s_{312}$. Then all the red cubes of dimension $x_1$ are generated. After dimension $x_2$ being similarly processed, the second floor is reduced to the first one, resulting in the production of all the linear functions on its regions. The third-floor case is the same as the second one.

We now generalize the above procedure to the four-dimensional input. Write $\boldsymbol{x} = [x_1, x_2, x_3, x_4]^T$. First use three-dimensional hyperplanes $l_{ij}$ for $i= 1,2,3$ and $j = 1, 2$, whose equation is $x_i = j/3$, to divide $U = [0, 1]^4$ into $3^3 = 27$ regions, with each $i$ corresponding to dimension $x_i$. Then $l_{4j}$'s, each of which is $x_4 = j/3$, are introduced to separate dimension $x_4$ into three parts that form the three floors of $U$, with each having $27$ regions. The total number of the regions partitioned by $l_{\nu j}$'s for $\nu = 1, 2, 3, 4$ and $j = 1, 2$ is thus $3^4 = 81$; each of them is denoted by $R_{ijkt}$ for $i, j, k, t = 1, 2, 3$, corresponding to dimensions $x_1$, $x_2$, $x_3$ and $x_4$, respectively; and to a fixed $t$ of dimension $x_4$, the arrangement of notations $R_{ijkt}$'s is the same as the three-dimensional case. Note that each fixed $t$ of $R_{ijkt}$'s corresponds to one certain floor of $U$. Let $s_{ijkt}$ be the linear function of $\hat{f}(\boldsymbol{x})$ on $R_{ijkt}$.

Next, we first turn to a general description of the red cubes of Figure \ref{Fig.9}b, which had been used by this lemma. In each dimension $x_i$ for $i = 1, 2, 3$, there are three red cubes that grow only in direction $x_i$, with the length of the remaining two dimensions restricted to be less than $1/M$, in terms of each red cube being adjacent to $l_{\nu1}$'s for $\nu \ne i$; and a special case is $R_{111}$ that is adjacent to all $l_{\nu1}$'s; this is the description of the associated condition of this lemma. To the notations of the regions (linear functions) of the red cubes, for example, when $j = 1$ and $k = 1$, they are $R_{i11}$'s ($s_{i11}$'s), with only $i$ changing for the growth in direction $x_1$. The total number of the red cubes of the three dimensions is $(M-1) \times n + 1 = 7$ when $M = 3$ and $n = 3$, which is of equation 7.6, because $R_{111}$ is shared by all the dimensions.

Then by the assumption of this lemma, the linear functions $s_{i111}$'s, $s_{1j11}$'s, $s_{11k1}$'s and $s_{111t}$'s have been set beforehand. The first floor of $U$ can be dealt with by the three-dimensional method. In the second floor, only $s_{1112}$ on $R_{1112}$ is known. In dimension $x_1$, shift $R_{1112}$ by step $1/M$ to obtain $R_{2112}$, which is adjacent to $R_{1112}$ and $R_{2111}$ since $R_{2112}$ and $R_{1112}$ as well as $R_{2112}$ and $R_{2111}$ are divided by $l_{11}$ and $l_{41}$, respectively. So $c_1 = R_{2112} \cap R_{1112}$ and $c_2 = R_{2112} \cap R_{2111}$ exist, both of which are three-dimensional (lemma 2). Because $(c_1 \cap c_2) \subset l_{11} \cap l_{41} \ne \emptyset$, $s_{2112}$ is obtained. The remaining $s_{i112}$'s for $i \ne 2$ of dimension $x_1$ of the second floor can be recursively generated.

In dimensions $x_2$ and $x_3$, similar operations can be done and the difference from $x_1$ lies in the direction of shifting $R_{1112}$ as well as the hyperplanes partitioning the adjacent regions. This contributes to the ''red four-dimensional regions'' of the second floor, such that the problem is reduced to the first floor and all of its linear functions could be produced by the three-dimensional method. The third floor can be analogously dealt with.

\textbf{Inductive steps}. The above procedure from three- to four-dimensional inputs can be generalized to that from $k$- to $k+1$-dimensional inputs for $k \ge 4$. Thus, the induction step can be done recursively until $k = n-1$, when the $n$-dimensional case is proved. This completes the proof.
\end{proof}

\begin{rmk}
This conclusion seems to be unreasonable, since so limited portion of function values are required to yield the remaining ones and the distance between the regions whose linear functions are provided and those whose linear functions are to be determined could be far within $U$. However, this illusion can be clarified by the fact that each region of this lemma has its distinct hyperplanes to shape its linear function, in terms of the combinatorial effect of different hyperplanes, such that the diversity and distinction of the linear functions can be ensured.
\end{rmk}

\begin{lem}
Notations being from lemma 9, any continuous piecewise linear function $\hat{f}(\boldsymbol{x}) \in \mathfrak{K}_n(H)$ for $\boldsymbol{x} \in [0, 1]^n$ with $\zeta = M^n$
linear pieces can be realized by a two-layer ReLU network $\mathfrak{N}$ whose hidden layer has
\begin{equation}
\Theta \ge Mn + 1
\end{equation}
units.
\end{lem}
\begin{proof}
By the three-dimensional example of Figure \ref{Fig.9}b, according to lemma 9, it suffices to implement the linear functions on the red-cube regions of $\hat{f}(\boldsymbol{x})$ by $\mathfrak{N}$. As shown in Figure \ref{Fig.9}b, we first divide $U$ for dimension $x_1$ by the planes $l_{11}$ and $l_{12}$ to be 3 regions, denoted by $R_i$'s for $i = 1, 2, 3$, respectively, in the order along the direction of $x_1$. The equations of $l_{11}$ and $l_{12}$ have been suitably set such that a strict partial order $\mathscr{P}_1 = l_{11} \prec l_{12}$ is formed, with $R_{1}$ as its initial region. The initial linear function on $R_1$ of $\mathscr{P}_1$ is set to be $s_{111}$ of $\hat{f}(\boldsymbol{x})$ by at least $n+1=4$ universal global units when $n = 3$. Then $s_{211}$ on $R_2$ and $s_{311}$ on $R_3$ can be successively realized based on $\mathscr{P}_1$.

To dimension $x_2$, further partition $U$ by $l_{21}$ and $l_{22}$ into $9$ regions $R_{ij}$'s for $i,j = 1, 2, 3$, with the notations arranged analogous to Figure \ref{Fig.9}a. We also make $l_{21}$ and $l_{22}$ form a new strict partial order $\mathscr{P}_2 = l_{21} \prec l_{22}$ whose initial region is $R_{11}$ and the associated ordered regions are $R_{12}$ and $R_{13}$. The initial linear function on $R_{11}$ of $\mathscr{P}_2$ has been set beforehand, and thus $s_{121}$ and $s_{131}$ of $\hat{f}(\boldsymbol{x})$ on $R_{12}$ and $R_{13}$, respectively, can be implemented. The previously constructed linear functions on $R_{\nu1}$'s for $\nu = 2, 3$ are not influenced by $\mathscr{P}_2$, because $R_{\nu1}$'s are all in $l_{21}^0$ and $l_{22}^0$, satisfying the condition \Romannum{2} of theorem 7.

The last step is for dimension $x_3$. The introduction of $l_{31}$ and $l_{32}$ finally contributes to regions $R_{ijk}$'s for $i,j,k = 1, 2, 3$ mentioned in the proof of lemma 9. The third established strict partial order $\mathscr{P}_3 = l_{31} \prec l_{32}$ would result in $s_{112}$ and $s_{113}$ on $R_{112}$ and $R_{113}$, respectively, without disturbing the linear functions constructed before. Now, the condition of lemma 9 is satisfied and the conclusion of this lemma follows. The general $n$-dimensional case is similar.

The units of the hidden layer of $\mathfrak{N}$ include those of $H$ of equation 7.4 and at least $n+1$ universal global ones, and hence the lower bound of their number is the sum of $(M-1)n$ and $n+1$, which is inequality 7.7.
\end{proof}

\begin{thm}[Universal-approximation capability]
Any $C^1$-function $f(\boldsymbol{x}) : U \to \mathbb{R}$, where $U = [0,1]^n$ for $n \ge 2$, can be approximated by a two-layer ReLU network $\mathfrak{N}$ with arbitrary precision, in terms of a continuous piecewise linear function $\hat{f}(\boldsymbol{x}) \in \mathfrak{K}_n(H)$ of lemma 9, where the parameters of $\mathfrak{N}$ are set by the method of lemma 10. To achieve an approximation $\epsilon$, if $\zeta$ linear pieces of $\hat{f}(\boldsymbol{x})$ are required, then the number of the units of the hidden layer of $\mathfrak{N}$ satisfies
\begin{equation}
\Theta \ge {\zeta}^{1/n}n + 1.
\end{equation}
\end{thm}
\begin{proof}
First, divide $U$ into $\zeta = M^n$ regions by a set $H$ of $n-1$-dimensional hyperplanes via lemma 9. Second, construct a continuous piecewise linear function $\hat{f}(\boldsymbol{x}) \in \mathfrak{K}_n(H)$ approximating $f(\boldsymbol{x})$ by lemma 6. Third, use network $\mathfrak{N}$ to realize $\hat{f}(\boldsymbol{x})$ through lemmas 10 and 9. When $M$ is large enough such that the volume of each region is sufficiently small, $\hat{f}(\boldsymbol{x})$ could approximate $f(\boldsymbol{x})$ as precisely as possible. Inequality 7.8 follows from $\zeta = M^n$ and inequality 7.7.
\end{proof}

\subsection{Related Universal-Approximation Work}
In the area of two-layer neural networks, the generalization from univariate-function approximation to the multivariate case usually resorts to what is called ``ridge function'' \citep*{Pinkus2015,Ismailov2021}, such as multi-dimensional Fourier transform \citep*{Gallant1988, Chen1992} and Radon transform \citep*{Carroll1989}. Our method is different in the main principle (including strict partial orders and continuity restriction) and, in particular, can explain the training solution obtained by the back-propagation algorithm.

\subsection{Boundary-Determination Problem}
The results of section 7.2 are based on specially designed hyperplanes and not general enough to explain other solutions. This section extends them to be a typical solution mode as well as a new mathematical phenomenon called ``boundary-determination problem'' in analogy to the boundary problem of some other branches of mathematics (e.g., partial differential equations).

\begin{dfn}[Standard partition]
In equation 7.2 of lemma 9, modify the number $M$ for each dimension $x_i$ to be $M_i$ with $M_i \ge 2$, such that $M_i$'s are not necessarily equal to each other. Then the cardinality of $\mathcal{R}$ of equation 7.3 becomes
\begin{equation}
|\mathcal{R}| = \prod_{i=1}^nM_i
\end{equation}
and the set
\begin{equation}
\mathcal{B} =  \mathcal{R} \cap \bigg(\bigcup_{\mu = 1}^{n}\bigcap_{\nu = 1, i_{\mu\nu} \in A_{\mu}}^{n-1}l_{i_{\mu\nu}1}^0\bigg)
\end{equation}
corresponding to equation 7.5 also exists. Write $\mathcal{R} = \{R_k: k = 1, 2, \dots, |R|\}$ . Let $N_k \subset \mathcal{R}$ be the set of the regions adjacent to each $R_k \in \mathcal{R}$, which describes the neighborhood information of $R_k$. We call the dividing of $l_{i,n_i}$'s for all $i$ and $n_i$ a standard partition of $U$, denoted by
\begin{equation}
\mathfrak{P} := \mathfrak{P}\big(H, \mathcal{R}, N\big),
\end{equation}
where $H = \{l_{i,n_i}: 1 \le i \le n, 1 \le n_i \le M_i-1\}$ and $N = \{N_k: k = 1, 2, \dots, |R|\}$.
\end{dfn}

\begin{dfn}[Isomorphic Partitions]
Let $h_{\nu}$'s for $\nu = 1, 2, \dots, \xi = \sum_{i = 1}^{n}(M_i-1)$ be $n-1$-dimensional hyperplanes dividing $U = [0, 1]^n$ into regions $Q_{\mu}$'s for $\mu = 1, 2, \dots, \zeta = \prod_{i=1}^nM_i$, where $M_i$ is from equation 7.9 and $n \ge 2$. Denote by $N_{\mu}$ the set composed of the regions adjacent to $Q_{\mu}$. Write
\begin{equation}
\mathcal{P} := \mathcal{P}\big(H', \mathcal{R}', N'\big),
\end{equation}
which is the partition of $U$ via $h_{\nu}$'s, where $H' = \{h_{\nu}: \nu = 1, 2, \dots, \xi\}$, $\mathcal{R}' = \{Q_{\mu}: \mu = 1, 2, \dots, \zeta\}$ and $N' = \{N_{\mu}: \mu = 1, 2, \dots, \zeta\}$.

Suppose that either of the following two conditions is satisfied:
\begin{itemize}
\item[\rm{1}.] $H'$ is an affine transformation $\mathscr{A}$ of $H$ of equation 7.11, with the constraint that the intersections of the elements of $H'$ are contained in $U$.
\item[\rm{2}.] $H'$ is a projective transformation $\mathscr{P}$ of $H$ fulfilling the two constraints: (a) One is that if two elements $l_{ij}$ and $l_{i'j'}$ of $H$ are parallel (i.e., $l_{ij} \cap l_{i'j'} = \emptyset$), and if their counterparts $h_{\nu} = \mathscr{P}(l_{ij})$ and $h_{\nu'} = \mathscr{P}(l_{i'j'})$ of $H'$ are unparallel, then their intersection $h = h_{\nu} \cap h_{\nu'}$ is not in $U$; (b) The other is that if $l_{ij} \cap l_{i'j'} \ne \emptyset$, then $h \ne \emptyset$ and $h \subset U$.
\end{itemize}
We say that the partition $\mathcal{P}$ of equation 7.12 is isomorphic to a standard one $\mathfrak{P}$ of equation 7.11 or write
\begin{equation}
\mathcal{P} \cong \mathfrak{P}.
\end{equation}
To emphasize the transformation between $\mathcal{P}$ and $\mathfrak{P}$, write
\begin{equation}
\mathcal{P} = \mathscr{T}(\mathfrak{P}),
\end{equation}
where $\mathscr{T} =  \mathscr{A}$ or $ \mathscr{P}$ above.
\end{dfn}

\begin{lem}
Under the isomorphic partitions of definition 22, two relationships of the standard partition $\mathfrak{P}$ of equation 7.11 can be transmitted to partition $\mathcal{P}$ of equation 7.12. The first is $R_{k} \in l_{ij}^+$ (or $R_{k} \in l_{ij}^0$) for arbitrary $k$, $i$, and $j$ between the regions of $\mathcal{R}$ and the hyperplanes of $H$, whose counterpart of $\mathcal{P}$ is $Q_{\mu} \in h_{\nu}^+$ (or $Q_{\mu} \in h_{\nu}^0$), where $Q_{\mu} = \mathscr{T}(R_k)$, $h_{\nu} = \mathscr{T}(l_{ij})$ and $\mathscr{T}$ is from equation 7.14. The second is that if $R_{k_1}$ is adjacent to $R_{k_2}$ for $k_1 \ne k_2$ and if they are divided by $l_{ij}$ of $H$, then $Q_{\mu_1} = \mathscr{T}(R_{k_1})$ is also adjacent to $Q_{\mu_2} = \mathscr{T}(R_{k_2})$ and they are separated by $h_{\nu} = \mathscr{T}(l_{ij})$ of $H'$. And on the basis of the above two relationships as well as the conditions of definition 22, the regions of $U$ formed by $H'$ correspond to those by $H$ through a bijective map.
\end{lem}
\begin{proof}
We call the first and second relationships of this lemma ``region-hyperplane'' and ``adjacent-region'' relationship, respectively. If $\mathscr{T}$ is an affine transformation, owning to the invariance of the parallel and unparallel properties of hyperplanes under $\mathscr{T}$, the two kinds of relationships are certainly preserved. Thus, to ensure the bijective map between the regions of the two partitions, it's sufficient to restrict the dividing of $H' = \mathscr{T}(H)$ to be in $U$, which is the constraint of condition 1 of definition 22.

When $\mathscr{T}$ is a projective transformation, both the parallel and unparallel properties may be changed under $\mathscr{T}$. To the case of parallel hyperplanes becoming unparallel ones, we make the intersection of the mapped hyperplanes out of $U$ by condition 2(a). To $l_{ij} \cap l_{i'j'} \ne \emptyset$ of $H$, $h \ne \emptyset$ of condition 2(b) makes sure that the intersecting properties between the hyperplanes of $H$ remain invariant under $\mathscr{T}$, while $h \in U$ ensures the partition of $H'$ done in $U$. Therefore, the two constraints of condition 2 render the variation of the parallel or unparallel property small enough to preserver the effect of the previous affine-transformation case, and thus this lemma holds for some projective transformations as well.
\end{proof}

\begin{dfn}[Boundary of a domain]
To a partition $\mathcal{P}\big(H', \mathcal{R}', N'\big)$ of equation 7.12 isomorphic to a standard one $\mathfrak{P}$ of equation 7.11, the set
\begin{equation}
\mathcal{B}' = \mathscr{T}(\mathcal{B})
\end{equation}
is called the \textsl{boundary} of $\mathcal{R}'$, where $\mathscr{T}$ and $\mathcal{B}$ are from equations 7.14 and 7.10, respectively.
\end{dfn}

\begin{thm}[Boundary-determination principle]
Notations being as in definition 23, let $\hat{f}(\boldsymbol{x}) \in \mathfrak{K}_n(H')$ of equation 6.1 be a continuous piecewise linear function and $\mathfrak{N}$ be a two-layer ReLU network whose units of the hidden layer are from the hyperplanes of $H'$ along with at least $n+1$ universal global hyperplanes. If the linear functions of $\hat{f}(\boldsymbol{x})$ on the boundary $\mathcal{B}'$ of equation 7.15 are implemented by $\mathfrak{N}$, then the ones on the remaining regions can also be simultaneously realized.
\end{thm}
\begin{proof}
By lemma 11, the ingredients of the proof of lemma 9, including the adjacent relationships between the regions as well as the influence of each hyperplane, can all be mapped to $\mathcal{P}$, which follows the conclusion.
\end{proof}

\begin{rmk}
It is possible that there's more than one partition combined, each isomorphic to a standard one, such that the formed function could be more complicated than that of this theorem. Section 8.2 will give some examples.
\end{rmk}

\begin{dfn}[Boundary-determination problem]
We call the phenomenon of theorem 11 that the function values on $\mathcal{R}'$ are determined by those on the boundary $\mathcal{B}'$ \textsl{a boundary-determination problem} of two-layer ReLU networks.
\end{dfn}

\subsection{Related Boundary-Determination Problems}
The well-known Cauchy's integral formula
\begin{equation}
f(z) = \oint_{C}\frac{f(\zeta)}{\zeta - z}dz
\end{equation}
of complex analysis is similar to the boundary-determination problem of section 7.3, which says that the function value $f(z)$ for $z$ that is a interior point of a simple closed contour $C$ is determined by the boundary values of $f(z)$ on $C$, in terms of the integration along $C$ of equation 7.16. A proof of equation 7.16 given in \citet*{Ablowitz2003} suggests that its mechanism includes Cauchy's theorem and the property of complex integration.

In the discipline of differential equations, the boundary-value problem is also related to the function values of both the boundary part and the inner part of a domain \citep*{Zwillinger2022}. Under some condition, the solution of a differential equation can be uniquely determined by the function given at the boundary \citep*{Courant1962}. The initial-value problem \citep*{Zwillinger2022} is also of this type, despite the distinct terminology due to different physical backgrounds.

Besides function values, there exist some other property merely depending on the boundary of a region. For instance, Green's theorem \citep*{Ablowitz2003}
\begin{equation*}
\iint_{R}(\frac{\partial v}{\partial x} - \frac{\partial u}{\partial y})dxdy = \oint_{C}udx + vdy,
\end{equation*}
tells us that that a double integral over a connect region $R$ equal to a line integral around the simple closed curve $C$ that is the boundary of $R$. And the more celebrated fundamental theorem of calculus
\begin{equation*}
\int_{a}^{b}f(x)dx = F(b) - F(a)
\end{equation*}
expresses the similar meaning, with the boundary of an interval being two points.

Another example is more like the boundary-determination problem of this paper in connection with the continuity property of a function, which is called ``analytic continuity'' of a function with a complex variable. A conclusion \citep*{Ablowitz2003} stemming from analytic continuity is that if two functions $f(z)$ and $g(z)$, both of which are analytic at a common domain $D$, coincide in some subdomain $D' \subset D$ or on a curve $\Gamma$ interior to $D$, then $f(z) = g(z)$ on the whole $D$. That is, part of the equal values of $f(z)$ and $g(z)$ leads to their equality on the whole domain.

Each of the above examples has its own mechanism to establish the relationship between the boundary and inner parts of a region or domain. And the case of this paper is based on the property of the functions produced by a two-layer ReLU network; as far as we know, it is a new mathematical phenomena that has not been reported before.

\section{Explanation of Experiments}
This section uses the developed theory to explain the training solution derived from the back-propagation algorithm. Section 8.1 validates corollary 3 for the one-dimensional input. Section 8.2 verifies the higher-dimensional conclusions of sections 6 and 7 in terms of the two-dimensional input.

\subsection{One-Dimensional Input}

\begin{figure}[!t]
\captionsetup{justification=centering}
\centering
\subfloat[Example 1.]{\includegraphics[width=2.7in, trim = {1.3cm 0.7cm 1.3cm 0.7cm}, clip]{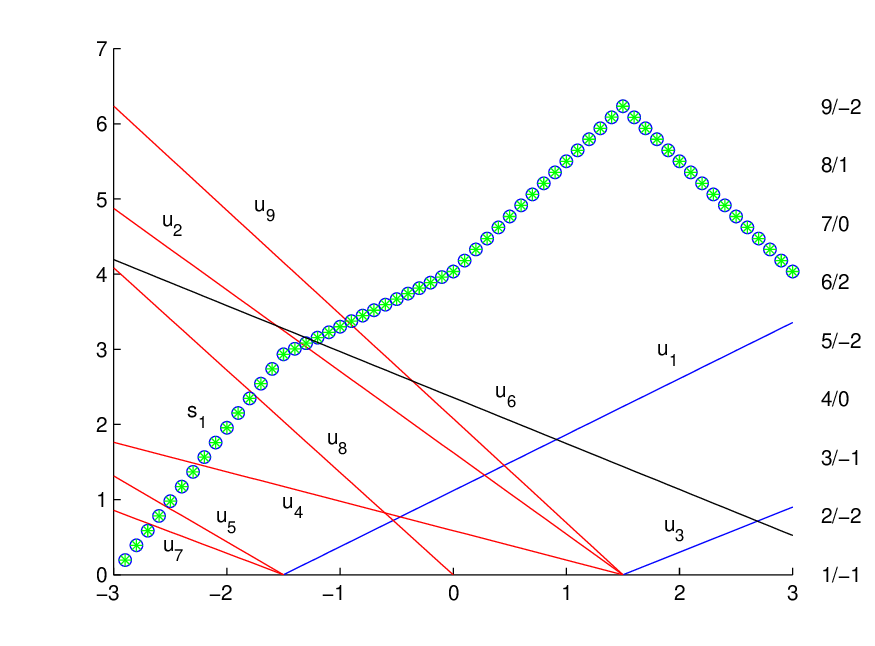}} \\
\subfloat[Example 2.]{\includegraphics[width=2.7in, trim = {1.3cm 0.7cm 1.3cm 0.65cm}, clip]{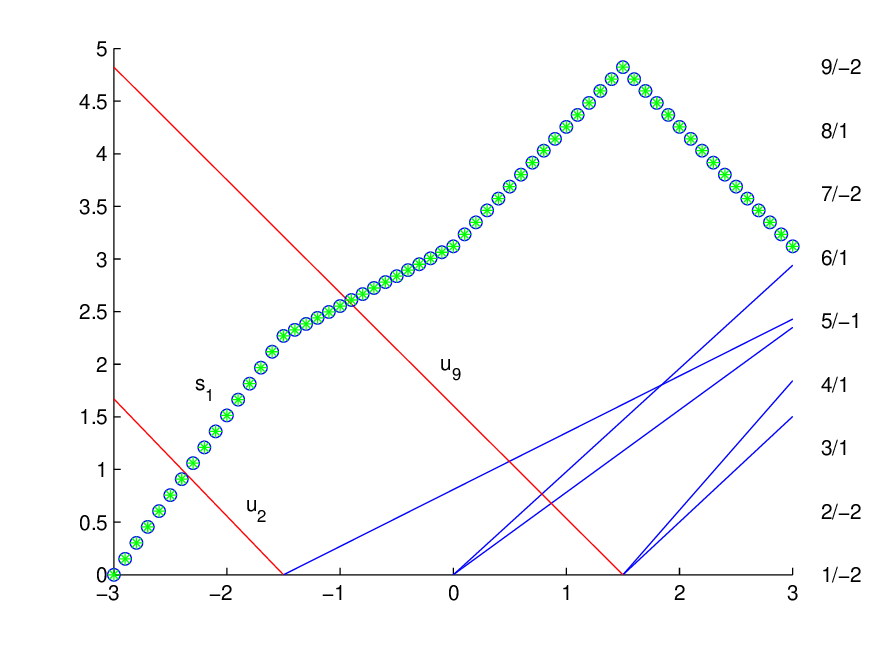}} \quad \quad
\subfloat[Example 3.]{\includegraphics[width=2.7in, trim = {1.3cm 0.7cm 1.3cm 0.7cm}, clip]{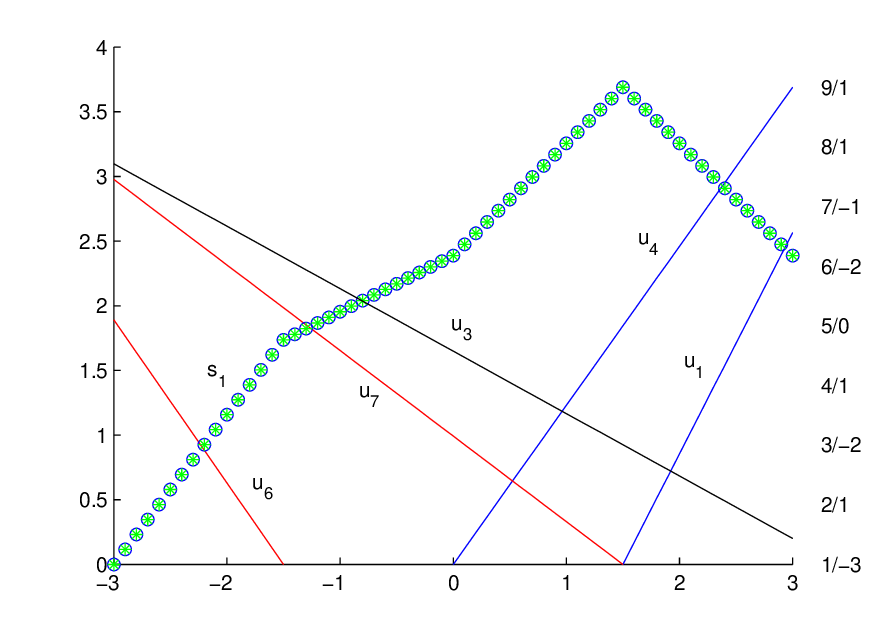}}
\caption{Training solutions of one-dimensional input.}
\label{Fig.10}
\end{figure}

Figure \ref{Fig.10} shows three examples of the two-sided bases obtained by the training way. The learning algorithm is the original back-propagation algorithm without any refinement or modification (e.g., stochastic gradient descent). The weights and bias are initialized by the uniform distribution $U(-1, 1)$. The learning rate and steps are set to be 0.002 and 10000, respectively.

In each of the three figures, the blue circles come from a continuous piecewise linear function with four linear pieces. The green asterisks are produced by a two-layer ReLU network with 9 units in the hidden layer. The red, blue and black lines are the outputs of the negative, positive and global units, respectively; some of the units are not shown in Figures \ref{Fig.10}b and \ref{Fig.10}c because they are not activated by the function domain. For convenience of visualization, the function values are normalized according to the range of the outputs of the ReLUs. The subscripts of the units are of the trained results. The examples of Figure \ref{Fig.10} validate the two-sided bases of section 5 as follows.

\textbf{Compound two-sided bases}. In Figure \ref{Fig.10}a, there are two bidirectional knots $x= -1.5$ and $x= 1.5$; the knot $x = 0$ has only one unit $u_8$ that is negative. So it is the compound type. The first linear function $s_1$ is implemented by one global units $u_6$ and five other local units $u_9$, $u_2$, $u_4$, $u_5$ ,$u_7$ and $u_8$, among which there are two groups of equivalent units including $\{u_9, u_2, u_4\}$ and $\{u_5, u_7\}$; $u_8$ is redundant, whose output weight is for the shaping of $s_3$. Those units except for $u_8$ are sufficient to produce $s_1$, since their number is greater than 2 (each equivalent group amounts to one), and the redundant ones cannot influence the realization of $s_1$. Figure \ref{Fig.10}c is similar.

\textbf{Added two-sided bases}. The case of the added two-sided bases of theorem 4 is verified by Figure \ref{Fig.10}b, in which $u_2$ and $u_9$ generate $s_1$ and both of them are not global units.

\textbf{Number of units required}. The minimum number $\Theta = \zeta + 1$ of the units required in theorem 4 is demonstrated by Figures \ref{Fig.10}c. Although 9 units are provided for the piecewise linear function with $\zeta = 4$ linear segments, only five of them are effective or activated and hence the solution can be regarded as using 5 units, which is the case of the minimum $\Theta = \zeta + 1 = 5$.

\subsection{Two-Dimensional Input}

\begin{figure}[htp]
\captionsetup{justification=centering}
\centering
\subfloat[Example 1: data-fitting effect.]{\includegraphics[width=2.6in, trim = {1.2cm 0.7cm 0.9cm 0.1cm}, clip]{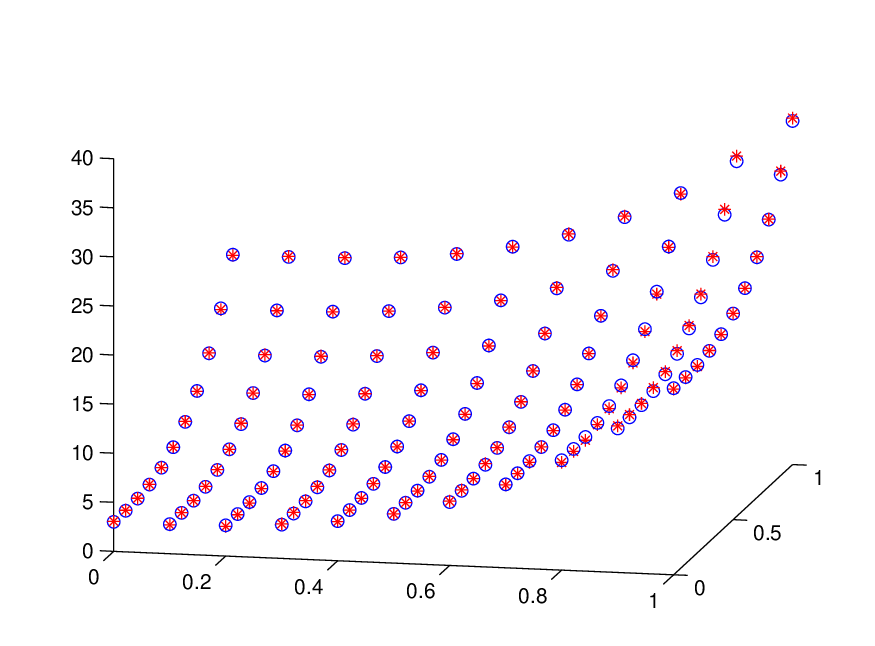}} \quad \ \
\subfloat[Example 1: solution of (a).]{\includegraphics[width=2.9in, trim = {1.2cm 0.7cm 1.0cm 0.1cm}, clip]{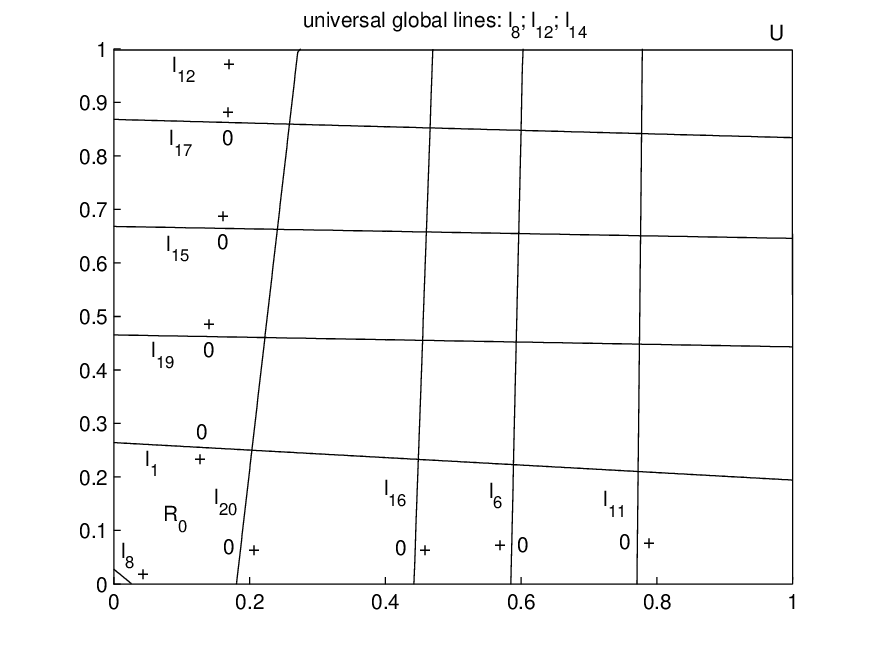}} \\
\subfloat[Example 2: data-fitting effect.]{\includegraphics[width=2.6in, trim = {1.1cm 0.7cm 0.9cm 0.1cm}, clip]{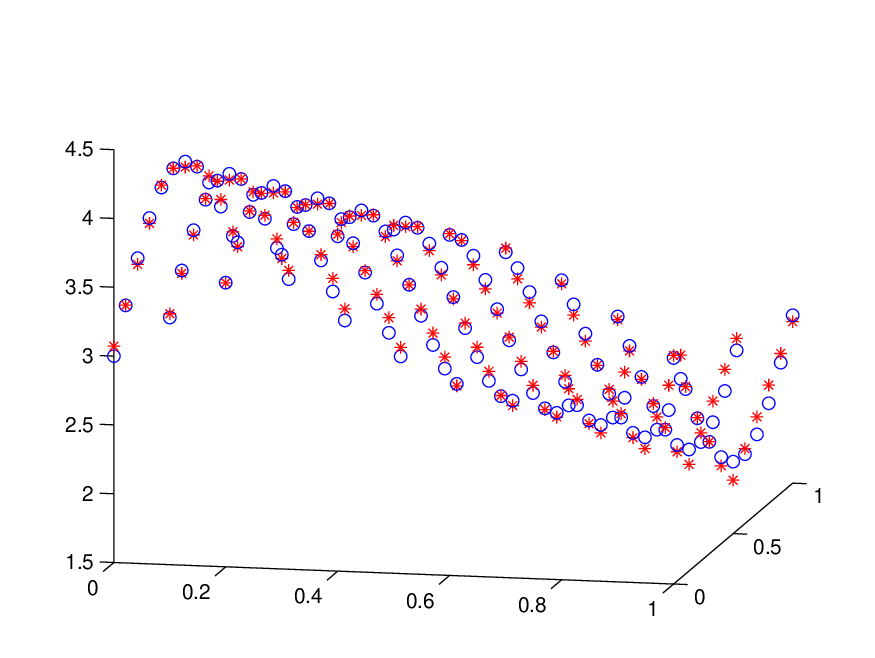}} \quad \ \
\subfloat[Example 2: solution of (c).]{\includegraphics[width=2.9in, trim = {1.2cm 0.7cm 1.0cm 0.1cm}, clip]{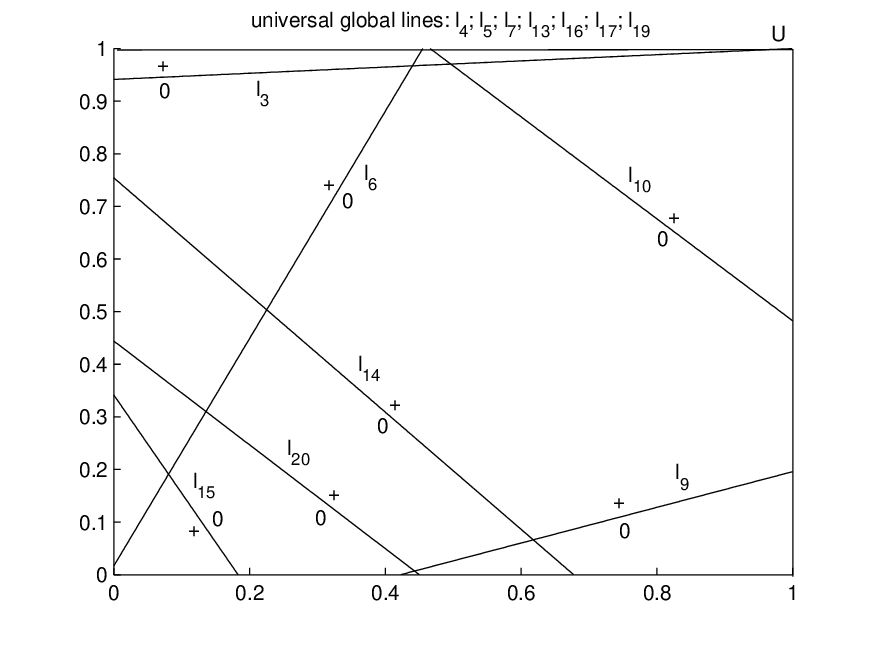}} \\
\subfloat[Example 3: data-fitting effect.]{\includegraphics[width=2.6in, trim = {1.1cm 0.7cm 0.9cm 0.1cm}, clip]{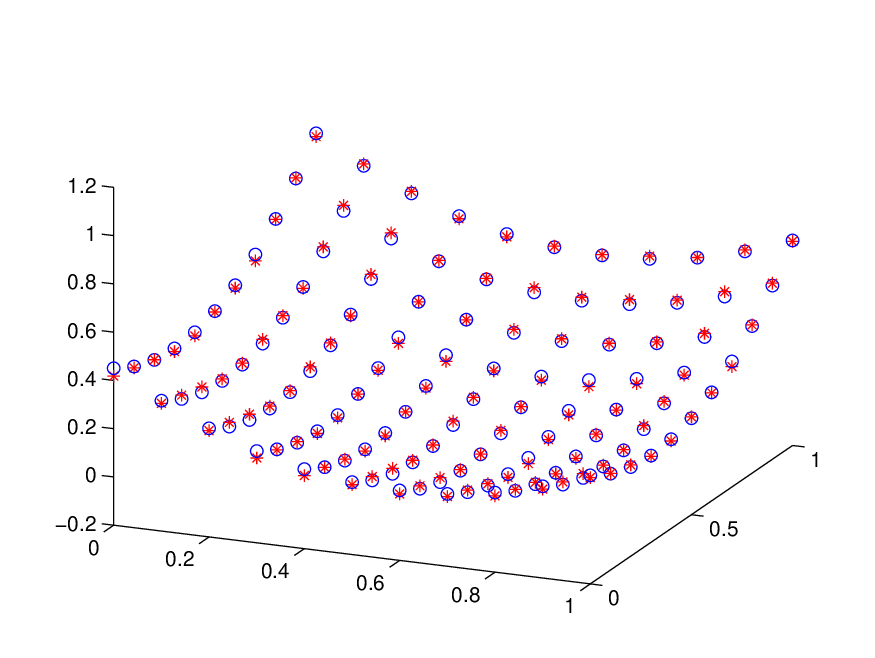}} \quad \ \
\subfloat[Example 3: solution of (e).]{\includegraphics[width=2.9in, trim = {1.2cm 0.7cm 1.0cm 0.1cm}, clip]{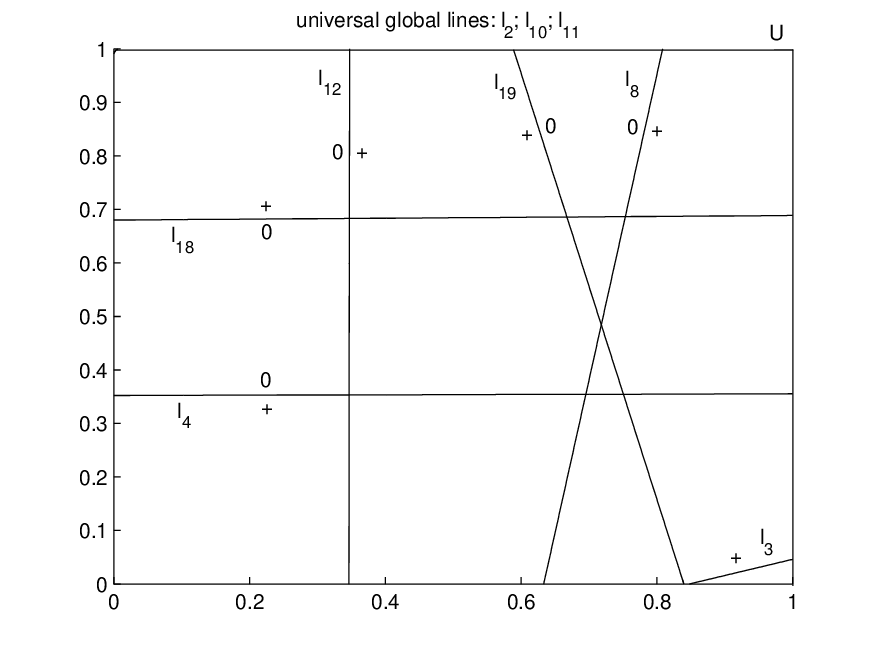}}
\caption{Training solutions of two-dimensional input.}
\label{Fig.11}
\end{figure}

We use the results of sections 6 and 7 to explain the training solutions of Figure \ref{Fig.11} for the two-dimensional input. The blue-circle points of Figure \ref{Fig.11}a are derived from the discretization of polynomial function $z = 16(x^3 + y^3)+3$ on $U = [0, 1]^2$, with step length 0.1 on both the dimensions of $x$ and $y$. The red asterisks are produced by the function $\hat{z}(x, y)$ of a two-layer ReLU network $\mathfrak{N}$ with 20 units in the hidden layer, initialized by uniform distribution $U(-1, 1)$ and trained by the back-propagation algorithm. The number of the training steps is $4000$ and the learning rate is $0.01$. The relative fitting error measured by
\begin{equation}
\epsilon = \frac{1}{z_M - z_m} \sqrt{\frac{\sum_{i}(z_i - \hat{z}_i)^2}{N}}
\end{equation}
is of order $10^{-3}$, where $z_i$ and $\hat{z}_i$ are the function values of $z(x,y)$ and $\hat{z}(x,y)$, respectively, $z_M = \max\{z_i\}$, $z_m = \min\{z_i\}$, and $N$ is the total number of $i$'s.


Figure \ref{Fig.11}b is the region dividing of $U = [0, 1]^2$ through the lines of the units of the hidden layer of $\mathfrak{N}$. The set of the universal global lines are given at the top of the figure. Since $l_8^0 \cap U$ is much smaller than $U$, we consider $l_8$ as a universal global line. The subscript as well as the ``0'' or ``+'' label of each line are of the training result. The inactivated units with respect to $U$ are not included in the figure.

Figures \ref{Fig.11}c and \ref{Fig.11}e differ from Figure \ref{Fig.11}a only in the continuous functions to be approximated, which are $z' = \sin{3(x + y + 1)} + 3$ and $z'' = (x - 0.6)^2 + (y - 0.3)^2$, respectively. Both the approximation errors to $z'$ and $z''$ computed by equation 8.1 are of order $10^{-2}$. Figure \ref{Fig.11}d is the solution for Figure \ref{Fig.11}c, while Figure \ref{Fig.11}f for Figure \ref{Fig.11}e.

The following items explain the solutions of Figure \ref{Fig.11} by the results of section 6: (a) \textbf{Multiple strict partial orders}. Regardless of the two-sides bases to be discussed later, in Figure \ref{Fig.11}a, we have two strict partial orders including
\begin{equation}
    \mathscr{P}_1 = l_1\prec l_{19} \prec l_{15} \prec l_{17}
\end{equation}
and
\begin{equation}
    \mathscr{P}_2 = l_{20}\prec l_{16} \prec l_{6} \prec l_{11},
\end{equation}
and $R_{0}$ is their common initial region. (b) \textbf{Solution pattern of universal global units}. The initial linear function on $R_0$ of Figure \ref{Fig.11}a is implemented by three universal global lines (units) of $l_8$, $l_{12}$ and $l_{14}$, with the number equal to $n+1=3$ for $n=2$, which is the lower bound of inequality 6.7. Figures \ref{Fig.11}d and \ref{Fig.11}f have 7 and 3 universal global lines, respectively, both of which satisfy the condition of inequality 6.7. (c) \textbf{Two-sided bases}. In equations 8.2 and 8.3, if we change $l_1$ and $l_6$ into their negative forms, this is the case of Figure \ref{Fig.11}b, satisfying the condition of theorem 8 for the two-sided bases such that the function realization for $\mathscr{P}_1$ and $\mathscr{P}_2$ would not be affected. (d) \textbf{Order of multiple strict partial orders}. In Figure \ref{Fig.11}d, without considering the two-sided bases, three local strict partial orders $\mathscr{P}_1 = l_{15} \prec l_{20} \prec l_{14} \prec l_{10}$, $\mathscr{P}_2 = l_{6} \prec l_{3}$ and $\mathscr{P}_3 = l_{9} \prec l_{\infty}$ can be found and their order $\mathscr{P}_1, \mathscr{P}_2, \mathscr{P}_3$ satisfies equation 6.10 ($\mathfrak{R}_0 = l_6^0 \cap l_{15}^0 \cap U$).

From the viewpoint of continuity restriction of section 7, we can interpret the solutions of Figure \ref{Fig.11} as follows. (a) \textbf{Standard partition of $U$}. The partition of Figure \ref{Fig.11}b is obviously the same as the standard one as in Figure \ref{Fig.9}a in terms of an one-to-one correspondence between their regions as well as the lines forming them; for simplicity of descriptions, we also call it a standard partition, and similarly for the cases of the rest of this section. This partition together with the excellent data-fitting result of Figure \ref{Fig.11}a verify the conclusion related to lemma 9 and theorem 11. (b) \textbf{Multiple standard partitions combined}. In Figure \ref{Fig.11}d, we can find four standard sub-partitions intermingled, including $\mathfrak{P}_1 = l_{20}^0 \cap U$, $\mathfrak{P}_2 = l_{20}^+ \cap l_9^+ \cap l_{10}^0\cap l_3^0 \cap U$, $\mathfrak{P}_3 = l_{14}^+ \cap l_9^+  \cap U$ and $\mathfrak{P}_4 = l_{20}^+ \cap l_6^0 \cap l_{10}^0 \cap U$, which can also yield all the linear functions on $U$ through lemma 9, theorem 7 and theorem 8. (c) \textbf{Scattered continuity restrictions combined}. Two standard sub-partitions $\mathfrak{P}_1 = l_{19}^+ \cap l_8^0 \cap U$ and $\mathfrak{P}_2 = l_{19}^0 \cap l_4^0 \cap U$ can be found in Figure \ref{Fig.11}f and the linear functions on them can be realized as in Figure \ref{Fig.11}d. Although the small triangle $l_4^0 \cap l_8^+ \cap l_{19}^+$ is in neither of the two standard partitions, its left and right adjacent regions are in $\mathfrak{P}_1$ and $\mathfrak{P}_2$, respectively, such that its linear function can be determined by the continuity-restriction principle of theorem 9. The linear function on $l_4^+ \cap l_{8}^+ \cap l_{19}^+ \cap U$ immediately follows due to its left adjacent region and the previous $l_4^0 \cap l_8^+ \cap l_{19}^+$, and then the one on $l_{19}^0 \cap l_{3}^+ \cap l_{4}^+ \cap U$. Finally, the last linear function on $l_3^0 \cap U$ can be produced by $l_3$.

\section{Highlights for Black Box}
In the preceding sections, the mechanism of the solutions of two-layer ReLU networks (i.e., the black box) is embedded in the proof of the results throughout this paper, in terms of rigorous mathematical language. We here summarize them briefly to highlight the key ideas.
\begin{itemize}
\item[\rm{\Romannum{1}}.] \textbf{Unit classification}. The units of the hidden layer of a two-layer ReLU network $\mathfrak{N}$ can be classified into two categories. One is the type of (universal) global units implementing the initial linear function, based on which each of other linear functions can be simply determined by only one parameter; their number is usually greater than or equal to $n+1$, with $n$ as the dimensionality of the input space. The other is of local units to produce the knots and each knot needs at least one local unit.
\item[\rm{\Romannum{2}}.] \textbf{Number of units required}. To realize an arbitrary piecewise linear function with $\zeta$ linear pieces, the number of the units of the hidden layer of $\mathfrak{N}$ could be at least $\Theta = {\zeta}^{1/n}n + 1$, much smaller than expected due to the principle of continuity restriction making different regions share common parameters with each other.
\item[\rm{\Romannum{3}}.] \textbf{Meaning of the parameters}. Let $\sigma(\boldsymbol{w}_i^T\boldsymbol{x} + b_i)$ be the activation function of a unit $\mathcal{U}$ and $\lambda_i$ be its output weight. If $\mathcal{U}$ is a global unit, its input parameters $\boldsymbol{w}_i$ and $b_i$ are used to form a global hyperplane and $\lambda_i$ is as a variable adjusted to produce the initial linear function. If $\mathcal{U}$ is a local unit, $\boldsymbol{w}_i$ and $b_i$ yield the knot, while $\lambda_i$ generates the linear function on the associated region.
\item[\rm{\Romannum{4}}.] \textbf{Geometric meaning of output weights}. Notations from item 3, when $\mathcal{U}$ is a local unit, its output weight $\lambda_i$ contains the information of two angles derived from the associated two adjacent linear pieces separated by $\mathcal{U}$, respectively. Parameter $\lambda_i$ is of somewhat ``intrinsic geometry'' and could be thoroughly determined by the geometric feature of the two linear pieces, once $\boldsymbol{w}_i$ and $b_i$ are fixed.
\item[\rm{\Romannum{5}}.] \textbf{Expressive capability}. Denote by $H$ the set of the hyperplanes of the hidden-layer units of $\mathfrak{N}$. Under the partition of some $H$, arbitrary piecewise linear function can be implemented by $\mathfrak{N}$ through only adjusting the output weights of the units of the hidden layer.
\item[\rm{\Romannum{6}}.] \textbf{Basic principles of solutions}. Global units, multiple strict partial orders and continuity restriction are the three basic principles of the solution for higher-dimensional input and their combination yields various concrete solutions.

\end{itemize}

\section{Discussion}
A two-layer ReLU network is also a ``black box'' despite the architecture being the simplest one. The main goal of this paper is to understand its training solution derived from the back-propagation algorithm, through constructing function-approximation solutions.

From theoretical viewpoint, the revealed solution space demonstrated its complexity and diversity due to the combination of different basic principles. To our original purpose, corollary 3 successfully explained the training solution for one-dimensional input and section 8.1 verified that by experiments. Several solution patterns for higher-dimensional input predicted by our theory were also experimentally validated in section 8.2. Thus, from both the theoretical and practical aspects, the theory of this paper deserves further attention and more detailed studies. We propose the following two open problems to delve deeper into the research of two-layer ReLU networks.
\begin{itemize}
\item[] \textbf{Problem 1}. Given an arbitrary partition $\mathcal{P}$ of $U = [0, 1]^n$ for $n \ge 2$ by a set of $n-1$-dimensional hyperplanes, can the three basic principles of solutions (item 6 of section 9) cover all the regions of $\mathcal{P}$?
\item[] \textbf{Problem 2}. Let $D$ be a set of data points derived from the discretization of a continuous function $f: U=[0, 1]^n \to \mathbb{R}$ for $n \ge 2$. Under a fixed number of $n-1$-dimensional hyperplanes and a certain data-fitting error for $D$, how many ways can be found to divide $U$ such that all the regions are covered by the three basic principles above?
\end{itemize}

We give some remarks on the two problems. If the answer to problem 1 is yes, the training solution of two-layer ReLU networks for input-dimensionality $n \ge 2$ would be completely understood; otherwise, there may exist some other mechanisms to be discovered or a two-layer ReLU network may have intrinsic deficiency in its expressive capability. Problem 2 is related to the complexity of the solution space for interpolation and is of great importance in understanding the minima of loss functions of the training method.

Finally, since a two-layer neural network is the simplest feedforward one, its mechanism may be fundamental to more general architectures. We hope that the results of this paper could advance the understanding of the ``black box'' of neural networks.

\end{document}